%% file: arxiv_version.tex
\documentclass{article}

\usepackage[nonatbib, final]{neurips_2020}
\usepackage[numbers]{natbib}

\usepackage[utf8]{inputenc} % allow utf-8 input
\usepackage[T1]{fontenc}    % use 8-bit T1 fonts

\usepackage{booktabs}       % professional-quality tables
\usepackage{amsfonts}       % blackboard math symbols
\usepackage{nicefrac}       % compact symbols for 1/2, etc.
\usepackage{microtype}      % microtypography
\usepackage[table, dvipsnames]{xcolor}

\usepackage{graphicx}
\usepackage{caption}
\usepackage{subcaption}

\usepackage{url}

\usepackage{algorithm,algpseudocode}
\usepackage{amsmath}
\usepackage{amssymb}
\usepackage{caption}
\usepackage{amsthm}
\usepackage{bbm}
\usepackage{array,multirow}
\usepackage{wrapfig}

\usepackage[english]{babel}
\usepackage{mdframed}   % for framing

\input{math_commands.tex}

\input{preamble.tex}

\usepackage{enumitem}
\setlist[enumerate]{topsep=0pt,itemsep=-1ex,partopsep=1ex,parsep=1ex,leftmargin=15pt}
\setlist[itemize]{topsep=0pt,itemsep=-0.3ex,partopsep=1ex,parsep=1ex,leftmargin=15pt}

\usepackage[compact]{titlesec}
\titlespacing{\section}{0pt}{2ex}{1ex}
\titlespacing{\subsection}{0pt}{1ex}{0ex}
\titlespacing{\subsubsection}{0pt}{0.5ex}{0ex}

\usepackage{hyperref}

\renewcommand{\paragraph}[1]{{\textbf{#1.}}}

\title{Meta-Learning Stationary Stochastic Process Prediction with Convolutional Neural Processes}

\author{%
  Andrew Y.~K.~Foong\thanks{Authors contributed equally.}\\
  University of Cambridge \\
  \texttt{ykf21@cam.ac.uk} \\%[-0.5em]
  \And
  Wessel P.~Bruinsma\footnotemark[1] \\
  University of Cambridge \\
  Invenia Labs \\
  \texttt{wpb23@cam.ac.uk} \\%[-0.5em]
  \And  
  Jonathan Gordon\footnotemark[1]\\
  University of Cambridge \\
  \texttt{jg801@cam.ac.uk} \\%[-0.5em]
  \And
  Yann Dubois \\
  University of Cambridge \\
  \texttt{yanndubois96@gmail.com} \\%[-2em]
  \And
  James Requeima \\
  University of Cambridge \\
  Invenia Labs \\
  \texttt{jrr41@cam.ac.uk} \\%[-2em]
  \And
  Richard E.~Turner \\
  University of Cambridge \\
  Microsoft Research \\
  \texttt{ret26@cam.ac.uk} \\%[-2em]
}

\begin{document}

\maketitle

\begin{abstract}
    Stationary stochastic processes (SPs) are a key component of many probabilistic models, such as those for off-the-grid spatio-temporal data. They enable the statistical symmetry of underlying physical phenomena to be leveraged, thereby aiding generalization.
    Prediction in such models can be viewed as a \emph{translation equivariant} map from observed data sets to predictive SPs, emphasizing the intimate relationship between stationarity and equivariance.
    Building on this, we propose the Convolutional Neural Process (ConvNP), which endows Neural Processes (NPs) with translation equivariance and extends convolutional conditional NPs to allow for dependencies in the predictive distribution. 
    The latter enables ConvNPs to be deployed in settings which require coherent samples, such as Thompson sampling or conditional image completion. 
    Moreover, we propose a new maximum-likelihood objective to replace the standard ELBO objective in NPs, which conceptually simplifies the framework and empirically improves performance. 
    We demonstrate the strong performance and generalization capabilities of ConvNPs on 1D regression, image completion, and various tasks with real-world spatio-temporal data.
\end{abstract}

\section{Introduction}
\label{sec:introduction}
Incorporating appropriate inductive biases into machine learning models is key to achieving good generalization performance.
Consider, for example, predicting rainfall at an unseen test location from rainfall measurements nearby.
A powerful inductive bias for this task is \textit{stationarity}: the assumption that the generative process governing rainfall is spatially homogeneous. 
Given only observations in a limited part of the space, stationarity allows the model to extrapolate to yet unobserved regions. 
Closely related to stationarity is \textit{translation equivariance} (TE). TE formalizes the intuitive idea that if observations are shifted in time or space, then the resulting predictions should be shifted by the same amount.
When stationarity or TE is appropriate, e.g.\ in time-series \citep{roberts2013gaussian}, images \citep{lecun1998gradient}, and spatio-temporal modelling \citep{delhomme1978kriging, cressie1990origins}, incorporating them into our models yields significant benefits. 

A general framework for these tasks is to view them as prediction of a \emph{stochastic process} (SP; \citep{ross1996stochastic}).
This principled approach has inspired a new set of deep learning architectures that bring the expressivity and fast test-time inference of deep learning to SP modelling.
\textit{Conditional Neural Processes} (CNPs; \citep{garnelo2018conditional}) use neural networks to directly parameterize a map from data sets to predictive SPs, which is trained via meta-learning \citep{schmidhuber1987evolutionary, thrun2012learning}.
However, CNPs suffer from several drawbacks that inhibit their use in scenarios where other SP models, e.g. Gaussian processes (GPs; \citep{rasmussen2003gaussian}), often succeed.
First, vanilla CNPs cannot account for TE as an inductive bias.
This was recently addressed with the introduction of ConvCNPs \citep{gordon2020convolutional}.
Second, both CNPs and ConvCNPs are limited to factorized, parametric predictive distributions. 
This makes them unsuitable for producing coherent predictive function samples or modelling complicated likelihoods.
\emph{Neural Processes} (NPs; \citep{garnelo2018neural}), a latent variable extension of CNPs, were introduced to enable richer joint predictive distributions.
However, the NP training procedure uses variational inference (VI) and amortization, which are known to suffer from certain drawbacks \citep{turner+sahani:2011a,cremer2018inference}.
Moreover, existing NPs do not incorporate TE. 

This paper builds on ConvCNPs and NPs \citep{garnelo2018neural, gordon2020convolutional} to develop \textit{Convolutional Neural Processes} (ConvNPs).
ConvNPs are a map from data sets to predictive SPs that is both TE \emph{and} capable of expressing complex joint distributions.
As training ConvNPs with VI poses technical and practical issues, we instead propose a simplified maximum-likelihood objective, which directly targets the predictive SP.
We show that ConvNPs produce compelling samples and generalize effectively, making them suitable for a broad range of spatio-temporal prediction tasks.
Our key contributions are:
\begin{enumerate}
    \item We introduce ConvNPs, extending ConvCNPs to model rich joint predictive distributions.
    \item We propose a simplified training procedure, discarding VI in favor of an approximate maximum-likelihood procedure, which improves performance for ConvNPs.
    \item We demonstrate the usefulness of ConvNPs on toy time-series experiments, image-based sampling and extrapolation, and real-world environmental data sets.
\end{enumerate}

\section{Problem Set-up and Background}
\label{sec:problem_statement}

\paragraph{Notation}
The main paper provides an informal treatment of ConvNPs. We refer the reader to the supplement for precise definitions and statements.
Let $\gX= \mathbb{R}^{d_{\mathrm{in}}},\gY= \mathbb{R}$ denote the input and output spaces, and let $(\vx, y)$ be an input-output pair.
Let $\gS$ be the collection of all finite data sets, with $D_c, D_t \in \gS$ a \textit{context} and \textit{target} set respectively.
We will later consider predicting the target set from the context set as in \citep{garnelo2018conditional,garnelo2018neural}.
Let $\mX_c, \vy_c$ be the inputs and corresponding outputs of $D_c$, with $\mX_t, \vy_t$ defined analogously.
We denote a single \textit{task} as $\xi = (D_c, D_t) = ((\mX_c, \vy_c), (\mX_t, \vy_t))$.
Let $\gP(\gX)$ denote the collection of stochastic processes on $\gX$, and let $C_b(\gX)$ denote the collection of continuous, bounded functions on $\gX$.

\subsection{Meta-Learning Stochastic Process Prediction}
\label{sec:stochastic_processes}
Consider rainfall $y$ as a function of position $\vx$. 
To model rainfall, we can view it as a \emph{random} function from $\gX$ to $\gY$. Mathematically, this corresponds to a SP on $\gX$---a probability distribution over functions from $\gX$ to $\gY$---which we denote by $P$.
Given perfect knowledge of $P$, we could predict rainfall at any location of interest by conditioning $P$ on observations $D_c$, yielding a \emph{predictive} SP.
However, in practice we will only have access to a large collection of sample functions from $P$.
Each function is known only at a finite set of inputs, $D = (\vx_n, y_n)_{n=1}^N$, which we divide into $D_c, D_t$ for meta-training. 
Given sufficient data, we can \emph{meta-learn} the map from context sets $D_c$ to the ground-truth predictive distribution: $D_c \mapsto p(\vy_t | \mX_t, D_c) = p(\vy_t, \vy_c| \mX_t, \mX_c) / p(\vy_c|\mX_c)$.
As long as the predictives for varying $\mX_t$ are Kolmogorov-consistent \citep[Section 2.4]{tao2011introduction}, this corresponds to learning a map from data sets directly to \emph{predictive} SPs.
We refer to the map that takes a context set $D_c$ to the \emph{exact} ground truth SP conditioned on $D_c$ as the \emph{prediction map} $\pi_P\colon \gS \to \gP(\gX)$ (details in \cref{app:bayes_map}).
The general prediction problem may then be viewed as learning to approximate $\pi_P$.

\subsection{Translation Equivariance and Stationarity} \label{sec:TE_maps_to_processes}
The prediction map $\pi_P$ possesses two important symmetries.
First, $\pi_P$ is \textit{invariant} to permutations of $D_c$ \citep{zaheer2017deep, gordon2020convolutional}.
Second, if the ground truth process $P$ is \textit{stationary}, then $\pi_P$ is \textit{translation equivariant}: whenever an input to the map is translated, its output is translated by the same amount (see \cref{app:stationary_implies_equivariant} for formal definitions and proofs).
This simple statement highlights the intimate relationship between stationarity and TE.
Moreover, it suggests that models for the prediction map should also be TE and permutation invariant.
As such models are a small subset of the space of \emph{all} models, building in these properties can greatly improve data efficiency and generalization for stationary SP prediction.
In \cref{sec:convnps}, we extend the TE maps of \citet{gordon2020convolutional} (reviewed next) to construct a rich class of models which incorporate these inductive biases.

\subsection{Convolutional Conditional Neural Processes}
\label{sec:convcnps}
We review ConvCNPs \citep{gordon2020convolutional}, which are an important building block in our proposed model.
ConvCNPs can be viewed from the perspective of SP prediction, revealing their key limitations.
Given a context set $D_c$, the ConvCNP models the predictive distribution over target outputs as:
\begin{align}
\label{eqn:convcnp_likelihood}
    p_{\bm{\phi}}(\vy_t| \mX_t, D_c) = {\textstyle\prod_{(\vx, y) \in D_t}} \mathcal{N}(y; \mu(\vx, D_c), \sigma^2(\vx, D_c)).
\end{align}
The mean $\mu(\,\mathord{\cdot}\,, D_c)$ and variance $\sigma^2(\,\cdot\,, D_c)$ are parametrized by \emph{convolutional deep sets} (ConvDeepSets; \citep{gordon2020convolutional}): a flexible parametrization for TE maps from $\gS$ to $C_b(\gX)$.
ConvDeepSets introduce the idea of \emph{functional representations}: whereas the standard DeepSets framework embeds data sets into a finite-dimensional vector space \citep{zaheer2017deep}, a ConvDeepSet embeds data sets in an infinite-dimensional function space. 
ConvDeepSets are a composition of two stages.
The first stage maps a data set $D$ to its functional representation via $D \mapsto \sum_{(\vx, y) \in D} \phi(y) \psi(\,\mathord{\cdot} - \vx)$.
Here $\phi(y) = (1, y) \in \mathbb{R}^2$ and $\psi$ is the Gaussian radial basis function. 
This functional representation is then passed to the second stage, a TE map between function spaces, implemented by a convolutional neural network (CNN).
See \cref{app:pseudocode} for a full description of the ConvCNP including pseudocode.

We observe that \cref{eqn:convcnp_likelihood} defines a map from context sets $D_c$ to predictive SPs.
Specifically, let $\gP_{\mathrm{N}}(\gX) \subset \gP(\gX)$ denote the set of \emph{noise GPs}: GPs on $\gX$ whose covariance is given by $\Cov(\vx, \vx') = \sigma^2(\vx)\delta[\vx - \vx']$, where $\sigma^2 \in C_b(\gX)$ and $\delta[0]=1$ with $\delta[\,\cdot\,]=0$ otherwise. 
Then the ConvCNP is a map $\mathrm{ConvCNP}: \gS \to \gP_{\mathrm{N}}(\gX)$ with \cref{eqn:convcnp_likelihood} defining its finite-dimensional distributions.
Since ConvDeepSets are TE, and the means and variances of ConvCNPs are ConvDeepSets, it follows that ConvCNPs are also TE as maps from $\gS \to \gP_{\mathrm{N}}(\gX)$ (see \cref{app:equivariance_proof} for a more formal derivation).
Unfortunately, processes in $\gP_{\mathrm{N}}(\gX)$ possess two key limitations.
First, it is impossible to obtain coherent function samples as each point of the function is generated independently. 
Second, Gaussian distributions cannot model multi-modality, heavy-tailedness, or asymmetry.

\section{The Convolutional Neural Process}
\label{sec:convnps}

We now present the ConvNP, which addresses the weaknesses of ConvCNPs.
We introduce their parametrization (\cref{sec:convnp_parameterization}) and a maximum-likelihood meta-training procedure (\cref{sec:ml_training}).

\subsection{Parametrizing Translation Equivariant Maps to Stochastic Processes Using ConvNPs}
\label{sec:convnp_parameterization}
%
\input{diagrams/forwardpass_full}
\begin{figure}
    \centering
    \includegraphics[width=\textwidth]{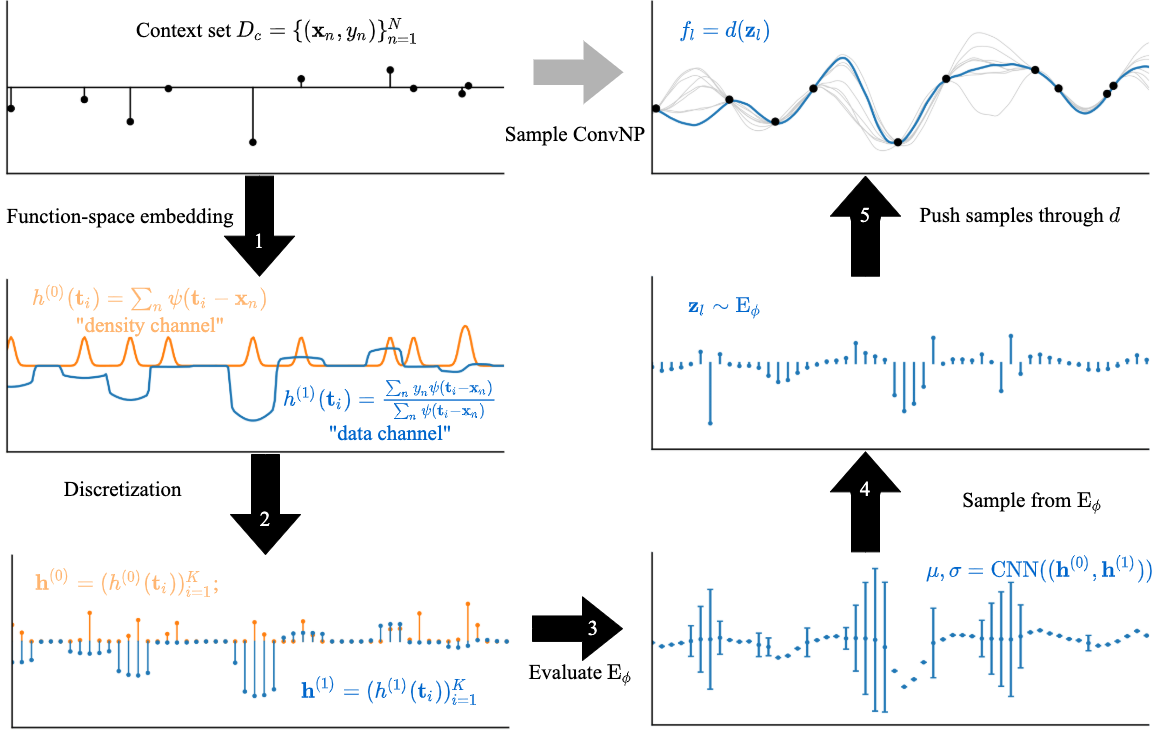}
    \caption{Forward pass of a ConvNP. Steps (1)-(4) depict sampling from the encoder $\mathrm{E}_{\bm{\phi}}$, which is a ConvCNP. This involves: (1) computing a functional representation of the context set, with separate `density' and `data' channels (described in detail in \citet{gordon2020convolutional} and \cref{app:pseudocode}), (2) discretizing the representation, (3) passing the representation through a CNN, which outputs the parameters of independent Gaussian distributions spaced on a grid, and (4) sampling from these distributions. However, the samples at each grid point are independent of each other, hence in (5) the samples are passed through \emph{another} CNN, the decoder, to induce dependencies, and then are smoothed out.}
    \label{fig:convnp_full_diagram}
\end{figure}
The ConvNP extends the ConvCNP by parametrizing a map to predictive SPs more expressive than $\gP_{\mathrm{N}}(\gX)$, allowing for coherent sampling and non-Gaussian predictives. It achieves this by passing the output of a ConvCNP through a non-linear, TE map between function spaces. 
Specifically, the ConvNP uses an encoder--decoder architecture, where the encoder $\mathrm{E}\colon \gS \to \gP_{\mathrm{N}}(\gX)$ is a ConvCNP and the decoder $d\colon \mathbb{R}^\gX \to \mathbb{R}^\gX$ is TE (here $\mathbb{R}^\gX$ denotes the set of all functions from $\gX$ to $\mathbb{R}$). 
Conditioned on $D_c$, ConvNP samples can be obtained by sampling a function $z \sim \mathrm{ConvCNP}(D_c)$ and then computing $f = d(z)$.
This is illustrated in \cref{fig:forward_pass}.
Importantly, $d$ takes functions to functions and does not necessarily act point-wise: letting $f(\vx)$ depend on the value of $z$ at multiple locations is crucial for inducing dependencies in the predictive.
This sampling procedure induces a map between SPs, $\mathrm{D}\colon \gP_{\mathrm{N}}(\gX) \to \gP(\gX)$ (see \cref{app:equivariance_proof}). Putting these together, with explicit parameter dependence in $\mathrm{E}$ and $\mathrm{D}$, the ConvNP is constructed as
\[
\mathrm{ConvNP}_{\bm{\theta}, \bm{\phi}} = \mathrm{D}_{\bm{\theta}} \circ \mathrm{E}_{\bm{\phi}}, \quad 
    \mathrm{E}_{\bm{\phi}} = \mathrm{ConvCNP}_{\bm{\phi}}, \quad 
    \mathrm{D}_{\bm{\theta}} = (d_{\bm{\theta}})_*,
\]
where $(d_{\mathrm{\bm{\theta}}})_*$ is the pushforward\footnote{i.e., $(d_{\mathrm{\bm{\theta}}})_*(\mathrm{E}_{\bm{\phi}})$ is the measure induced on $\mathbb{R}^\gX$ by sampling a function from $\mathrm{E}_{\bm{\phi}}$ and passing it through $d_{\mathrm{\bm{\theta}}}$.} under $d_{\mathrm{\bm{\theta}}}$.
In \cref{app:equivariance_proof}, we prove that $\mathrm{ConvNP}_{\bm{\theta}, \bm{\phi}}$ is indeed TE.

In practice, we cannot compute samples of noise GPs ($\mathcal{P}_{\mathrm{N}}$) because they comprise uncountably many independent random variables. Instead, we consider a discrete version of the model, which enables computation.
Following \citet{gordon2020convolutional}, we discretize the domain of $z$ on a grid $(\vx_i)_{i=1}^{K}$, with $\vz \coloneqq (z(\vx_i))_{i=1}^{K}$.
As a consequence, the model can only be equivariant up to shifts on this discrete grid.
With this discretization, sampling $\vz \sim \mathrm{ConvCNP}_{\bm{\phi}}(D_c)$ amounts to sampling independent Gaussian random variables, and $d_{\bm{\theta}}$ is implemented by passing $\vz$ through a CNN. The forward pass of a trained ConvNP is illustrated in \cref{fig:convnp_full_diagram}.
Note that CNNs are not always entirely TE due to the zero padding that occurs at each layer. 
In practice, we find that this is not an issue.\footnote{See \citet[Appendix D.6]{gordon2020convolutional} for a discussion.} 
Following \citet{kim2018attentive}, we define the model likelihood by adding heteroskedastic Gaussian observation noise $\sigma^2_y(\vx, \vz)$ to the predictive function draws $f = d_\vtheta(\vz) \in \mathbb{R}^{\gX}$:
\begin{equation}
    \textstyle
    p_{\vphi,\vtheta}(\vy_t | \mX_t, D_c) = \mathop\mathbb{E}_{\vz \sim \mathrm{E}_{\bm{\phi}}(D_c)} \Big[\prod_{(\vx, y) \in D_t} \gN \left(y; d_{\bm{\theta}}(\vz) (\vx), \sigma_y^2(\vx, \vz) \right) \Big].\label{eqn:convnp_likelihod_fn}
\end{equation}
Although the product in the expectation factorizes, $p_{\vphi,\vtheta}(\vy_t | \mX_t, D_c)$ does not: $\vz$ induces dependencies in the predictive, in contrast to \cref{eqn:convcnp_likelihood}.
See \cref{app:pseudocode} for full implementation details for the ConvNP.

\subsection{Maximum Likelihood Learning of ConvNPs}
\label{sec:ml_training}
We now propose a maximum-likelihood training procedure for ConvNPs. 
Let the ground truth task distribution be $p(\xi) = p(D_c, D_t)$.
Let $\gL_{\mathrm{ML}}(\bm{\theta}, \bm{\phi}; \xi) \coloneqq \log p_{\vphi,\vtheta}(\vy_t | \mX_t, D_c)$ be the single-task likelihood, and let $\gL_{\mathrm{ML}}(\bm{\theta}, \bm{\phi}) \coloneqq \mathbb{E}_{p(\xi)} [\log p_{\vphi,\vtheta}(\vy_t | \mX_t, D_c)]$ be the task-averaged likelihood.
The following proposition shows that maximizing $\gL_{\mathrm{ML}}$ recovers the prediction map $\pi_P$ in a suitable limit:

\begin{proposition} \label{rem:recover_bayes_map}
    Let $\Psi\colon \gS \to \gP(\gX)$ be a map from data sets to SPs, and let $\gL_{\mathrm{ML}}(\Psi) \coloneqq \mathbb{E}_{p(\xi)}[ \log p_{\Psi}(\vy_t| \mX_t, D_c)]$ where $p_{\Psi}$ is the density of $\Psi(D_c)$ at $\mX_t$.
    Then $\Psi$ globally maximizes $\gL_{\mathrm{ML}}(\Psi)$ if and only if $\Psi = \pi_P$.
    See \cref{app:log_el_convergence} for more details and conditions. 
\end{proposition}

In practice, we do not have infinite flexibility in our model or infinite data to compute expectations over $p(\xi)$, but \cref{rem:recover_bayes_map} shows that maximum-likelihood training is sensible with an expressive model and sufficient data.
Letting $\gD = \{\xi_n\}_{n=1}^{N_{\mathrm{tasks}}}$ be a \textit{meta-training} set, we can train a ConvNP by stochastic gradient maximization of $\gL_{\mathrm{ML}}$ with tasks sampled from $\gD$.
Unfortunately, for non-linear decoders, $\log p_{\vphi,\vtheta}(\vy_t | \mX_t, D_c)$ is intractable due to the expectation over $\vz$ (\cref{eqn:convnp_likelihod_fn}).
For a given task $\xi$, we instead optimize the following Monte Carlo estimate of $\gL_{\mathrm{ML}}(\bm{\theta}, \bm{\phi}; \xi)$, which is conservatively biased, consistent, and monotonically increasing in $L$ (in expectation) \citep{burda2015importance}:
\begin{equation}
\label{eqn:log_el_estimator}
    \textstyle
    \hat{\gL}_{\mathrm{ML}}(\bm{\theta}, \bm{\phi}; \xi) \coloneqq \log \left[ \frac{1}{L} \sum_{l=1}^L \exp \left( \sum_{(\vx, y) \in D_t} \log p_\vtheta(y | \vx, \vz_l) \right) \right];\,\,\,\, \vz_l \sim \mathrm{E}_{\bm{\phi}}(D_c).
\end{equation}
One drawback of this objective is that single sample estimators are not useful, as they drive $\vz$ to be deterministic. 
In our experiments, we set $L$ between 16 and 32. For further discussion of the effect of $L$ see \cref{app:effect_of_L}.
\Cref{eqn:log_el_estimator} can be viewed as importance sampling in which the prior is the proposal distribution. Prior sampling is typically ineffective as it is unlikely to propose functions that pass near observed data.
Here, however, $\mathrm{E}_{\bm{\phi}}$ depends on context sets $D_c$, which often is sufficient to constrain prior function samples to be close to $D_t$.
In \cref{sec:experiments}, we demonstrate that, perhaps surprisingly, this estimator often significantly outperforms VI-inspired estimators (discussed next).

\section{The Latent Variable Interpretation of ConvNPs}
\label{sec:latent_var_convnps}

We now describe an alternative approach to training the ConvNP via variational lower bound maximization. 
This serves the dual purpose of relating ConvNPs to the NP family, and contrasting the existing NP framework with our simplified, maximum-likelihood approach from  \cref{sec:ml_training}.

\subsection{A Variational Lower Bound Approach to ConvNPs}
\label{sec:vi_interpretation_convnp}

\citet{garnelo2018neural} propose viewing Neural Processes as performing approximate Bayesian inference and learning in the following latent variable model:
% \[ \textstyle
%     \vz \sim p_\vtheta(\vz); \quad y(\vx) = f_\vtheta(\vx; \vz); \quad p_\vtheta(\vy_t | \mX_t, \vz) = \prod_{(\vx, y) \in D_t} \gN \left(y ; f_\vtheta(\vx ; \vz), \sigma_y^2 \right).
% \]
\begin{align}\textstyle
\label{eqn:np_model}
    \vz \sim p_\vtheta(\vz); \quad y(\vx) = f_\vtheta(\vx; \vz); \quad p_\vtheta(\vy_t | \mX_t, \vz) = \prod_{(\vx, y) \in D_t} \gN \left(y ; f_\vtheta(\vx ; \vz), \sigma_y^2 \right).
\end{align}
To train the model, they propose using \textit{amortized} VI \citep{kingma2013auto, rezende2015variational}.
This approach involves introducing a variational approximation $q_{\bm{\phi}}$ which maps data sets $S \in \gS$ to distributions over $\vz$, and maximizing a lower bound (ELBO) on $\log p_{\bm{\theta}}(\vy_t | \mX_t, D_c)$.
We can define a similar procedure for ConvNPs. 
For ConvNPs, $\vz$ is a latent \textit{function}, $q_{\bm{\phi}}$ is a map from data sets to SPs, and $f_\vtheta$ is a map between function spaces.
A natural choice is to use a ConvCNP and CNN for $q_{\bm{\phi}}$ and $f_\vtheta$, respectively.
This results in the same parameterization as in \cref{sec:convnps}, but a different modelling interpretation and meta-training objective.
Given a task $\xi = (D_c, D_t)$, the ELBO for this model is: 
\[
\E_{\vz \sim q_{\bm{\phi}}(\vz | D_c \cup D_t)} \left[\log  p_\vtheta(\vy_t |\mX_t, \vz) \right] - \mathrm{KL}(q_{\bm{\phi}}(\vz | D_c \cup D_t) \| \textcolor{orange}{p(\vz | D_c)}).
\]
As $\textcolor{orange}{p(\vz | D_c)}$ is intractable to compute, \citet{garnelo2018neural} instead propose the following objective:
% As with NPs, this bound cannot be evaluated, because it requires computing $p(\vz | D_c)$, an intractable Bayesian posterior. 
%
% \citet{garnelo2018neural} instead propose the following objective:
%
\begin{align} \textstyle
\label{eqn:np_objective}
    \gL_{\mathrm{NP}}(\vtheta, \vphi; \xi) \coloneqq \E_{\vz \sim q_{\bm{\phi}}(\vz | D_c \cup D_t)} \left[\log  p_\vtheta(\vy_t |\mX_t, \vz) \right] - \mathrm{KL}(q_{\bm{\phi}}(\vz | D_c \cup D_t) \| \textcolor{blue}{q_{\bm{\phi}}(\vz | D_c)}),
\end{align}
where the intractable term $\textcolor{orange}{p(\vz | D_c)}$ has been substituted with our variational approximation $\textcolor{blue}{q_{\bm{\phi}}(\vz | D_c)}$.
Due to this substitution, $\gL_{\mathrm{NP}}$ is no longer a valid ELBO for the original model (\cref{eqn:np_model}). 
Rather, if we define \emph{separate} models for each context set $D_c$, and \emph{define} the conditional prior for each model as $p(\vz | D_c) \coloneqq q_{\bm{\phi}}(\vz | D_c)$, then $\gL_{\mathrm{NP}}$ may be thought of as performing VI in this \emph{collection} of models. However, there is no guarantee that these conditional priors are consistent in the sense that they correspond to a single Bayesian model as in \cref{eqn:np_model}.
%Rather, it may be viewed as separately performing VI in a family of inconsistent models, one for each context set.
%

For the non-discretized ConvNP, \cref{eqn:np_objective} involves KL divergences between SPs which cannot be computed directly and must be treated carefully \citep{matthews2016sparse, sun2018functional}.
On the other hand, for the discretized ConvNP, the KL divergences can be computed, but grow in magnitude as the discretization becomes finer, and it is not clear that the KL divergence between SPs is recovered in the limit. 
This raises practical issues for the use of \cref{eqn:np_objective} with the ConvNP, as the balance between the two terms depends on the choice of discretization.

\subsection{Maximum-Likelihood vs Variational Lower Bound Maximization for Training NPs}
\label{sec:convnp_objective_comparison}
We argue that the VI interpretation is unnecessary when focusing on predictive performance, and particularly detrimental for ConvNPs, where $\vz$ has many elements.
Noting the equivalence
\begin{equation}
\label{eqn:np_objective2}
\begin{split}
     \gL_{\mathrm{NP}}(\vtheta, \vphi; \xi)  =
      \gL_{\mathrm{ML}}(\vtheta, \vphi; \xi)
    - \mathrm{KL}\left( q_\vphi ( \vz | D_c \cup D_t ) \middle\| p_\vtheta( D_t | \vz ) q_\vphi ( \vz | D_c ) / Z \right),
\end{split}
\end{equation}
where $Z$ is a normalizing constant (see \cref{app:neural_process_models} for a full derivation),
we see that $\gL_{\mathrm{NP}}$ is equal to $\gL_{\mathrm{ML}}$ up to an additional KL term.
This KL term encourages consistency among the $q_{\bm{\phi}}(\vz|D)$ in the sense that Bayes' theorem is respected if the target set is subsumed into the context set.
In the infinite capacity/data limit, $\gL_{\mathrm{NP}}$ is globally maximized if the ConvNP recovers
\begin{inlinelist}
    \item the prediction map $\pi_P$ for $\vy_t$ and
    \item exact inference for $\vz$. 
\end{inlinelist}
This follows from
\begin{inlinelist}
    \item \cref{rem:recover_bayes_map}, since $\pi_P$ globally optimizes $\gL_{\mathrm{ML}}$; and 
    \item that exact inference for $\vz$ is Bayes-consistent, sending the KL term to zero.
\end{inlinelist}
In most applications, only the distribution over $\vy_t$ is of interest.
Given only finite capacity/data, it can be advantageous to not expend capacity in enforcing Bayes-consistency for $\vz$, which suggests it could be beneficial to use $\gL_{\mathrm{ML}}$ over $\gL_{\mathrm{NP}}$.
Further, $\gL_{\mathrm{ML}}$ has the advantage of being easy to specify for any map parameterizing a predictive process, posing no conceptual issues for the ConvNP.
In \cref{sec:experiments} we find that $\gL_{\mathrm{ML}}$ significantly outperforms $\gL_{\mathrm{NP}}$ for ConvNPs, and often also for ANPs. 

\section{Experiments}
\label{sec:experiments}
We evaluate ConvNPs on a broad range of tasks.
Our main questions are:
\begin{inlinelist}
    \item Does the ConvNP produce coherent, meaningful predictive samples?
    \item Can it leverage translation equivariance to outperform baseline methods within and beyond the training range (generalization)?
    \item Does it learn expressive non-Gaussian predictive distributions?
\end{inlinelist}

\paragraph{Evaluation and baselines}
We use several approaches for evaluating NPs. 
First, as in \citep{garnelo2018neural,kim2018attentive}, we provide qualitative comparisons of samples.
These allow us to see if the models display meaningful structure, quantify uncertainty, and are able to generalize spatially. 
Second, NPs lack closed-form likelihoods, so we evaluate \emph{lower bounds} on their predictive log-likelihoods via importance sampling \citep{le2018empirical}. 
As these bounds can be quite loose (\cref{app:effect_num_samples_eval}), they are primarily useful to show when NPs outperform baselines with \emph{exact} likelihoods, such as GPs and ConvCNPs. 
Finally, in \cref{sec:climate_prediction} we consider Bayesian optimization to evaluate the usefulness of ConvNPs for downstream tasks. 
In \cref{sec:1d_regression,sec:image_completion}, we compare against the Attentive NP (ANP; \citep{kim2018attentive}), which in prior work is trained with $\gL_{\mathrm{NP}}$.
The ANP architectures used here are comparable to those in \citet{kim2018attentive}, and have a parameter count comparable to or greater than the ConvNP. Full details provided in the supplement.\footnote{Code to reproduce the 1D regression experiments can be found at \url{https://github.com/wesselb/NeuralProcesses.jl}, and code to implement the image-completion experiments can be found at \url{https://github.com/YannDubs/Neural-Process-Family}.}

\input{plots/gp_plots_1d}
\input{tables/gp1d_table}

\subsection{1D Regression}
\label{sec:1d_regression}
We train on samples from 
\begin{inlinelist}
    \item a Mat\'ern-$\frac52$ GP,
    \item a weakly periodic GP, and
    \item a non-Gaussian sawtooth process with random shifts and frequency
\end{inlinelist}
(see \cref{app:details_1D} for details). 
\Cref{fig:1D_predictions} shows predictive samples, where during training the models only observe data within the grey regions (training range).
While samples from the ANP exhibit unnatural ``kinks'' and do not resemble the underlying process, the ConvNP produces smooth samples for Mat\'ern--$\frac52$ and samples exhibiting meaningful structure for the weakly periodic and sawtooth processes.
The ConvNP also generalizes gracefully beyond the training range, whereas ANP fails catastrophically. The ANP with $\gL_{\mathrm{NP}}$ collapses to deterministic samples, with the epistemic uncertainty explained using the heteroskedastic noise $\sigma_y^2(\vx, \vz)$. This was also noted in \citet{le2018empirical}. This behaviour is alleviated when training with $\gL_{\mathrm{ML}}$, with much of the predictive uncertainty due to variations in the sampled functions.

\Cref{table:1d_bakeoff_main_text} compares lower bounds on the log-likelihood for ConvNP with our proposed $\gL_{\mathrm{ML}}$ objective and ANP with both $\gL_{\mathrm{ML}}$ and the standard $\gL_{\mathrm{NP}}$ objective. 
We also show three \emph{exact} log-likelihoods: %
\begin{inlinelist}
    \item the ground-truth GP (full)
    \item the ground-truth GP with diagonalised predictions (diag), and
    \item ConvCNP.
\end{inlinelist}
The ConvCNP performs on par with GP (diag), which is the optimal factorized predictive.
The ConvNP lower bound is consistently higher than the GP (diag) and ConvCNP log-likelihoods, demonstrating that its correlated predictives improve predictive performance.
Further, the ConvNP performs similarly inside and outside its training range, demonstrating that TE helps generalization; this is in contrast to the ANP, which fails catastrophically outside its training range.
In \cref{app:additional_results_1d}, we provide a thorough comparison for multiple models, training objectives, and data sets.
%Code to reproduce the 1D regression experiments can be found at \url{https://github.com/wesselb/NeuralProcesses.jl}.

\subsection{Image Completion}
\label{sec:image_completion}

We evaluate ConvNPs on image completion tasks focusing on spatial generalization.
To test this, we consider zero-shot multi MNIST (ZSMM), where we train on single MNIST digits but test on two MNIST digits on a larger canvas. 
We randomly translate the digits during training, so the generative SP is stationary.
The black background on MNIST causes difficulty with heteroskedastic noise, as the models can obtain high likelihood by predicting the background with high confidence whilst ignoring the digits. 
Hence for MNIST and ZSMM we use homoskedastic noise $\sigma_y^2(\vz)$.
\cref{fig:images_zsmm_anp,fig:images_zsmm_convnp} show that the ANP fails to generalize spatially, whereas this is naturally handled by the ConvNP.
\input{tables/image_main_text_squeezed_homo}
\input{plots/image_completion/image_figures_homo}
We also test the ConvNP's ability to learn non-Gaussian predictive distributions.
\cref{fig:images_marginal_convnp} shows that the ConvNP can learn highly multimodal predictives, enabling the generation of diverse yet coherent samples. 
A quantitative comparison of models using log-likelihood lower bounds is provided in \cref{table:images_main_text}, where ConvNP trained with $\gL_{\mathrm{ML}}$ consistently achieves the highest values.
\cref{app:details_images} provides details regarding the data, architectures, and protocols used in our image experiments.
In \cref{app:additional_images}, we provide samples and further quantitative comparisons of models trained on SVHN \citep{netzer2011reading}, MNIST \cite{lecun1989backpropagation}, and $32\times 32$ CelebA \cite{netzer2011reading} in a range of scenarios, along with full experimental details. 
%
%Code to implement the models and image-completion experiments can be found at \url{https://github.com/YannDubs/Neural-Process-Family}.

\subsection{Environmental Data}
\label{sec:climate_prediction}
We next consider a real-world data set, ERA5-Land \citep{era5land}, containing environmental measurements at a $\sim$9 km spacing across the globe. 
We consider predicting daily precipitation $y$ at position $\vx$. 
We also provide the model with orography (elevation) and temperature values.
We choose a large region of central Europe as our train set, and use regions east, west and south as held-out test sets. 
For such tasks, models must be able to make predictions at locations spanning a range different from the training set, inhibiting the deployment of NPs not equipped with TE.
To sample a task at train time, we sample a random date between 1981 and 2020, then sample a sub-region within the train region, which is split into context and target sets.
In this section, we train using $\gL_{\mathrm{ML}}$. See \cref{app:details_environmental} for details.
\input{tables/environmental_metrics_v2}
\input{plots/percipitation_samples}
\input{plots/bayes_opt_plot}

\paragraph{Prediction}
We first evaluate the ConvNP's predictive performance, comparing to a GP trained individually on each task as a baseline.
In about $10\%$ of tasks, the GP obtains a poor likelihood ($< 0$ nats); we remove these outliers from the evaluation. 
The results are shown in \cref{table:era5_loglik_rmse}. 
The ConvNP and GP have comparable RMSEs except on south, where the ConvNP outperforms the GP. 
However, the ConvNP consistently outperforms the GP in log-likelihood, which is expected for the following reasons:
\begin{inlinelist}
    \item the GP does not share information between tasks and hence is prone to overfitting on small context sets, resulting in overconfident predictions; and
    \item the ConvNP can learn non-Gaussian predictive densities (illustrated in \cref{app:environmental_figures}).
\end{inlinelist}  
\cref{fig:precipitation_example_task} shows samples from the predictive process of a ConvNP and GP, over the whole of the train region.
This demonstrates spatial extrapolation, as the ConvNP is trained only on random subregions.

\paragraph{Bayesian optimization}
We demonstrate the ConvNP in a downstream task by considering a toy Bayesian optimisation problem, where the goal is to identify the location with heaviest rainfall on a given day. 
We also test the ConvNP's spatial generalization, by optimising over larger regions (for central, west, and south) than the model was trained on. 
We test both Thompson sampling (TS) \citep{thompson1933likelihood} and upper confidence bounds (UCB) \citep{auer2002using} as methods for acquiring points.
Note that TS requires coherent samples.
The results are shown in \cref{fig:bayesopt}. 
On all data sets, ConvNP TS and UCB significantly outperform the random baseline by the 50th iteration; the GP does not reliably outperform random. 
We hypothesize this is due to its overconfidence, in line with the results on prediction.

\section{Related Work and Discussion}
\label{sec:conclusions}

We have introduced the ConvNP, a TE map from observed data sets to predictive SPs.
Within the NP framework, ConvNPs bring together three key considerations.

\paragraph{Expressive joint densities}       
ConvNPs extend ConvCNPs to allow for expressive joint predictive densities.
A powerful alternative approach is to combine  \textit{autoregressive} (AR) models (such as PixelCNN++ \citep{salimans2017pixelcnn++} and the Image Transformer \citep{parmar2018image}) with CNPs.
A difficulty in introducing AR sampling to CNPs is the need to specify a sampling ordering, which is in tension with permutation invariance and relates to the discussion on Bayes-consistency (\cref{sec:convnp_objective_comparison}).
Several works have considered \textit{exchangeable} NP models \citep{louizos2019functional, kumar2018consistent, korshunova2020conditional}, providing an avenue for future investigation.

\paragraph{Translation equivariance}
There has been much interest in incorporating equivariance with respect to symmetry groups into neural networks, e.g.\ \citep{kondor2008group, cohen2016group, cohen19gauge, kondor2018generalization}, with a comprehensive treatment provided by \citet{bloem2020probabilistic}.
ConvNPs leverage a simple relationship between translation equivariance and stationarity to construct a model particularly well suited to stationary SPs.
Similar ideas have been explored for 3D point-cloud modelling \citep{qi2017pointnet, qi2017pointnetplus}.
For example, the models proposed in \citep{wu2019pointconv, wang2018deep} perform convolutions over continuous domains, which are both TE and permutation invariant, achieving excellent performance in point-cloud classification.
In contrast with ConvNPs, point-cloud models 
\begin{inlinelist}
    \item are generally used as classification function approximators, rather than meta or few-shot learners;
    \item are typically tailored towards point clouds, making heavy use of specific properties for function design; and
    \item have not considered latent variable or stochastic generalizations.
\end{inlinelist}

\paragraph{Neural Process training procedures}
One of the key benefits of CNPs is their simple  maximum-likelihood training procedure \citep{garnelo2018conditional, gordon2020convolutional}. 
In contrast, NPs are usually trained with VI-inspired objectives \citep{garnelo2018neural}, variants of which are empirically investigated in \citet{le2018empirical}.
We propose an alternative training procedure that discards VI in favor of a (biased) maximum-likelihood approach that focuses on directly optimizing predictive performance.
In this regard, our work is similar to \citet{gordon2018meta}, albeit in a very different domain.
This approach has two benefits:
\begin{inlinelist}
    \item it does not require carefully designed inference procedures, and works ``out-of-the-box'' for a range of models; and
    \item empirically, we find that it leads to improved performance for ConvNPs and, often, for ANPs.
\end{inlinelist}

\section*{Broader Impact}

The proposed model and training procedure are geared towards off-the-grid, spatio-temporal applications.
As such, ConvNPs are particularly well-suited for many important applications in the medical and environmental sciences, such as modelling electronic healthcare records or the temporal evolution of temperatures.
We hope that one impact of ConvNPs is to increase the usability of deep learning tools in the sciences.
Another potential application of ConvNPs is image generation, which has potentially negative societal impacts. 
However, ConvNPs focus on predicting distributions over images, and are far from state-of-the-art in terms of perceptual quality.
Thus we believe the societal impact of ConvNPs via image-generation will be insignificant.

\section*{Acknowledgements}
The authors would like to thank Invenia Labs for their support during the project. We thank William Tebbutt for insightful discussions. We thank David R.~Burt, Eric Nalisnick, Cozmin Ududec and John Bronskill for helpful comments on the manuscript. Andrew Y.~K.~Foong gratefully acknowledges funding from a Trinity Hall Research Studentship and the George and Lilian Schiff Foundation. Part of the work was done while Yann Dubois was working as an AI resident at Facebook. Richard E.~Turner is supported by Google, Amazon, ARM, Improbable, EPSRC grants EP/M0269571 and EP/L000776/1, and the UKRI Centre for Doctoral Training in the Application of Artificial Intelligence to the study of Environmental Risks (AI4ER).

\bibliography{references}
\bibliographystyle{plainnat}

\clearpage
\newpage

\appendix

\section{Formal Definitions and Set-up}
\label{app:bayes_map}
\paragraph{Notation}
We first review the notation introduced in the main body for convenience. 
Let $\gX = \mathbb{R}^{d_{\mathrm{in}}}$ and $\gY = \mathbb{R}$ denote the input and output spaces respectively, and let $(\vx, y)$ denote a generic input-output pair (higher-dimensional outputs can be treated easily).
Define $\gS_N = (\gX \times \gY)^N$ to be the collection of all data sets of size $N$, and let $\gS \coloneqq \bigcup_{N=1}^\infty \gS_N$. 
Let $D_c, D_t \in \gS$ denote a \textit{context} and \textit{target} set respectively. Later, as is common in recent meta-learning approaches, we will consider predicting the target set from the context set \citet{garnelo2018conditional,garnelo2018neural}.
Let $\mX_c = (\vx_1, \hdots, \vx_{N_c})$ denote a matrix of context set inputs, with $\vy_c = (y_1, \hdots, y_{N_c})$ the corresponding outputs; $\mX_t, \vy_t$ are defined analogously.
We denote a single \textit{task} as $\xi = (D_c, D_t) = (D_c, (\mX_t, \vy_t)).$

\paragraph{Stochastic processes}
For our purposes, a stochastic process on $\gX$ will be defined\footnote{Strictly speaking, this is non-standard terminology, since $P$ is the \emph{law} of a stochastic process.} as a probability measure on the set of functions from $\gX \to \mathbb{R}$, i.e.~$\mathbb{R}^{\gX}$, equipped with the product $\sigma$-algebra of the Borel $\sigma$-algebra over each index point \citep{tao2011introduction}, denoted $\Sigma$. The measurable sets of $\Sigma$ are those which can be specified by the values of the function at a countable subset $I \subset \gX$ of its input locations. Since in practice we only ever observe data at a finite number of points, this is sufficient for our purposes.
We denote the set of all such measures as $\gP(\gX)$.
We model the world as having a ground truth stochastic process $P \in \gP(\gX)$.
Consider a Kolmogorov-consistent (i.e.~consistent under marginalization) collection of distributions on finite index sets $I \subset \gX$. By the Kolmogorov extension theorem, there exists a unique measure on $(\mathbb{R}^{\gX}, \Sigma)$ that has these distributions as its finite marginals. Hence we may think of these stochastic processes as defined by their finite-dimensional marginals.

\paragraph{Conditioning on observations}
We now define what it means to condition on observations of the stochastic process $P$. Let $p(\vy| \mX)$ denote the density with respect to Lebesgue measure of the finite marginal of $P$ with index set $\mX$ (we assume these densities always exist). Assume we have observed $P$ at a finite number of points $(\mX_c, \vy_c)$, with $p(\vy_c | \mX_c) > 0$. Let $\mX_t$ be another finite index set. Then we define the finite marginal at $\mX_t$ conditioned on $D_c$ as the distribution with density
\begin{align} \label{eqn:conditionals}
    p(\vy_t|\mX_t, D_c) = \frac{p(\vy_t, \vy_c|\mX_t, \mX_c)}{p(\vy_c|\mX_c)}.
\end{align}
It can easily be verified that for a fixed $D_c$, the conditional marginal distributions for different $\mX_t$ in \cref{eqn:conditionals} are Kolmogorov-consistent. Again, the Kolmogorov extension theorem implies there is a unique measure $P_{D_c}$ on $(\mathbb{R}^{\gX}, \Sigma)$ that has \cref{eqn:conditionals} as its finite marginals. We now define $\pi_P: \gS \to \gP(\gX), \pi_P: D_c \mapsto P_{D_c}$ as the \emph{prediction map}, so called because it maps each observed dataset $D_c$ to the exact predictive stochastic process conditioned on $D_c$. The meta-learning task may be viewed as learning an approximation to the prediction map.

\section{Stationary Processes and Translation Equivariance}\label{app:stationary_implies_equivariant}

\begin{definition}[Translating data sets and SPs] We define the action of the translation operator $T_{\bm{\tau}}$ on data sets and SPs, where $\bm{\tau} \in \gX$ denotes the shift vector of the translation.\footnote{To prevent notational clutter, the same symbol, $T_{\bm{\tau}}$, will denote translations on multiple kinds of objects.}
\begin{enumerate}
    \item Let $(\vx_n, \vy_n)_{n=1}^N = S \in \gS$. For the index set $\mX = (\vx_1, \hdots, \vx_n)$, the translation by $\bm{\tau}$ is defined as $T_{\bm{\tau}}\mX = (\vx_1 + \bm{\tau}, \hdots, \vx_n + \bm{\tau})$. Similarly, $T_{\bm{\tau}}S \coloneqq (\vx_n + \bm{\tau}, \vy_n)_{n=1}^N$.
    \item For a function $f \in \mathbb{R}^{\gX}$, define $T_{\bm{\tau}}f(\vx) \coloneqq f(\vx - \bm{\tau})$ for all $\vx \in \gX$. Let $F \in \Sigma$ be a measurable set of functions. Then $T_{\bm{\tau}}F \coloneqq \{ T_{\bm{\tau}} f: f \in F \}$.
    \item For any SP $P \in \gP(\gX)$, we now define $T_{\bm{\tau}}P$ by setting\footnote{This is well-defined since $\Sigma$ is closed under translations. Equivalently, we could define $T_{\bm{\tau}}P$ as the push-forward of $P$ under the the translation map on functions, $T_{\bm{\tau}}: \mathbb{R}^{\gX} \to \mathbb{R}^{\gX}$.} $T_{\bm{\tau}}P(F) \coloneqq P(T_{-\bm{\tau}}F)$ for all $F \in \Sigma$.

\end{enumerate}

\end{definition}

\begin{definition}[Stationary SP] 
We say a stochastic process is (strictly) \emph{stationary} if the densities of its finite marginals satisfy
\begin{align}
    p(\vy_t | \mX_t) = p(\vy_t | T_{\bm{\tau}}\mX_t)
\end{align}
for all $\vy_t$, $\mX_t$ and $\bm{\tau}$. 
\end{definition}

\begin{definition}[Translation equivariant prediction maps]
\label{property:translation_equivariance}
We say that $\Psi\colon \gS \to \gP(\gX)$ is \emph{translation equivariant} if $\Psi(T_{\bm{\tau}}S) = T_{\bm{\tau}}\Psi(S)$ for any data set $S \in \gS$ and shift $\bm{\tau} \in \gX$. 
\end{definition}
The following simple statement highlights the link between stationarity and translation equivariance:
\begin{proposition} \label{prop:stationary_implies_equivariant}
    Let $P$ be a stationary SP. Then the prediction map $\pi_P$ is translation equivariant.\footnote{We exclude conditioning on observations that have zero density, so that the prediction map is well defined.}
\end{proposition}
\begin{proof}
Let $p(\vy_t|\mX_t, D_c)$ denote the finite dimensional density of $\pi_P(D_c)$ at index set $\mX_t$. To show that $\pi_P(T_{\bm{\tau}}D_c) = T_{\bm{\tau}}\pi_P(D_c)$ it suffices to show that $p(\vy_t|\mX_t, T_{\bm{\tau}}D_c) = p(\vy_t|T_{-\bm{\tau}}\mX_t, D_c)$. We have
\begin{align*}
    p(\vy_t|\mX_t, T_{\bm{\tau}}D_c) &= \frac{p(\vy_t, \vy_c| \mX_t, T_{\bm{\tau}}\mX_c)}{p(\vy_c|T_{\bm{\tau}}\mX_c)}\\
    &= \frac{p(\vy_t, \vy_c| T_{\bm{-\tau}}\mX_t, \mX_c)}{p(\vy_c|\mX_c)}\\
    &= p(\vy_t|T_{-\bm{\tau}}\mX_t, D_c),
\end{align*}
where we used the stationarity assumption in the second line. \qedhere
\end{proof}

\section{Description and Pseudocode for ConvCNP and ConvNP}
\label{app:pseudocode}
We provide additional details and pseudo-code for ConvCNP and ConvNP. 
Similar to \citet{gordon2020convolutional}, we distinguish between the ``on-the-grid'' and ``off-the-grid'' versions of the model. 
In our experiments, we use the ``off-the-grid'' version of the model for the 1d experiments in \cref{sec:1d_regression}, and the ``on-the-grid'' version for the image and environmental experiments in \cref{sec:image_completion,sec:climate_prediction}.

\subsection{ConvCNP Pseudo-Code and Details}
\label{sec:convcnp_pseudo_code}
\paragraph{Off-the-grid ConvCNP}
We begin by providing details for off-the-grid ConvCNP. 
As detailed in the main text, the encoder $\mathrm{E}_\vphi$ is defined by a ConvCNP, which provides a distribution over latent functions $z$.
In practice, we consider the discretized version, where we denote the grid of discretization locations as $(\vt_i)_{i=1}^K$, with $\vt_i \in \gX$.
Let $p_\vphi(\vz_i|\vt_i, D_c)$ denote the density of the latent function at the $i$th position, i.e.~at $\vz_i = z(\vt_i)$. 
Then in order to sample $\vz \sim \mathrm{E}_\vphi$ (as in e.g.~\cref{eqn:convnp_likelihod_fn} in the main body) we specify the density of the entire discretized latent function $\vz$ as:
\begin{align}
p_{\bm{\phi}}(\vz| D_c) = \prod_{i=1}^K p_{\bm{\phi}}(\vz_i| \vt_i, D_c) = \prod_{i=1}^K \mathcal{N}(\vz_i; \mu(\vt_i, D_c), \sigma^2(\vt_i, D_c)),
\end{align}
where $\mu$ and $\sigma^2$ are parametrized by ConvDeepSets \citep{gordon2020convolutional}. 

ConvDeepSets can be expressed as the composition of two functions.
Let $\Phi = \rho \circ \gamma$ be a ConvDeepSet.
$\gamma$ maps a data set $D$ to its functional representation via 
\[
    \gamma(D) = \sum_{(\vx, y) \in D} \phi(y) \psi(\,\mathord{\cdot} - \vx).
\]
Following \citet{gordon2020convolutional}, we set $\phi(y) = [1, y]^{\mathsf{T}} \in \mathbb{R}^2$, and $\psi$ to be a radial basis function. 
$\gamma(D)$ is itself discretized by evaluating it on a grid (which for simplicity we can also take to be $(\vt_i)_{i=1}^K$).

Next, $\rho$ maps the discretized $\gamma(D)$ to a continuous function, which we denote $f = \rho(\gamma(D))$.
$\gamma$ is itself implemented in two stages. 
First a deep CNN maps the discretized $\gamma(D)$ to a discretized output.
Second, this discrete output is mapped to a continuous function by using the CNN outputs as weights for evenly-spaced basis functions (again employing radial basis functions), which we denote by $\psi_{\rho}$. 

Whenever models output standard deviations, we enforce positivity via a function (e.g. the soft-plus function), which we denote $\text{pos}(\cdot)$.
Pseudo-code for a forward pass through an off-the-grid ConvCNP is provided in \cref{alg:convcnp_off_the_grid}.
Note the forward pass involves the computation of a \emph{density channel} $\vh^{(0)}$, whose role intuitively is to allow the model to know where it has observed datapoints.
This is discussed further in \citet{gordon2020convolutional}.
\begin{algorithm}[t]
\caption{Forward pass through ConvCNP (off-the-grid)}
\label{alg:convcnp_off_the_grid}
    \begin{algorithmic}[1]
    \Require $\rho = (\text{CNN}, \psi_\rho)$, $\psi$, and density $\zeta$ 
    \Require context $(\vx_n, y_n)_{n=1}^N$, target $(\vx^\ast_m)_{m=1}^M$ 
    \State $\text{lower, upper} \leftarrow \text{range}\!\left( (\vx_n)_{n=1}^N \!\cup\! (\vx^\ast_m)_{m=1}^M \right)$
    \State $(\vt_i)_{i=1}^K \leftarrow \text{uniform\_grid(lower, upper} ; \gamma)$
    \State $\vh_i \leftarrow \sum_{n=1}^N \begin{bmatrix} 1 & y_n  \end{bmatrix}^\top \psi(\vt_i - \vx_n)$ 
    \State $\vh^{(1)}_i \leftarrow \vh^{(1)}_i / \vh^{(0)}_i$ 
    \State $(f_\mu(\vt_i), f_\sigma(\vt_i))_{i=1}^T \leftarrow \textsc{CNN}((\vt_i, \vh_i)_{i=1}^T)$
    \State $\vmu_m \leftarrow \sum_{i=1}^K f_\mu(\vt_i) \psi_\rho(\vx^\ast_m - \vt_i)$ 
    \State $\vsigma_m \leftarrow \sum_{i=1}^K \text{pos}(f_\sigma(\vt_i)) \psi_\rho(\vx^\ast_m - \vt_i)$ \\
    \Return $(\vmu_m, \vsigma_m)_{m=1}^M$
\end{algorithmic}
\end{algorithm}

\paragraph{On-the-grid ConvCNP}
Next, we describe the ConvCNP for on-the-grid data, which is used in our image and environmental experiments. This version is simpler to implement in practice, and is applicable whenever the input data is confined to a regular grid.
As in \citet{gordon2020convolutional} we choose the discretization $(\vt_i)_{i=1}^K$ to be the pixel locations. 

Let $ \mathrm{I} \in \R^{H \times W \times C}$ be an image of dimensions $H, W, C$ (height, width, and channels, respectively).
We define a mask $\mathrm{M}_c$, which is such that $[\mathrm{M}_c]_{i,j} = 1$ if pixel location $(i,j)$ is in the context set, and $0$ otherwise.
Masking an image is then achieved via element-wise multiplication, denoted $\mathrm{M}_c \odot \mathrm{I}$.
This allows us to flexibly define context and target sets for an image (target sets are typically considered as the complete image, so the masks $\mathrm{M}_c$ are simply binary-valued tensors with the same dimensions as the image).
In this setting, we implement $\phi$, by selecting the context points, and prepend the context mask: $\phi = [\mathrm{M}_c, \mathrm{Z}_c]^\top$. 
We then implement $\gamma$ by a simple convolutional layer, which we denote $\textsc{conv}_\vtheta$ to emphasize that we use a standard 2d convolutional layer.
Full pseudo-code for the on-the-grid ConvCNP is provided in \cref{alg:convcnp_on_the_grid}.
\begin{algorithm}[t]
\caption{ConvCNP Forward pass (on-the-grid)}
\label{alg:convcnp_on_the_grid}
    \begin{algorithmic}[1]
    \Require $\rho = (\text{CNN}, \psi_\rho)$ and $\textsc{conv}_\vtheta$
    \Require image $\mathrm{I}$, context $\mathrm{M}_c$, and target mask $\mathrm{M}_t$
    \State We discretize at the pixel locations.
    \State $\mathrm{I}_c \leftarrow \mathrm{M}_c \odot \mathrm{I}$
    \State $\vh \leftarrow \textsc{conv}_\vtheta ([\mathrm{M}_c, \mathrm{I}_c]^\top ) $
    \State $\vh^{(1:C)} \leftarrow \vh^{(1:C)} / \vh^{(0)} $
    \State $f_t \leftarrow \mathrm{M}_t \odot \text{CNN}(\vh)$
    \State $\vmu \leftarrow f_t^{(1:C)}$ 
    \State $\vsigma \leftarrow \text{pos}( f_t^{(C+1:2C)})$\\
    \Return $(\vmu, \vsigma)$
\end{algorithmic}
\end{algorithm}

\subsection{Pseudo-Code for the ConvNP}
\label{app:convnp_pseudo_code}

The ConvNP can be implemented very simply by passing samples from the ConvCNP through an additional CNN decoder, which we denote $d_\vtheta$.
For an ``off-the-grid'' ConvNP, similarly to the ConvCNP, we must map the output of a standard CNN back to functions on a continuous domain $\gX$.
This can be achieved via an RBF mapping, similar to the off-the-grid ConvCNP, e.g. \cref{alg:convcnp_off_the_grid} lines 6, 7.
Pseudo-code for off- and on-the-grid ConvNPs are provided in \cref{alg:convnp_off_the_grid,alg:convnp_on_the_grid}, respectively.
\begin{algorithm}[t]
\caption{Forward pass through ConvNP (off-the-grid)}
\label{alg:convnp_off_the_grid}
    \begin{algorithmic}[1]
    \Require $d = (\text{CNN}, \psi_d)$, $\mathrm{E}_\vphi$ (off-the-grid ConvCNP), and number of samples $L$
    \Require context $(\vx_n, y_n)_{n=1}^N$, target $(\vx^\ast_m)_{m=1}^M$
    
    \State $\vmu_z, \vsigma_z \leftarrow \mathrm{E}_\vphi(D_c)$
    \For{$l = 1, \hdots, L$}
        \State $\vz_l \sim \gN(\vz; \vmu_z, \vsigma_z^2)$
        \State $(f_\mu(\vt_i), f_\sigma(\vt_i))_{i=1}^K \leftarrow \text{CNN}(\vz_l)$
        \State $\vmu_{m,l} \leftarrow \sum_{i=1}^T f_\mu(\vt_i) \psi_d(\vx^\ast_m - \vt_i)$ 
        \State $\vsigma_{m,l} \leftarrow \text{pos}\left(f_\sigma(\vt_i)\right)$
    \EndFor \\
    \Return $(\vmu, \vsigma)$
\end{algorithmic}
\end{algorithm}
\begin{algorithm}[t]
\caption{Forward pass through ConvNP (on-the-grid)}
\label{alg:convnp_on_the_grid}
    \begin{algorithmic}[1]
    \Require $d = \text{CNN}$, $\mathrm{E}_\vphi$ (on-the-grid ConvCNP), and number of samples $L$
    \Require image $\mathrm{I}$, context mask $\mathrm{M}_c$, and target mask $\mathrm{M}_t$
    \State $\vmu_z, \vsigma_z \leftarrow \mathrm{E}_\vphi(I, M_c)$
    \For{$l = 1, \hdots, L$}
        \State $\vz_l \sim \gN(\vz; \vmu_z, \vsigma_z^2)$
        \State $(f_\mu(\vt_i), f_\sigma(\vt_i))_{i=1}^K \leftarrow \text{CNN}(\vz_l)$
        \State $\vmu \leftarrow f_t^{(1:C)}$ 
        \State $\vsigma \leftarrow \text{pos}\left( f_t^{(C+1:2C)}\right)$
    \EndFor \\
    \Return $(\vmu, \vsigma)$
\end{algorithmic}
\end{algorithm}
Note that for the ConvNP, the discretization of the latent function $\vz$ is typically on a
pre-specified grid, and therefore lines 6 and 7 of \cref{alg:convcnp_off_the_grid} are unnecessary when calling the ConvCNP (\cref{alg:convnp_off_the_grid}, line 1).
%
% Finally, \cref{fig:convnp_diagram} provides an illustration of a forward pass through the ConvNP.
%
%The diagram was created using a context set drawn from an EQ kernel, and passed through a trained ConvNP.
%
% \begin{figure}
%     \centering
%     \includegraphics[width=\textwidth]{diagrams/ConvNP_v2.png}
%     \caption{Illustration of a forward pass through a trained ConvNP. }
%     \label{fig:convnp_diagram}
% \end{figure}

\section{Translation Equivariance of the ConvNP} \label{app:equivariance_proof}
We prove that the ConvNP is a translation equivariant map from data sets to stochastic processes, by proving that the decoder and encoder are separately translation equivariant. In this section we suppress the dependence on parameters $(\vphi, \vtheta)$.
\begin{lemma}\label{lem:decoder}
Let $d$ be a measurable, translation equivariant map from $(\mathbb{R}^\gX, \Sigma)$ to $(\mathbb{R}^\gX, \Sigma)$. The ConvNP decoder $\mathrm{D}: \gP(\gX) \to \gP(\gX)$, defined by $\mathrm{D}(P) = d_*(P)$, where $d_*(P)$ is the pushforward measure under $d$, is translation equivariant.
\end{lemma}

\begin{proof}
Let $F \in \Sigma$ be measurable. Then:
\begin{align*}
    \mathrm{D}(T_{\bm{\tau}}P)(F) &\stackrel{\text{(a)}}{=} T_{\bm{\tau}}P(d^{-1}(F)) \\
    &= P(T_{-\bm{\tau}}d^{-1}(F)) \\
    &\stackrel{\text{(b)}}{=} P(d^{-1}(T_{-\bm{\tau}}F)) \\
    &= \mathrm{D}(P)(T_{-\bm{\tau}}F) \\
    &= T_{\bm{\tau}}\mathrm{D}(P)(F).
\end{align*}
Here (a) follows from definition of the pushforward, and (b) follows because
\begin{align*}
    T_{-\bm{\tau}}d^{-1}(F) &= T_{-\bm{\tau}} \{ f : d(f) \in F \} \\
    &= \{ T_{-\bm{\tau}}f : d(f) \in F \}\\
    &= \{ f: d(T_{\bm{\tau}}f) \in F \}\\
    &= \{ f: T_{\bm{\tau}}d(f) \in F \}\\
    &= \{ f: d(f) \in T_{-\bm{\tau}}F \}\\
    &= d^{-1} (T_{-\bm{\tau}}F ) .\qedhere
\end{align*}
\end{proof}
\begin{lemma} \label{lem:ConvCNP_translation_equivariant}
The ConvNP encoder $\mathrm{E}$ (a ConvCNP), is a translation equivariant map from data sets to stochastic processes.
\end{lemma}
\begin{proof}
Recall that the mean and variance $\mu(\cdot, S), \sigma^2(\cdot, S)$ (viewed as maps from $\gS \to C_b(\gX)$) of the encoder $\mathrm{E}$ are both given by ConvDeepSets. Due to the translation equivariance of ConvDeepSets \citep[Theorem 1]{gordon2020convolutional}, $\mu(\cdot, T_{\bm{\tau}}S) = T_{\bm{\tau}}\mu(\cdot, S)$ for all $S, \bm{\tau}$, and similarly for $\sigma^2$.
Let $F \in \Sigma$. Then since the measure $\mathrm{E}(S) \in \gP_{\mathrm{N}}(\gX)$ is defined entirely by its mean and variance function, $\mathrm{E}(T_{\bm{\tau}}S)(F) = \mathrm{E}(S)(T_{-\bm{\tau}}F) = T_{\bm{\tau}}\mathrm{E}(S)(F)$. 
\end{proof}
Noting that a composition of translation equivariant maps is itself translation equivariant, we obtain the following proposition:
\begin{proposition}
    Define $\mathrm{ConvNP} = \mathrm{D} \circ \mathrm{E}$. Then $\mathrm{ConvNP}$ is a translation equivariant map from data sets to stochastic processes.
\end{proposition}

\section{Recovering the Prediction Map in the Infinite Data / Capacity Limits}
\label{app:log_el_convergence}
\paragraph{Task generation procedure}
Assume tasks $\xi = (D_c, D_t)$ are generated as follows: first, some finite number of input locations $\mX_t, \mX_c$ are sampled. 
Assume that $\mathrm{Pr}(|\mX_t| = n) > 0$ for all $n \in \mathbb{Z}_{\geq0}$, where $|\mX_t|$ denotes the number of datapoints in $\mX_t$, and assume the same is true of $\mathrm{Pr}(|\mX_c| = n)$. 
Further assume that for each $n > 0$, the distribution of $\mX$ given $|\mX| = n$ has a continuous density with support over all of $\mathbb{R}^{n \times d_{\mathrm{in}}}$. 
Next, we sample $\vy_t, \vy_c$ from the finite marginal of the ground truth stochastic process $P$, which has density $p(\vy_t, \vy_c|\mX_t, \mX_c)$. Finally, we set $(D_c, D_t) \coloneqq ((\mX_t, \vy_t), (\mX_c, \vy_c))$. 
\begin{proposition}
    Let $\Psi: \gS \to \gP(\gX)$ be any map from data sets to stochastic processes, and let $\gL_{\mathrm{ML}}(\Psi) \coloneqq \mathbb{E}_{p(\xi)}[ \log p_{\Psi}(\vy_t| \mX_t, D_c)]$, where the density $p_{\Psi}$ is that of $\Psi(D_c)$ evaluated at $\mX_t$.
    Then $\Psi$ globally maximises $\gL_{\mathrm{ML}}$ if and only if $\Psi = \pi_P$, the prediction map.
\end{proposition}

\begin{proof}
We have:
\begin{align}
    \gL_{\mathrm{ML}}(\Psi) &= \E_{p(D_c, \mX_t, \vy_t)} \left[ \log p_{\Psi}(\vy_t| \mX_t, D_c) \right] \\
    &= \E_{p(D_c, \mX_t)}\left[ \E_{p(\vy_t|\mX_t, D_c)} \left[ \log p_{\Psi}(\vy_t| \mX_t, D_c) \right] \right]\\
    &= -\E_{p(D_c, \mX_t)} \left[ \mathrm{KL}\left( p( \vy_t |\mX_t, D_c ) \middle\| p_{\Psi}(\vy_t| \mX_t, D_c) \right) \right] + \mathrm{constant}, \label{eqn:expected_KLs}
\end{align}
where the additive constant is constant with respect to $\Psi$.
First note that the KL-divergence is non-negative, and that the prediction map sends all the KL-divergences to zero, globally optimising $\gL(\Psi)$. 
Furthermore, the KL-divergence is equal to zero if and only if the two distributions are equal, and this must hold for all $\mX_t, D_c$. 
For, if this were not the case, the KL-divergence would contribute a non-zero amount to the expectation in \cref{eqn:expected_KLs}.
\end{proof}

Strictly speaking, this argument only shows that the finite marginals of the prediction map and $\Psi$ must be equal for almost all $(D_c , \mX_t)$ with respect to $p(D_c , \mX_t)$. Since the task generation procedure outlined in this section assumes a finite probability of generating any finite-sized context and target set, this is not very restrictive. However, in practice we often limit the maximum size of the sampled data sets, and also their range in $\gX$ space. Hence we can only expect the model to learn reasonable predictions within the ranges seen during train time.

\section{Relationship Between Neural Process and Maximum-Likelihood Objectives}
\label{app:neural_process_models}

Let $D \coloneqq D_t \cup D_c$, and let $Z = \int p_{\bm{\theta}}(\vy_t| \mX_t,  \vz) q_\vphi(\vz|D_c) \, \mathrm{d}\vz$. The NP objective is:
\begin{align}
\label{eqn:neural_process_objective}
    \mathcal{L}_{\mathrm{NP}}(\bm{\theta},\bm{\phi}; \xi) 
    &\coloneqq \mathbb{E}_{q_{\bm{\phi}}(\vz | D)} [\log p_{\bm{\theta}}(\vy_t| \mX_t, \vz)] - \mathrm{KL}(q_{\bm{\phi}} (\vz | D) \| q_{\bm{\phi}} (\vz | D_c))      \\
    &= \mathbb{E}_{q_{\bm{\phi}}(\vz | D)} [\log p_{\bm{\theta}}(\vy_t| \mX_t, \vz) +  \log q_{\bm{\phi}} (\vz | D_c) -  \log q_{\bm{\phi}} (\vz | D) ]\\
    &= \mathbb{E}_{q_{\bm{\phi}}(\vz | D)} \left[\log Z + \log \frac{p_{\bm{\theta}}(\vy_t| \mX_t, \vz)q_{\bm{\phi}} (\vz | D_c)}{Z} -  \log q_{\bm{\phi}} (\vz | D) \right]\\
    &=    \log  Z - \mathrm{KL}\left(q_{\bm{\phi}}(\vz|D)\middle\|\frac{1}{Z} p_{\bm{\theta}}(\vy_t| \mX_t,  \vz) q(\vz|D_c) \right).
\end{align}
If we identify the approximate posterior $q_{\bm{\phi}}$ with the encoder of the maximum-likelihood ConvNP, (which in the maximum-likelihood framework does not have an approximate inference interpretation), then $\log Z = \gL_{\mathrm{ML}}(\vtheta, \vphi; \xi)$.

\section{Effect of Number of Samples Used to Estimate Objective During Training and Evaluation}
\label{app:effect_of_L}

In this section we empirically examine the effect of $L$, the number of samples used to estimate likelihood bounds, on the training and evaluation of ConvNPs and ANPs. 

\subsection{Effect of Number of Samples Used for Evaluation}
\label{app:effect_num_samples_eval}

\input{plots/perf_v_samples/perf_v_samples}
As the true log-likelihoods of NP-based models are intractable, quantitative evaluation and comparison of models is challenging.
Instead, we compare models by using an estimate of the log-likelihood.
A natural candidate is $\gL_{\mathrm{ML}}$.
However, unless large $L$ is used, $\gL_{\mathrm{ML}}$ is conservative and tends to significantly underestimate the log-likelihood.
One way to improve the estimate of $\gL_{\mathrm{ML}}$ is through importance weighting (IW) \citep{wu2016quantitative,le2018empirical}.
Denoting $D = D_c \cup D_t$, the encoder $\mathrm{E}_\vphi(D)$ can be used as a proposal distribution:
\begin{equation}
\label{eqn:iw_estimator}
    \hat\gL_{\mathrm{IW}} (\vtheta, \vphi; \xi) \coloneqq \log \left( \frac{1}{L} \sum_{l=1}^L \exp \left( \log w(\vz_l) + \sum_{(\vx, y) \in D_t} \log p_\vtheta(y | \vx, \vz_l) \right) \right), \quad \vz_l \sim \mathrm{E}_\vphi(D),
\end{equation}
where the importance weights are given by $\log w(\vz_l) \coloneqq \log q_\vphi(\vz|D_c) - \log q_\vphi(\vz|D)$.
Here $q_\vphi(\vz|D)$ is the density of the encoder distribution.
We find that training models with $\gL_{\mathrm{ML}}$ results in encoders that are ill-suited as proposal distributions, so we only use $\gL_{\mathrm{IW}}$ to evaluate models trained with $\gL_{\mathrm{NP}}$.

\cref{fig:perf_v_samples} demonstrates the effect of the number of samples $L$ used to estimate the evaluation objective for the ConvNP and ANP trained with $\gL_{\textrm{ML}}$ and $\gL_{\textrm{NP}}$.
The models used to generate \cref{fig:perf_v_samples} are the same models used in \cref{sec:1d_regression}, i.e.\ having heteroskedastic noise.
Observe the general trend that the log-likelihood estimates tend to increase with $L$, as expected.
The ANP trained with $\gL_{\mathrm{NP}}$ collapsed to a conditional ANP, meaning that the encoder became deterministic;
in that case, $\gL_{\mathrm{ML}}$ is exact, which means that larger $L$ and importance weighting will not increase the estimate.
In contrast, the ANP trained with $\gL_{\mathrm{ML}}$ did not collapse, and we see that there the estimate increases with $L$.
For the ConvNP trained with $\gL_{\mathrm{NP}}$, evaluating with $\gL_{\mathrm{IW}}$ yields a significant increase, showing that the bound estimated with $\gL_{\mathrm{IW}}$ is very loose.
The models trained with $\gL_{\mathrm{ML}}$ tend to be the best performing, although the ConvNP trained with $\gL_{\mathrm{NP}}$ is best for weakly periodic kernel and appears to still be increasing with $L$.

In both the main and the supplement, all log-likelihood lower bounds reported are computed with $\gL_{\mathrm{ML}}$ if the model was trained using $\gL_{\mathrm{ML}}$ and with $\gL_{\mathrm{IW}}$ if the model was trained using $\gL_{\mathrm{NP}}$. 

\subsection{Effect of Number of Samples Used During Training}
\begin{figure}[th]
    \centering
    \includegraphics[width=0.8\linewidth]{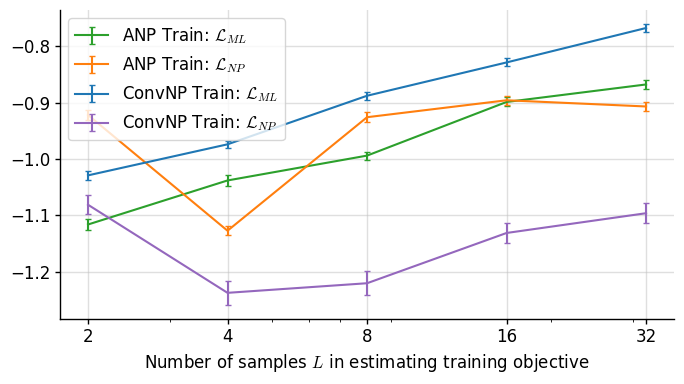}
    \caption{
        Interpolation performance (within training range) for context set sizes uniformly sampled from $\{0, \ldots, 50\}$ of the ConvNP and ANP on Mat\'ern--$\frac52$ samples.
        The models are trained with $\gL_{\mathrm{ML}}$ and $\gL_{\mathrm{NP}}$ for various number of samples $L$. Models trained with $\gL_{\mathrm{ML}}$ are evaluated with $\gL_{\mathrm{ML}}$, while models trained with $\gL_{\mathrm{NP}}$ are evaluated with $\gL_{\mathrm{ML}}$. At evaluation, all bounds are estimated using 2,048 samples.
    }
    \label{fig:effect_of_L_training}
\end{figure}
\Cref{fig:effect_of_L_training} shows the effect of the number of samples $L$ in the training objectives on the performance of the ConvNP and ANP.
Observe that the performance of $\gL_{\mathrm{ML}}$ reliably increases with the number of samples $L$ and that $\gL_{\mathrm{ML}}$ outperforms $\gL_{\mathrm{NP}}$.
The performance for $\gL_{\mathrm{NP}}$ does not appear to increase with the number of samples $L$ and appears more noisy than $\gL_{\mathrm{ML}}$.
Note that the models used for \cref{fig:effect_of_L_training} were trained with homoskedastic observation noise.
This is achieved by pooling $f_\sigma$ over the time dimension.

\section{Experimental Details on 1D Regression}
\label{app:details_1D}

For the full results of the 1D regression tasks, see \cref{app:additional_results_1d}.
Code to reproduce the 1D regression experiments can be found at \url{https://github.com/wesselb/NeuralProcesses.jl}.

In the 1D regression experiments, we consider the following generative processes:
\begin{enumerate}[leftmargin=8em,itemsep=0.5em]
    \item[EQ:]
        samples from a Gaussian process with the following exponentiated-quadratic kernel:
        \[
            k(t, t') = \exp\left(-\frac18(t - t')^2\right);
        \]
    \item[Mat\'ern--$\frac52$:]
        samples from a Gaussian process with the following Mat\'ern--$\frac52$ kernel:
        \[
            k(t, t') = \left(1 + 4\sqrt{5} d  + \frac53 d^2\right) \exp\left(-\sqrt{5} d \vphantom{\frac52}\right)
        \]
        with $d = 4|x - x'|$;
    \item[noisy mixture:]
        samples from a Gaussian process with the following noisy mixture kernel:
        \[
            k(t, t') =
                \exp\left(-\frac18(t - t')^2\right) +
                \exp\left(-\frac12(t - t')^2\right) +
                10^{-3}\delta[t - t'];
        \]
    \item[weakly periodic:]
        samples from a Gaussian process with the following weakly-periodic kernel:
        \[
            k(t, t') = \exp\left(-\frac12 (f_1(t) - f_1(t'))^2 -\frac12 (f_2(t) - f_2(t'))^2 - \frac18(t-t')^2\right)
        \]
        with $f_1(t) = \cos(8 \pi t)$ and $f_2(t) = \sin(8\pi t)$; and 
    \item[sawtooth:]
        samples from the following sawtooth process:
        \[
            f(t) = \frac{A}{2} - \frac{A}{\pi}\sum_{k=1}^K (-1)^k \frac{\sin(2\pi k f (t - s))}{k}
        \]
        with $A = 1$, $f \sim \mathcal{U}[3, 5]$, $s \sim \mathcal{U}[-5, 5]$, and $K \in \{10, \ldots, 20\}$ chosen uniformly.
\end{enumerate}

We compare the following models, where all activation functions are leaky ReLUs with leak $0.1$:
\begin{enumerate}[leftmargin=8em,itemsep=0.5em]
    \item [ConvCNP:]
        The first model is the ConvCNP.
        The architecture of the ConvCNP is equal to that of the encoder in the ConvNP, described next.
    \item [ConvNP:]
        The second model is the ConvNP as described in the main body.
        The functional embedding uses separate length scales for the data channel and density channel (\cref{fig:convnp_full_diagram}), which are initialized to twice the inter-point spacing of the discretization and learned during training.
        The discretization uniformly ranges over $[\min(x) - 1, \max(x) + 1]$ at density $\rho=64$ points per unit, where $\min(x)$ is the minimum $x$ value occurring in the union of the context and target sets in the current batch and $\max(x)$ is corresponding maximum $x$ value.
        The discretization is passed through a 10-layer (excluding an initial and final point-wise linear layer) CNN with $64$ channels and depthwise-separable convolutions.
        The width of the filters depends on the data set and is chosen such that the receptive field sizes are as follows:
        \begin{enumerate}[leftmargin=8em]
            \item[EQ:] $2$,
            \item[Mat\'ern--$\frac52$:] $2$,
            \item[noisy mixture:] $4$,
            \item[weakly periodic:] $4$,
            \item[sawtooth:] $16$.
        \end{enumerate}
        The discretized functional representation consists of 16 channels.
        The smoothing at the end of the encoder also has separate length scales for the mean and variance which are initialized similarly and learned.
        The encoder parametrizes the standard deviations by passing the output of the CNN through a softplus.
        The decoder has the same architecture as the encoder.
    \item [ANP:]
        The third model is the Attentive NP with latent dimensionality $d = 128$ and $8$-head dot-product attention \citep{vaswani2017attention}.
        In the attentive deterministic encoder, the keys ($t$), queries ($t$), and values (concatenation of $t$ and $y$) are transformed by a three-layer MLP of constant width $d$.
        The dot products are normalised by $\sqrt{d}$.
        The output of the attention mechanism is passed through a constant-width linear layer, which is then passed through two layers of layer normalization \citep{ba2016layer} to normalise the latent representation.
        In the first of these two layers, first the transformed queries are passed through a constant-width linear layer and added to the input.
        In the second of these two layers, the output of the first layer is first passed through a two-layer constant-width MLP and added to itself, making a residual layer.
        In the stochastic encoder, the inputs and outputs are concatenated and passed though a three-layer MLP of constant width $d$.
        The result is mean-pooled and passed through a two-layer constant-width MLP.
        The decoder consists of a three-layer MLP of constant width $d$.
    \item [NP:]
        The fourth model is the original NP \citep{garnelo2018neural}.
        The architecture is similar to that of the ANP, where the architecture of the deterministic encoder is replaced by that of the stochastic encoder.
\end{enumerate}
For all models, positivity of the observation noise is enforced with a softplus function.
Parameter counts of the ConvCNP, ConvNP, ANP, and NP are listed in \cref{tab:1D_parameter_counts}.

\input{tables/parameter_counts_1d}

The models are trained with $\gL_{\mathrm{ML}}$ ($L=20$) and $\gL_{\mathrm{NP}}$ ($L=5$).
For $\gL_{\mathrm{NP}}$, the context set is appended to the target set when evaluating the objective.
The models are optimised using ADAM with learning rate $5\cdot 10^{-3}$ for $100$ epochs.
One epoch consists of $2^{14}$ tasks divided into batches of size $16$.
For training, the inputs of the context and target sets are sampled uniformly from $[-2, 2]$.
The size of the context set is sampled uniformly from $\{0,\ldots,50\}$ and the size of the target set is fixed to 50.
To encourage the NP-based models---not the CNP-based models---to fit and not revert to their conditional variants, the observation noise standard deviation $\sigma$ is held fixed to $10^{-2}$ for the first 20 epochs.

For evaluation, the size of the context set is sampled uniformly from $\{0,\ldots,10\}$, and the losses are evaluated with $L=5000$ and batch size one.
To test interpolation within the training range, the inputs of the context and target sets are, like training, sampled uniformly from $[-2, 2]$.
To test interpolation beyond the training range, the inputs of the context and target sets are sampled uniformly from $[2, 6]$.
To test extrapolation beyond the training range, the inputs of the context sets are sampled uniformly from $[-2, 2]$ and the inputs of the target sets are sampled uniformly from $[-4, -2]\cup[2,4]$.
As described in \cref{app:effect_num_samples_eval}, models trained with $\gL_{\mathrm{NP}}$ are evaluated using importance weighting to obtain a better estimate of the evaluation loss.

\section{Additional Results on 1D Regression}
\label{app:additional_results_1d}

\Cref{table:1d_bakeoff_amortized} presents results for all models with all losses on all data sets described in \cref{app:details_1D} according to the evaluation protocol described in \cref{app:effect_num_samples_eval,app:details_1D}.

\input{tables/bakeoff1d_amortised}

\section{Experimental Details on Image Completion} 
\label{app:details_images}

\subsection{Data Details} 
\input{plots/image_completion/data/zsmm_data}

We use three standard data sets throughout our image experiments: SVHN \citep{netzer2011reading}, MNIST \cite{lecun1989backpropagation}, and $32\times 32$ CelebA \cite{netzer2011reading}.
The aforementioned standard data sets all contain only a single, well-centered object.
To evaluate the translation equivariance and generalization capabilities of our model we evaluate on a Zero Shot Multi-MNIST (ZSMM) task, which is similar to ZSMM described in Appendix D.2 of \cite{gordon2020convolutional}.
Namely, we generate a test set by randomly sampling with replacement 10000 pairs of digits from the MNIST test set, place them on a black $56 \times 56$ background, and translate the digits in such a way that the digits can be arbitrarily close but cannot overlap (\Cref{fig:zsmm_test}). 
The difference with the dataset from \citet{gordon2020convolutional}, is that the training set consists of the standard MNIST digits (instead of a single digit placed in the center of $56\times 56$ canvas), augmented by up to 4 pixel shifts (\Cref{fig:zsmm_train}).
The model thus has to generalize both to a larger canvas size as well as to seeing multiple digits.

For all data sets, pixel values are divided by 255 to rescale them to the $[0,1]$ range.
We evaluate on predefined test splits when available (MNIST, SVHN, ZSMM) and make our own test set for CelebA by randomly selecting $10\%$ of the data.
For each dataset we also set aside $10\%$ of the training set as validation.

\subsection{Training Details}

In all experiments, we sample the number of context pixels uniformly from $\mathcal{U}(0, \frac{n_{\text{total}}}{2})$, and the number of target points is set to $n_{\text{total}}$.
The weights are optimized using Adam \citep{kingma2014adam} with learning rate $5\times 10^{-4}$.
We use a maximum of $100$ epochs, with early stopping --- based on log likelihood on the validation set --- of 10 epochs patience.
Unless stated otherwise, we use $L=16$ samples from the latent function during training, and $L=128$ at test time.
We clip the $L2$ norm of all gradients to 1, which was particularly important for ConvNP.
We use a batch size of 32 for all models besides ANP trained on ZSMM which used a batch size of 8 due to memory constraints.

\subsection{Architecture Details}

\paragraph{General architecture details}
For all models, we follow \citet{le2018empirical} and process the predicted standard deviation of the latent function $\vsigma_z$ using a sigmoid and the standard deviation $\vsigma$ of the predictive distribution using lower-bounded softplus: 
\begin{align}
\vsigma_z &= 0.001 + (1-0.001) \frac{1}{1+ \exp(f_{\sigma,z} )}  \label{eq:sigmaz_process} \\
 \vsigma &= 0.001 + (1-0.001) \ln (1 + \exp(f_{\sigma}))    \label{eq:sigma_process} 
\end{align}
As the pixels are rescaled to $[0,1]$, we also process the mean of the posterior predictive (conditioned on a single sample) to be in $[0,1]$ using  a logistic function 
\begin{equation}
\vmu =  \frac{1}{1+ \exp(-f_{\mu})} \label{eq:mmu_process} 
\end{equation}
In the following, we describe the architecture of ANP and ConvNP. Unless stated otherwise, all vectors in the following paragraphs are in $\mathbb{R}^{128}$ and all MLPs have 128 hidden units.

\paragraph{ANP details}
We provide details for the ANP trained with $\gL_{\mathrm{ML}}$.
As the ANP cannot take advantage of the fact that images are on the grid, we preprocess each pixel so that $\mathbf{x} \in [-1,1]^2$.
The only exception being for the test set of ZSMM, where $\mathbf{x} \in [-\frac{56}{32},\frac{56}{32}]^2$ as the model is trained on $32 \times 32$ but evaluated on $56 \times 56$ images.
Each context feature is first encoded $\mathbf{x}^{(c)} \mapsto \mathbf{r}_{x}^{(c)}$ by a single hidden layer MLP, while a second single hidden layer MLP encodes values
$\mathbf{y}^{(c)} \mapsto \mathbf{r}_{y}^{(c)}$. 
We produce a representation $\mathbf{r}_{xy}^{(c)}$ by summing both representations $\mathbf{r}_{x}^{(c)} + \mathbf{r}_{y}^{(c)}$ and passing them through two self-attention layers \citep{vaswani2017attention}.
Following \citet{parmar2018image}, each self-attention layer is implemented as 8-headed attention, a skip connection, and two layer normalizations \citep{ba2016layer}.
To predict values at each target point $t$, we embed $\mathbf{x}^{(t)} \mapsto \mathbf{r}_x^{(t)}$ using the hidden layer MLP used for $\mathbf{r}_x^{(c)}$.
A deterministic target representation $\mathbf{r}_{xy}^{(t)}$ is then computed by applying cross-attention (using an 8-headed attention described above) with keys $\mathrm{K} \coloneqq \{\mathbf{r}_x^{(c)}\}_{c=1}^C$, values $\mathrm{V} \coloneqq \{\mathbf{r}_{xy}^{(c)}\}_{c=1}^C$, and query $\mathbf{q} \coloneqq \mathbf{r}_{x}^{(t)}$.
For the latent path, we average over context representations $\mathbf{r}_{xy}^{(c)}$, and pass the resulting representation through a single hidden layer MLP that outputs $(\vmu_z, \vsigma_{z}) \in \mathbb{R}^{256}$.
$\vsigma_{z}$ is made positive by post-processing it using \cref{eq:sigmaz_process}.
We then sample (with reparametrization \citep{kingma2013auto}) $L$ latent representation $\vz_l \sim \gN(\vz; \vmu_z, \vsigma_z^2)$.

We describe the remainder of the forward pass for a single $\vz_l$, though in practice multiple samples may be processed in parallel.
The deterministic and latent representations of the context set are concatenated, and the resulting representation is passed through a linear layer $[\mathbf{r}_{xy}^{(t)}; \vz_l] \to \mathbf{r}_{xyz}^{(t)} \in \mathbb{R}^{128}$.
Given the target and context-set representations, the predictive posterior is given by a Gaussian pdf with diagonal covariance parametrised by $(\vmu^{(t)}, \vsigma_{\text{pre}}^{(t)}) = \mathrm{decoder}([\mathbf{r}_x^{(t)}; \mathbf{r}_{xyz}^{(t)}])$ where $\vmu^{(t)}, \vsigma_{\text{pre}}^{(t)} \in \mathbb{R}^3$ and $\mathrm{decoder}$ is a 4 hidden layer MLP.
Finally, the $\vsigma^{(t)}$ is processed by \cref{eq:sigma_process}  using \cref{eq:mmu_process}.
In the case of MNIST and ZSMM, $\vsigma^{(t)}$ is also spatially mean pooled, which corresponds to using homoskedastic noise.
This improves the qualitative performance by forcing ANP and ConvNP to model the digit instead of focusing on predicting the black background with high confidence.
\citet{kim2018attentive} did not suffer from that issue as they used a much larger lower bound for \cref{eq:sigma_process}.

\paragraph{ConvNP details} 
The core algorithm of on-the-grid ConvNP is outlined in \cref{alg:convnp_on_the_grid} as well as \cref{alg:convcnp_on_the_grid}. 
Here we discuss the parametrizations used for each step of the algorithm.
All convolutional layers are depthwise separable \citep{chollet2017xception}.
$\textsc{conv}_\vtheta$ is a convolutional layer with kernel size of 11 (no bias).
Following \citet{gordon2020convolutional}, we enforce positivity on the weights in the first convolutional layer by only convolving their absolute value with the signal.

The $\textsc{CNN}$s are ResNets \cite{he2016deep} with 9 blocks, where each convolution has a kernel size of 3.
Each residual block consists of two convolutional layers, pre-activation batch normalization layers \cite{ioffe2015batch}, and ReLU activations.
The output of the pre-latent CNN (CNN in \cref{alg:convcnp_on_the_grid}) goes through a single hidden layer MLP that outputs $(\vmu_z, \vsigma_z) \in \mathbb{R}^{256}$.
As with ANP, $f_{\sigma,z}$ is processed by \cref{eq:sigmaz_process} and then used to sample (with reparametrization \citep{kingma2013auto}) $L$ latent functions $\mathbf{Z}_l$.
Importantly, we found that the coherence of samples improves if the model uses a \emph{global representation} in addition to the the pixel dependent representation.
We achieve this by mean-pooling half of the functional representation.
Namely, we replace $\vz_l$ by the channel-wise concatenation of $\vz_l^{(1:64)}$ and $\textsc{mean}( \vz_l^{(65:128)})$, where the mean is taken over the spatial dimensions.
This latent function then goes through the post-latent CNN (CNN in \cref{alg:convnp_on_the_grid}), as well as a linear layer to output $(f_{\mu}, f_{\sigma}) \in \mathbb{R}^{256}$.
As for ANP $f_{\mu}$ is processed by \cref{eq:mmu_process} and $f_{\sigma}$ is re-scaled with \cref{eq:sigma_process} and is spatially pooled in the case of MNIST and ZSMM to obtain homoskedastic noise.

\section{Additional results on image completion.} 
\label{app:additional_images}
We provide additional qualitative samples and quantitative analyses for the ConvNP and ANP.
\input{plots/image_completion/samples_convnp/samples_convnp_mixed.tex}

\paragraph{Additional ConvNP samples}
\Cref{fig:samples_convnp} provides further samples from a ConvNP trained with $\gL_{\mathrm{ML}}$.
We observe that the ConvNP produces reasonably diverse yet coherent samples when evaluated in a regime that resembles the training regime (in the first four sub-columns of MNIST, SVHN, and CelebA).
However, \Cref{fig:samples_convnp} also demonstrates that the ConvNP struggles with context sets that are significantly different from those seen during training.

\paragraph{Further comparisons of ANP and ConvNP}
\input{plots/image_completion/samples_anp_convnp/samples_anp_convnp_mixed}
We provide further qualitative comparisons of ConvNPs, ANPs trained with $\gL_{\mathrm{ML}}$, and ANPs trained with $\gL_{\mathrm{NP}}$.
We omit ConvNPs trained with $\gL_{\mathrm{NP}}$ as these are significantly outperformed by ConvNPs trained with $\gL_{\mathrm{ML}}$ (see e.g.~\Cref{table:images_main_text}).

\Cref{fig:samples_anp_convnp} shows that all models perform relatively well when context sets are drawn from a similar distribution as employed during training (first four sub-columns of MNIST, SVHN, and CelebA).
Furthermore, we observe that samples from the ConvNP prior tend to be closer to samples from the underlying data distribution (e.g. for CelebA).

The qualitative advantage of ConvNP is most significant in settings that require translation equivariance for generalization.
\cref{fig:samples_anp_convnp} row 2 (ZSMM) clearly demonstrates that ConvNP generalizes to larger canvas sizes and multiple digits, while ANP attempts to reconstruct a single digit regardless of the context set.
Finally, \cref{fig:samples_anp_convnp_kde} provides the test log-likelihood distributions of ANP and ConvNP as well as some qualitative comparisons between the two.
\input{plots/image_completion/samples_anp_convnp_kde/samples_anp_convnp_kde_mixed}

\section{Experimental Details on Environmental Data} 
\label{app:details_environmental}

\subsection{Data Details}

\input{tables/region_coordinates}
\input{plots/europe_regions}

ERA5-Land \citep{era5land} contains high resolution information on environmental variables at a 9 km spacing across the globe.\footnote{URL: \href{https://www.ecmwf.int/en/era5-land}{https://www.ecmwf.int/en/era5-land}. Neither the European Commission nor ECMWF is responsible for any use that may be made of the Copernicus Information or data it contains.} 
The data we use contains daily measurements of accumulated precipitation at 11pm and temperature at 11pm at every location, between 1981 and 2020, yielding a total of 14,304 temporal measurements across the spatial grid.
In addition, we provide orography (elevation) values for each location.
We normalize the data such that the precipitation values in the train set have zero mean and unit standard deviation.

We consider the task of predicting daily precipitation $y$, with latitude and longitude as $\vx$. 
In addition, at each context and target location, we provide the model with access to side information in the form of orography (elevation) and temperature values.
We also normalize the orography and temperature values to have zero mean and unit standard deviation.
We choose a large region of central Europe as our train set, and use regions East, West and South of the train set as held out test sets (see \cref{fig:eu_data_regions,table:era5_region_coordinates}). 
At train time, to sample a task, we first sample a random date between 1981 and 2020. 
We then sample a square subregion of grid of values from within the train region (which has size $61 \times 201$). 
We consider two models, one trained on $28 \times 28$ subregions, and another trained on $40 \times 40$ subregions.
During training, each subregion is then split into context and target sets. Context points are randomly chosen with a keep rate $p_{\mathrm{keep}}$ with $p_{\mathrm{keep}} \sim \gU[0, 0.3]$. In this section, we train only on the $\gL_{\mathrm{ML}}$ objective.

\subsection{Gaussian Process Baseline} 
\label{app:environmental_gp_baselines}
We mean-centre the data for each task for the GP before training, and add the mean offset back for evaluation and sampling.
We use an Automatic Relevance Determination (ARD) kernel, with separate factors for latitude/longitude, temperature and orography. 
In detail, let $\vx = (x_{\mathrm{lat}}, x_{\mathrm{lon}})$ denote position, and let $\omega, t$ denote orography and precipitation respectively, and let $\vr \coloneqq (\vx, \omega, t)$. 
Then the kernel is given by
\begin{align*}
    k(\vr, \vr') = \sigma_v^2 k_{l}(\vx, \vx') k_{\omega}(\omega, \omega')  k_{t}(t, t') + \sigma_n^2 \delta(\vr, \vr').
\end{align*}
Here each of $k_l, k_\omega$ and $k_t$ are Mat\'ern--$\frac52$ kernels with separate learnable lengthscales; $\delta(\vr, \vr') = 1$ if $\vr = \vr'$ and $0$ otherwise; and $\sigma_v^2, \sigma_n^2$ are learnable signal and noise variances respectively.
We learn all hyperparameters by maximising the log-marginal likelihood using Scipy's implementation of L-BFGS. 

\paragraph{Transforming the data} As the data is non-negative, we considered applying the transform $y \mapsto \log(\epsilon + y)$ for the GP to model. If $\epsilon = 0$, this would guarantee that the GP would only yield positive samples, which would be physically sensible as precipitation is non-negative. However, this cannot be done as precipitation often takes the value $y=0$, which would lead to the transform being undefined. On the other hand, if $\epsilon > 0$, the GP samples after performing the inverse transform could still predict a precipitation value as low as $- \epsilon$, which is still unphysical. Further, a small value of $\epsilon$ leads to large distortion of the $y$ values in transformed space. In the end, we run all experiments for the GP and NP without log-transforming the data; hence the models have to learn non-negativity.

\subsection{ConvNP Architecture and Training Details}
\label{app:environmental_convnp_details}
As the ERA5-Land dataset is regularly spaced, we use the on-the-grid version of the architecture, without the need for an RBF smoothing layer at the input (see \cref{app:pseudocode}). 
All experiments used a convolutional architecture with 3 residual blocks \citep{he2016deep} for the encoder and 3 residual blocks for the decoder.
Each residual block is defined with two layers of ReLU activations followed by convolutions, each with kernel size 5. 
The first convolution in each block is a standard convolution layer, whereas the second is depthwise separable \citep{chollet2017xception}.
All intermediate convolutional layers have 128 channels, and the latent function $\vz$ has 16 channels.
The networks were trained using ADAM with a learning rate of $10^{-4}$. 
We used 16 channels for the latent function $\vz$, and estimated $\gL_{\mathrm{ML}}$ using 16-32 samples at train time, with batches of 8-16 images. 

We train the models for between 400 and 500 epochs, where each epoch is defined as a single pass through each day in the training set, where at each day, a random subregion of the full $61 \times 201$ central Europe region is cropped. We estimated the predictive density using 2500 samples of $\vz$ during test time. 

\subsection{Prediction and Sampling}
To create \cref{table:era5_loglik_rmse}, at test time we sample $28 \times 28$ subregions from each of the train and test regions. This is done 1000 times.
For the GP, we randomly restart optimisation 5 times per task and use the best hyper-parameters found.
In order to remove outliers where the GP has very poor likelihood, we set a log-likelihood threshold for the GP. 
If the GP has a log-likelihood of less than 0 nats on a particular task, then that task is removed from the evaluation.

We find that to produce high quality samples, we need to train the model on subregions that are roughly as large as the lengthscale of the precipitation process.
Hence we sample from the model trained on $40 \times 40$ subregions in \cref{fig:precipitation_example_task} in the main body.
We show samples from the model trained on both $28 \times 28$ subregions and $40 \times 40$ subregions in \cref{app:environmental_figures}.
We also compare to samples from GPs trained on each context set (no random restarts were used for sampling).

\subsection{Bayesian Optimization}
\label{app:bayesian_optimization}
We use the models described in \cref{app:environmental_convnp_details}, trained on random $28 \times 28$ subregions of the train region, and compare to the GP baselines described in \cref{app:environmental_gp_baselines}.
For the Bayesian optimization experiments in \cref{fig:bayesopt} in the main body, we do not perform random restarts as this was too time-consuming.
We carry out the Bayesian optimization (BayesOpt) experiments in each of the four regions: Central (train), West (test), East (test), and South (test).
Each Bayesian optimization ``episode'' is defined by randomly sub-sampling a day (uniformly at random between 1981 and 2020), then sampling a sub-region from the tested region.
To test the models' spatial generalization capacity (where possible), we sub-sample episodes from each of the four regions with the following sizes: 
\begin{inlinelist}
    \item Central: 42x42,
    \item West: 40x40,
    \item East: 28x28, and
    \item South: 36x36.
\end{inlinelist}

Episodes begin from empty sets $D_c^{(0)} = \empty$, and models sequentially query locations for $t=1, \hdots, 50$. 
Denoting $(\vx^{(t)}, y^{(t)})$ the query location and queried value at iteration $t$, the context set is then updated as $D_c^{(t)} = D_c^{(t-1)} \cup \{(\vx^{(t)}, y^{(t)})\}$. 
Denoting $\vy$ as the complete set of rainfall values in the sub-region, and $\vy^{(t)}$ as the set of queried values at iteration $t$, we can define the \textit{instantaneous regret} as 
$
    r_t = \text{max} (\vy) - \text{max} (\vy_c^{(t)}),
$
and compute the average regret (plotted in \cref{fig:bayesopt} in the main text) at the $t^{\text{th}}$ iteration as $\bar{r}_t = \frac{1}{t} \sum_{i=1}^{t} r_i$.

\section{Additional Figures for Environmental Data}
\label{app:environmental_figures}

\subsection{Predictive density}

\cref{fig:pred_density} displays the predictive densities for precipitation at different locations, conditioned on a context set used for testing.
The density of the ConvNP is estimated using 2500 samples of $\vz$. 
To examine why the ConvNP outperforms the GP in terms of log-likelihood, we plot cases where the ConvNP likelihood is significantly better than the GP likelihood.
We see that this is due to the GP occasionally making very overconfident predictions compared to the ConvNP.
We also see that the ConvNP in a small proportion of cases exhibits very non-Gaussian, asymmetric predictive distribtuions.

\input{plots/kde_plots}

\subsection{Additional Samples}

In this section we show additional samples from the model trained on $28 \times 28$ images (\cref{fig:precipitation_sample_28x28_2,fig:precipitation_sample_28x28}) and also on $40 \times 40$ images (\cref{fig:precipitation_sample_40x40_2,fig:precipitation_sample_40x40}). Training on larger images reduces the occurence of blocky artefacts. \Cref{fig:precipitation_example_task} in the main body was trained on $40 \times 40$ images. Note that samples shown here are $61 \times 201$, i.e.~the size of the entire central Europe train region.

\input{plots/prec_samples/prec_samples_app_28x28}

\input{plots/prec_samples/prec_samples_app_28x28_2}

\input{plots/prec_samples/prec_samples_app_40x40}

\input{plots/prec_samples/prec_samples_app_40x40_2}

\input{plots/prec_samples/prec_samples_app_40x40_3}

\end{document}

%% file: math_commands.tex
\usepackage{amsmath,amsfonts,bm}

\def\eqref#1{equation~\ref{#1}}
\def\1{\bm{1}}

\def\vmu{{\bm{\mu}}}
\def\vsigma{{\bm{\sigma}}}
\def\vtheta{{\bm{\theta}}}

\def\vphi{{\bm{\phi}}}

\def\vh{{\bm{h}}}

\def\vr{{\bm{r}}}

\def\vt{{\bm{t}}}

\def\vx{{\bm{x}}}
\def\vy{{\bm{y}}}
\def\vz{{\bm{z}}}

\def\mX{{\bm{X}}}

\DeclareMathAlphabet{\mathsfit}{\encodingdefault}{\sfdefault}{m}{sl}
\SetMathAlphabet{\mathsfit}{bold}{\encodingdefault}{\sfdefault}{bx}{n}

\def\gD{{\mathcal{D}}}

\def\gL{{\mathcal{L}}}

\def\gN{{\mathcal{N}}}

\def\gP{{\mathcal{P}}}

\def\gS{{\mathcal{S}}}

\def\gU{{\mathcal{U}}}

\def\gX{{\mathcal{X}}}
\def\gY{{\mathcal{Y}}}

\newcommand{\E}{\mathbb{E}}

\newcommand{\R}{\mathbb{R}}

\newcommand{\Cov}{\mathrm{Cov}}

%% file: preamble.tex
\usepackage[T1]{fontenc}    % Use 8-bit T1 fonts

\usepackage{microtype}      % Microtypography
\usepackage[utf8]{inputenc} % Allow utf-8 input
\usepackage{hyperref}       % Hyperlinks
\usepackage{url}            % Simple URL typesetting
\usepackage{booktabs}       % Professional-quality tables
\usepackage{multirow}       % Helper package for organizing tables    
\usepackage{amsfonts}       % Blackboard math symbols
\usepackage{nicefrac}       % Compact symbols for 1/2, etc.
\usepackage{tabularx}       % Tables
\usepackage{makecell}       % More for tables
\usepackage{ragged2e}       % Justification
\usepackage{paralist}       % Compact lists
\usepackage{arydshln}       % Dashed lines in tables
\usepackage{xargs}          % Use more than one optional parameter in a new commands
\usepackage{color, colortbl}
\usepackage[colorinlistoftodos,prependcaption,textsize=tiny]{todonotes}
\usepackage{xparse}

\usepackage{import}         % Importing subdocuments
\usepackage{standalone}     % Compilable subdocuments and inclusion thereof
\usepackage{enumitem}
  \newlist{inlinelist}{enumerate*}{1}
  \setlist*[inlinelist,1]{%
          label=(\roman*),
      }

\usepackage{amsmath}
\usepackage{cases}          % Cases in math
\usepackage{amsthm}         % Theorems
\usepackage{thmtools}       % Tools for theorems
\usepackage{mathtools}      % More math commands

\usepackage{pgf}
\usepackage{tikz}
\usepackage{pgfplots}
\pgfplotsset{compat=1.13}

\usepgfplotslibrary{groupplots}
\usetikzlibrary{calc}
\usetikzlibrary{positioning}
\usetikzlibrary{angles,quotes}
\usetikzlibrary{backgrounds}
\usetikzlibrary{external}
\usetikzlibrary{fit}
\usetikzlibrary{bayesnet}
\usetikzlibrary{calc}
\usetikzlibrary{fit}
\usetikzlibrary{spy, shadows}
\usetikzlibrary{arrows.meta}
\usetikzlibrary{shadings}
\usetikzlibrary{fadings}
\usetikzlibrary{matrix}     % SI units
\usepackage[
    capitalize,
    nameinlink]{cleveref}   % Automatic referencing

\begingroup
    \makeatletter
    \@for\theoremstyle:=theorem,corollary,lemma,model,definition,remark,plain\do{%
        \expandafter\g@addto@macro\csname th@\theoremstyle\endcsname{%
            \setlength\thm@preskip\parskip
            \setlength\thm@postskip\parskip%{0pt}
            \addtolength\thm@preskip\parskip
            }%
        }
    \makeatother
\endgroup

\usepackage{thmtools,thm-restate}

\theoremstyle{definition}

\newtheorem*{theorem*}{Thm}
\newtheorem{lemma}{Lem}
\newtheorem*{lemma*}{Lem}

\newtheorem*{remark*}{Rem}
\newtheorem{definition}{Def}
\newtheorem*{definition*}{Def}
\newtheorem{proposition}{Prop}
\newtheorem*{proposition*}{Prop}

\newtheorem*{property*}{Prop}

\crefname{thm}{Thm}{Thms}
\Crefname{thm}{Thm}{Thms}
\crefname{lem}{Lem}{Lems}
\Crefname{lem}{Lem}{Lems}
\crefname{proposition}{Prop}{Props}
\Crefname{proposition}{Prop}{Props}
\crefname{definition}{Def}{Defs}
\Crefname{definition}{Def}{Defs}
\crefname{property}{Prop}{Props}
\Crefname{property}{Prop}{Props}
\crefname{equation}{Eq}{Eqs}
\Crefname{equation}{Eq}{Eqs}
\crefname{section}{Sec}{Secs}
\Crefname{section}{Sec}{Secs}
\crefname{figure}{Fig}{Figs}
\Crefname{figure}{Fig}{Figs}
\crefname{table}{Tab}{Tabs}
\Crefname{table}{Tab}{Tabs}
\crefname{appendix}{App}{Apps}
\Crefname{appendix}{App}{Apps}

\newcolumntype{L}{>{\RaggedRight\arraybackslash}X}

\renewcommand{\max}{\operatorname{max}}
\definecolor{tabgreen}{HTML}{59a14f}

\newcommandx{\wpbnote}[2][1=]{\todo[linecolor=red,backgroundcolor=red!25,bordercolor=red,#1]{WPB: #2}}
\newcommandx{\jrrnote}[2][1=]{\todo[linecolor=blue,backgroundcolor=blue!25,bordercolor=blue,#1]{JRR: #2}}
\newcommandx{\aykfnote}[2][1=]{\todo[linecolor=purple,backgroundcolor=purple!25,bordercolor=purple,#1]{AYKF: #2}}
\newcommandx{\jgnote}[2][1=]{\todo[linecolor=green,backgroundcolor=green!25,bordercolor=green,#1]{JG: #2}}
\newcommandx{\retnote}[2][1=]{\todo[linecolor=yellow,backgroundcolor=yellow!25,bordercolor=yellow,#1]{RET: #2}}
\newcommandx{\ydnote}[2][1=]{\todo[linecolor=orange,backgroundcolor=orange!25,bordercolor=orange,#1]{YD: #2}}
\newcommandx{\thiswillnotshow}[2][1=]{\todo[disable,#1]{#2}}

%% file: diagrams/forwardpass_full.tex
% Source: https://tex.stackexchange.com/questions/7032/good-way-to-make-textcircled-numbers
\newcommand*\circled[1]{\tikz[baseline=(char.base)]{\node[shape=circle,color=white,fill=black,draw,inner sep=1pt] (char) {#1};}}

\begin{figure}[t]
    \centering
    \includegraphics[width=\linewidth]{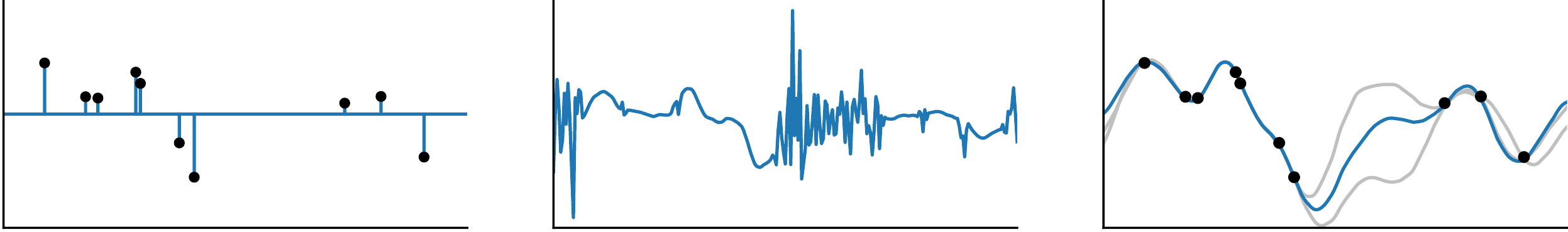}%
    \begin{tikzpicture}[overlay, remember picture]
        \node [anchor=south] () at (-11.8cm, 2cm) {\scriptsize \circled{\textbf{1}}\hspace{5 pt}Context set $D_c$};
        \node [anchor=south] () at (-7cm, 2cm) {\scriptsize \circled{\textbf{2}}\hspace{5 pt}Encoder: $z \sim \mathrm{ConvCNP}(D_c)$};
        \node [anchor=south] () at (-2cm, 2cm) {\scriptsize \circled{\textbf{3}}\hspace{5 pt}Decoder: $f = d(z)$};
        \draw [
            line width=1mm,
            ->,
            > = {
                Triangle[length=1mm, width=2.0mm]
            }
        ] (-9.6cm, 0.9cm) -- node [pos=0.5, anchor=south] {
            \footnotesize
            $\mathrm{E}_\vphi$
        } (-9.25cm, 0.9cm);
        \draw [
            line width=1mm,
            ->,
            > = {
                Triangle[length=1mm, width=2.0mm]
            }
        ] (-4.7cm, 0.9cm) -- node [pos=0.5, anchor=south] {
            \footnotesize
            $\mathrm{D}_\vtheta$
        } (-4.35cm, 0.9cm);
    \end{tikzpicture}%
    \caption{
        ConvNP encoder-decoder architecture. The encoder is a ConvCNP which takes the context set as input (left panel) and outputs a single sample of $z$ (center panel). The decoder takes this as input and outputs a predictive sample (right panel blue; two other samples shown in grey).
    }
    \label{fig:forward_pass}
    %\vspace*{-1.5em}
\end{figure}

%% file: plots/gp_plots_1d.tex
\begin{figure}
    \centering
    \makebox[0.49\linewidth][c]{\textsc{ConvNP}}
    \hfill \makebox[0.49\linewidth][c]{\textsc{ANP}} \\[2pt]
    \begin{tikzpicture}[overlay, remember picture]
        \node [rotate=90, anchor=south] () at (0, 0) {\tiny\textsc{Mat\'ern--$\frac52$}};
    \end{tikzpicture}%
    \includegraphics[width=0.49\linewidth]{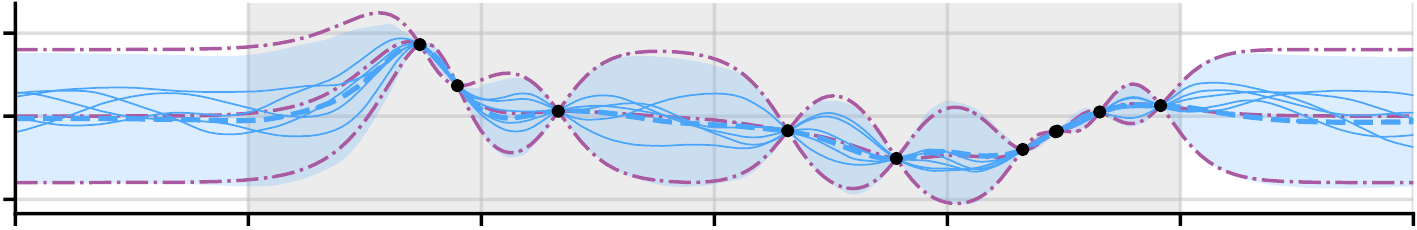}
    \hfill \includegraphics[width=0.49\linewidth]{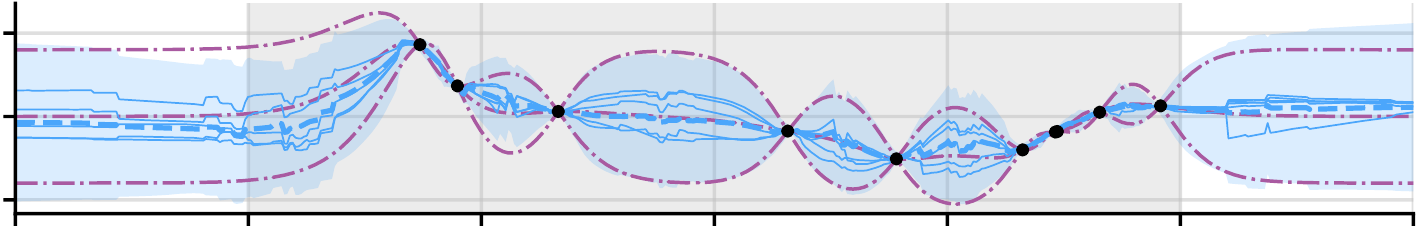}%
    \begin{tikzpicture}[overlay, remember picture]
        \node [rotate=90, anchor=north] () at (0, 15pt) {\tiny $\gL_{\mathrm{ML}}$};
    \end{tikzpicture}\\
    \includegraphics[width=0.49\linewidth]{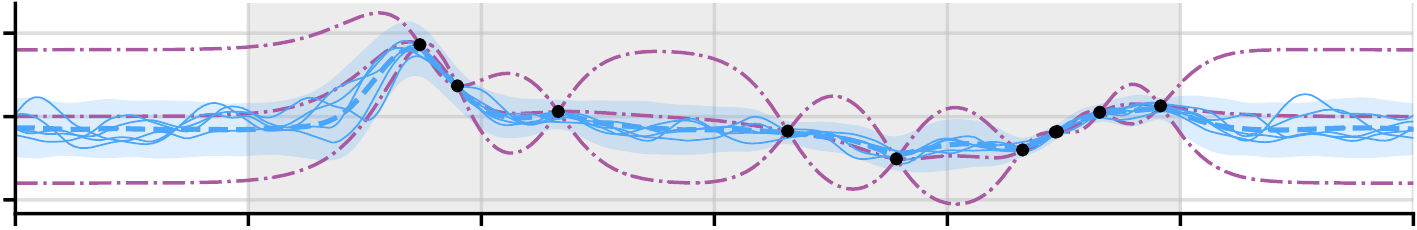}
    \hfill \includegraphics[width=0.49\linewidth]{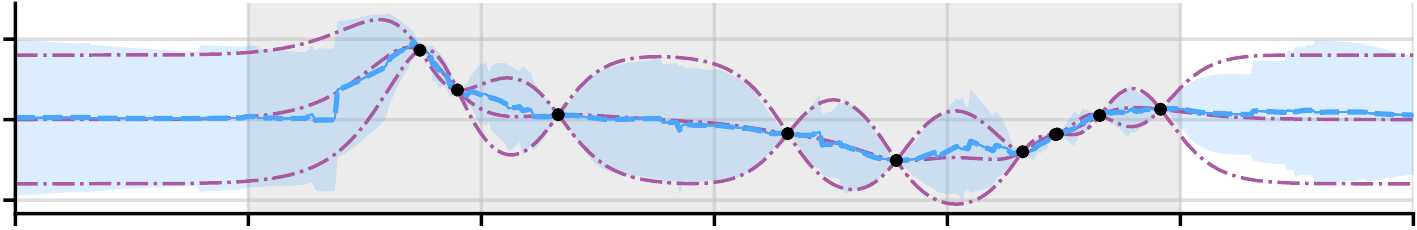}%
    \begin{tikzpicture}[overlay, remember picture]
        \node [rotate=90, anchor=north] () at (0, 15pt) {\tiny $\gL_{\mathrm{NP}}$};
    \end{tikzpicture}\\
    \begin{tikzpicture}[overlay, remember picture]
        \node [rotate=90, anchor=south] () at (0, 0) {\tiny\textsc{Mat\'ern--$\frac52$}};
    \end{tikzpicture}%
    \includegraphics[width=0.49\linewidth]{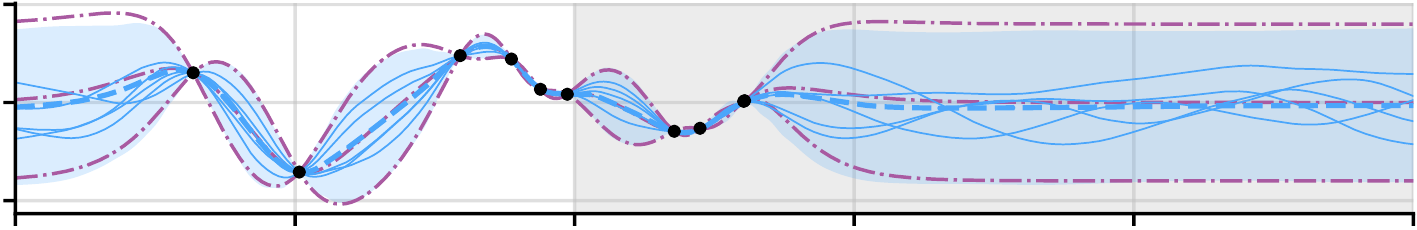}
    \hfill \includegraphics[width=0.49\linewidth]{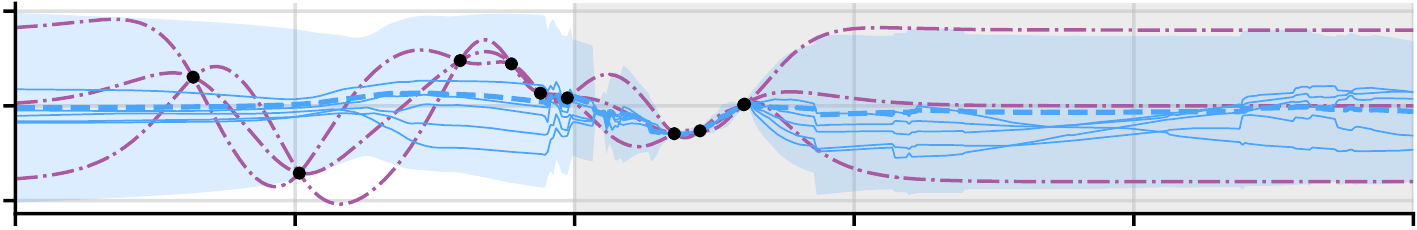}%
    \begin{tikzpicture}[overlay, remember picture]
        \node [rotate=90, anchor=north] () at (0, 15pt) {\tiny $\gL_{\mathrm{ML}}$};
    \end{tikzpicture}\\
    \includegraphics[width=0.49\linewidth]{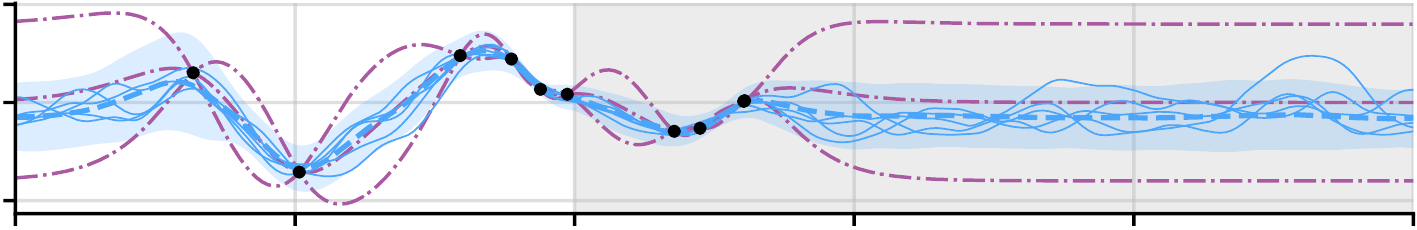}
    \hfill \includegraphics[width=0.49\linewidth]{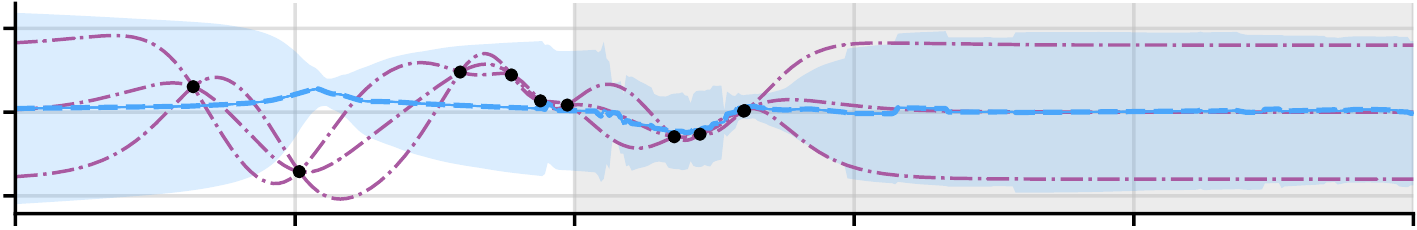}%
    \begin{tikzpicture}[overlay, remember picture]
        \node [rotate=90, anchor=north] () at (0, 15pt) {\tiny $\gL_{\mathrm{NP}}$};
    \end{tikzpicture}\\
    \begin{tikzpicture}[overlay, remember picture]
        \node [rotate=90, anchor=south] () at (0, 0) {\tiny\textsc{Weakly Periodic}};
    \end{tikzpicture}%
    \includegraphics[width=0.49\linewidth]{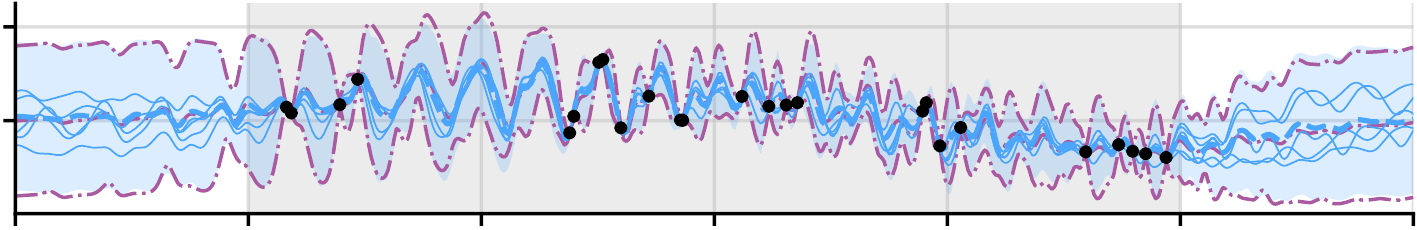}
    \hfill \includegraphics[width=0.49\linewidth]{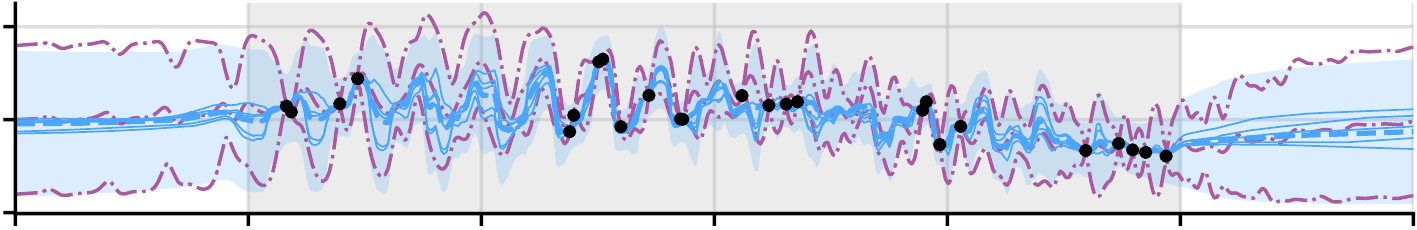}%
    \begin{tikzpicture}[overlay, remember picture]
        \node [rotate=90, anchor=north] () at (0, 15pt) {\tiny $\gL_{\mathrm{ML}}$};
    \end{tikzpicture}\\
    \includegraphics[width=0.49\linewidth]{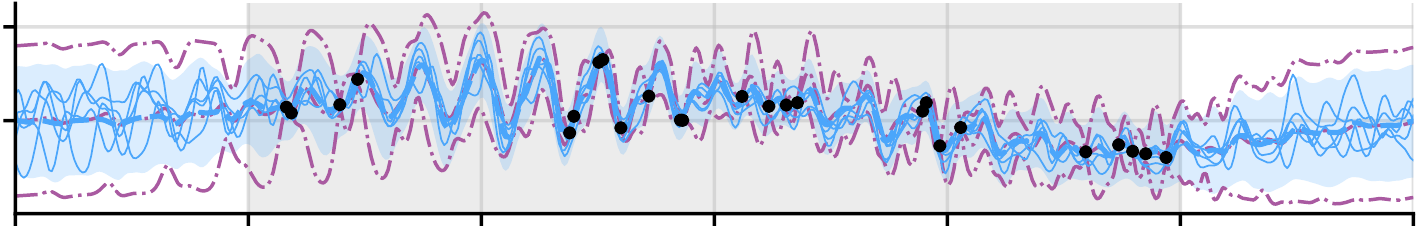}
    \hfill \includegraphics[width=0.49\linewidth]{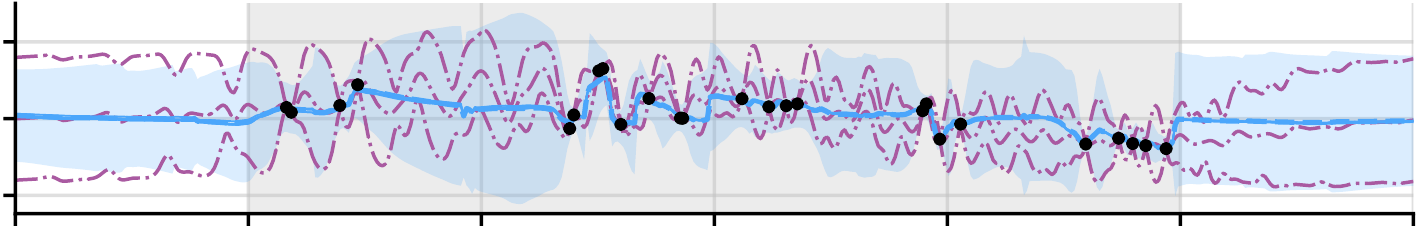}%
    \begin{tikzpicture}[overlay, remember picture]
        \node [rotate=90, anchor=north] () at (0, 15pt) {\tiny $\gL_{\mathrm{NP}}$};
    \end{tikzpicture}\\
    \begin{tikzpicture}[overlay, remember picture]
        \node [rotate=90, anchor=south] () at (0, 0) {\tiny\textsc{Sawtooth}};
    \end{tikzpicture}%
    \includegraphics[width=0.49\linewidth]{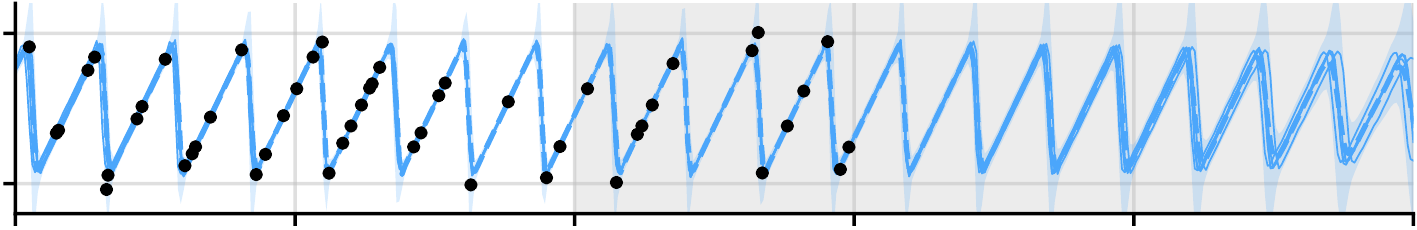}
    \hfill \includegraphics[width=0.49\linewidth]{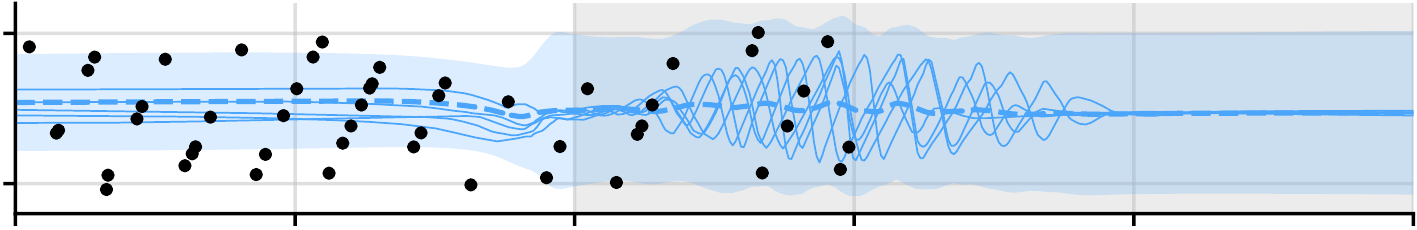}%
    \begin{tikzpicture}[overlay, remember picture]
        \node [rotate=90, anchor=north] () at (0, 15pt) {\tiny $\gL_{\mathrm{ML}}$};
    \end{tikzpicture}\\
    \includegraphics[width=0.49\linewidth]{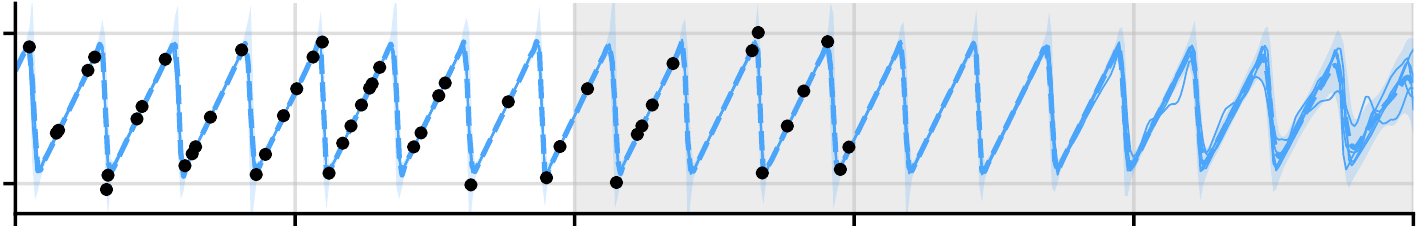}
    \hfill \includegraphics[width=0.49\linewidth]{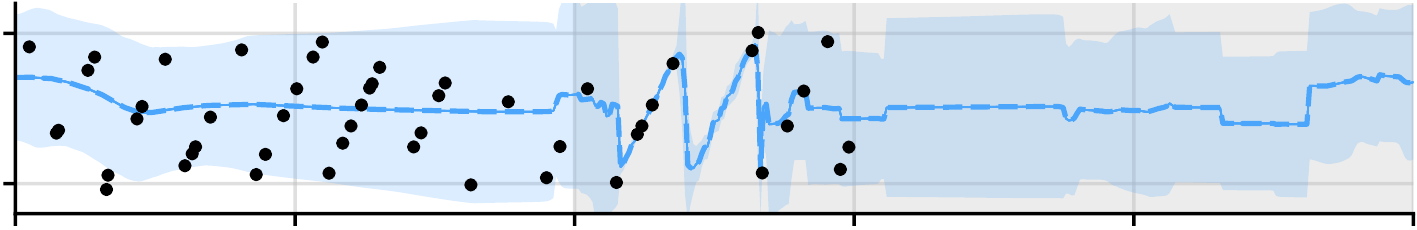}%
    \begin{tikzpicture}[overlay, remember picture]
        \node [rotate=90, anchor=north] () at (0, 15pt) {\tiny $\gL_{\mathrm{NP}}$};
    \end{tikzpicture}
    \caption{
        Predictions of ConvNPs and ANPs trained with $\gL_{\mathrm{ML}}$ and $\gL_{\mathrm{NP}}$, showing interpolation and extrapolation within (grey background) and outside (white background) the training range. Solid blue lines are samples, dashed blue lines are means, and the shaded blue area is $\mu \pm 2\sigma$. Purple dash--dot lines are the ground-truth GP mean and $\mu\pm2\sigma$. ConvNP handles points outside the training range naturally, whereas this leads to catastrophic failure for the ANP. Note ANP with $\gL_{\mathrm{NP}}$ tends to collapse to deterministic samples, with all uncertainty explained with the heteroskedastic noise. In contrast, models trained with $\gL_{\mathrm{ML}}$ show diverse samples that account for much of the uncertainty. 
    }
    \label{fig:1D_predictions}
    \vspace*{-1.5em}
\end{figure}

%% file: tables/gp1d_table.tex
\begin{table}[t]
\footnotesize
\caption{Log-likelihoods on 1D regression tasks. Lower bounds marked with asterisk. Highest non-GP values in bold.}
\label{table:1d_bakeoff_main_text}
\begin{center}
\begin{tabular}{@{}l@{\hspace{4pt}}
                c@{\hspace{4pt}}
                r@{\hspace{4pt}}
                r@{\hspace{4pt}}
                r@{\hspace{4pt}}
                r@{\hspace{4pt}}
                r@{\hspace{4pt}}
                r@{}}

\toprule 
 & &
 \multicolumn{3}{c}{\textsc{within training range}} & 
 \multicolumn{3}{c}{\textsc{beyond training range}}\\
 \midrule
 &  &
 \multicolumn{1}{c}{Mat\'ern-$\frac52$}  & 
 \multicolumn{1}{c}{Weakly Per.}         & 
 \multicolumn{1}{c}{Sawtooth}            &
 \multicolumn{1}{c}{Mat\'ern-$\frac52$}  &
 \multicolumn{1}{c}{Weakly Per.}         &
 \multicolumn{1}{c}{Sawtooth}\\[0.2em]
 GP  & (full) &
 $1.22 { \scriptstyle \,\pm\, 6\text{\textsc{e}}{\,\text{--}3} }$             &         %   Matern Interpolation
 $\text{--}{0.06} { \scriptstyle \,\pm\, 5\text{\textsc{e}}{\,\text{--}3} }$  &         %   Weakly Per. Interpolation 
 \multicolumn{1}{c}{N/A}                                                      &         %   Sawtooth Interpolation 
 $1.22 { \scriptstyle \,\pm\, 6\text{\textsc{e}}{\,\text{--}3} }$             &         %   Matern Extrapolation
 $\text{--}{0.06} { \scriptstyle \,\pm\, 5\text{\textsc{e}}{\,\text{--}3} }$  &         %   Weakly Per. Extrapolation 
 \multicolumn{1}{c}{N/A}      \\                                                        %   Sawtooth Extrapolation 
 ConvNP$^\ast$ & ($\gL_{\mathrm{ML}}$) & 
 $\text{--}\textbf{0.58} { \scriptstyle \,\pm\, 0.01 }$ &                                     %   Matern Interpolation
 $\text{--}\textbf{1.02} { \scriptstyle \,\pm\, 6\text{\textsc{e}}{\,\text{--}3} }$ &         %   Weakly Per. Interpolation 
 $\textbf{2.30} { \scriptstyle \,\pm\, 0.01 }$                                      &         %   Sawtooth Interpolation 
 $\text{--}\textbf{0.58} { \scriptstyle \,\pm\, 0.01 }$ &                                     %   Matern Extrapolation
 $\text{--}\textbf{1.03} { \scriptstyle \,\pm\, 6\text{\textsc{e}}{\,\text{--}3} }$ &         %   Weakly Per. Extrapolation 
 $\textbf{2.29} { \scriptstyle \,\pm\, 0.02 }$ \\                                             %   Sawtooth Extrapolation 
 ANP$^\ast$ & ($\gL_{\mathrm{ML}}$) & 
 $\text{--}0.73 { \scriptstyle \,\pm\, 0.01 }$ &                                    %   Matern Interpolation
 $\text{--}1.14 { \scriptstyle \,\pm\, 6\text{\textsc{e}}{\,\text{--}3} }$ &        %   Weakly Per. Interpolation
 $0.09 { \scriptstyle \,\pm\, 3\text{\textsc{e}}{\,\text{--}3} }$ &                 %   Sawtooth Interpolation
 $\text{--}1.39 { \scriptstyle \,\pm\, 7\text{\textsc{e}}{\,\text{--}3} }$ &        %   Matern Extrapolation
 $\text{--}1.35 { \scriptstyle \,\pm\, 4\text{\textsc{e}}{\,\text{--}3} }$                             &        %   Weakly Per. Extrapolation
 $\text{--}0.17 { \scriptstyle \,\pm\, 1\text{\textsc{e}}{\,\text{--}3} }$ \\       %   Sawtooth Extrapolation
 ANP$^\ast$ & ($\gL_{\mathrm{NP}}$) & 
 $\text{--}0.96 { \scriptstyle \,\pm\, 0.01 }$ &                                    %   Matern Interpolation
 $\text{--}1.37 { \scriptstyle \,\pm\, 6\text{\textsc{e}}{\,\text{--}3} }$ &        %   Weakly Per. Interpolation
 $0.20 { \scriptstyle \,\pm\, 9\text{\textsc{e}}{\,\text{--}3} }$ &                 %   Sawtooth Interpolation
 $\text{--}1.48 { \scriptstyle \,\pm\, 4\text{\textsc{e}}{\,\text{--}3} }$ &        %   Matern Extrapolation
 $\text{--}1.66 { \scriptstyle \,\pm\, 0.01 }$                             &        %   Weakly Per. Extrapolation
 $\text{--}0.30 { \scriptstyle \,\pm\, 4\text{\textsc{e}}{\,\text{--}3} }$ \\       %   Sawtooth Extrapolation
 GP  & (diag) &
 $\text{--}{0.84} { \scriptstyle \,\pm\, 9\text{\textsc{e}}{\,\text{--}3} }$        &         %   Matern Interpolation
 $\text{--}{1.17} { \scriptstyle \,\pm\, 5\text{\textsc{e}}{\,\text{--}3} }$        &         %   Weakly Per. Interpolation 
 \multicolumn{1}{c}{N/A}                                                            &         %   Sawtooth Interpolation 
 $\text{--}{0.84} { \scriptstyle \,\pm\, 9\text{\textsc{e}}{\,\text{--}3} }$        &         %   Matern Extrapolation
 $\text{--}{1.17} { \scriptstyle \,\pm\, 5\text{\textsc{e}}{\,\text{--}3} }$        &         %   Weakly Per. Extrapolation 
 \multicolumn{1}{c}{N/A}   \\                                                                 %   Sawtooth Extrapolation 
 ConvCNP & &
 $\text{--}{0.88} { \scriptstyle \,\pm\, 0.01 }$                                    &         %   Matern Interpolation
 $\text{--}{1.19} { \scriptstyle \,\pm\, 7\text{\textsc{e}}{\,\text{--}3} }$        &         %   Weakly Per. Interpolation 
 $1.15 { \scriptstyle \,\pm\, 0.04 }$                                               &         %   Sawtooth Interpolation 
 $\text{--}{0.87} { \scriptstyle \,\pm\, 0.01 }$                                    &         %   Matern Extrapolation
 $\text{--}{1.19} { \scriptstyle \,\pm\, 7\text{\textsc{e}}{\,\text{--}3} }$        &         %   Weakly Per. Extrapolation 
 $1.11 { \scriptstyle \,\pm\, 0.04 }$ \\                                                      %   Sawtooth Extrapolation 
 \bottomrule

\end{tabular}
\end{center}
\vspace*{-2em}
\end{table}

%% file: tables/image_main_text_squeezed_homo.tex
\begin{table}[t]
\footnotesize
\caption{Test log-likelihood lower bounds for image completion (5 runs).}
\label{table:images_main_text}
\vspace*{-2pt}
\begin{center}
\begin{adjustbox}{max width=\textwidth}
\begin{tabular}{@{}lcccccccc@{}}
\toprule
\multirow{2}{*}{} & 
\multicolumn{2}{c}{MNIST}                 & 
\multicolumn{2}{c}{CelebA32}              & 
\multicolumn{2}{c}{SVHN}                  & 
\multicolumn{2}{c}{ZSMM}                  \\ 
& 
$\gL_{\mathrm{ML}}$ &  $\gL_{\mathrm{NP}}$ & 
$\gL_{\mathrm{ML}}$ &  $\gL_{\mathrm{NP}}$ & 
$\gL_{\mathrm{ML}}$ &  $\gL_{\mathrm{NP}}$ & 
$\gL_{\mathrm{ML}}$ & $\gL_{\mathrm{NP}}$ \\ 
\midrule
ConvNP    &  
$\textbf{2.11} { \scriptstyle \,\pm\, 0.01 }$  &   % MNIST ML
$0.99 { \scriptstyle \,\pm\, 0.42 }$           &   % MNIST NP                 
$\textbf{6.92} { \scriptstyle \,\pm\, 0.10 }$  &   % CELEBA ML              
$-0.27 { \scriptstyle \,\pm\, 0.00 }$           &   % CELEBA NP
$\textbf{9.89} { \scriptstyle \,\pm\, 0.09}$   &   % SVHN ML                   
$0.17 { \scriptstyle \,\pm\, 0.00 }$           &   % SVHN NP                   
$\textbf{4.58} { \scriptstyle \,\pm\, 0.04 }$  &   % ZSMM ML
$0.14  { \scriptstyle \,\pm\, 0.00}$           \\  % ZSMM NP
ANP   &
$1.66 { \scriptstyle \,\pm\, 0.03 }$           &  % MNIST ML
$1.64 { \scriptstyle \,\pm\, 0.03 }$           &  % MNIST NP           
$5.98 { \scriptstyle \,\pm\, 0.08 }$           &  % CELEBA ML         
$6.04 { \scriptstyle \,\pm\, 0.10 }$           &  % CELEBA NP         
$9.18 { \scriptstyle \,\pm\, 0.08 }$           &  % SVHN ML         
$8.91 { \scriptstyle \,\pm\, 0.06 }$           &  % SVHN NP         
$-10.8 { \scriptstyle \,\pm\, 1.99 }$           &  % ZSMM ML         
$-6.45 { \scriptstyle \,\pm\, 0.99 }$           \\ % ZSMM NP
\bottomrule
\end{tabular}
\end{adjustbox}
\end{center}
\vspace*{-1.0em}
\end{table}

% Note : MNIST and ZSMM is homoskedastic  the rest is hetero

%% file: plots/image_completion/image_figures_homo.tex
\begin{figure}[t]
    \centering
\begin{subfigure}[b]{0.43\textwidth}
    \begin{subfigure}[b]{0.48\columnwidth}
        \includegraphics[width=\textwidth]{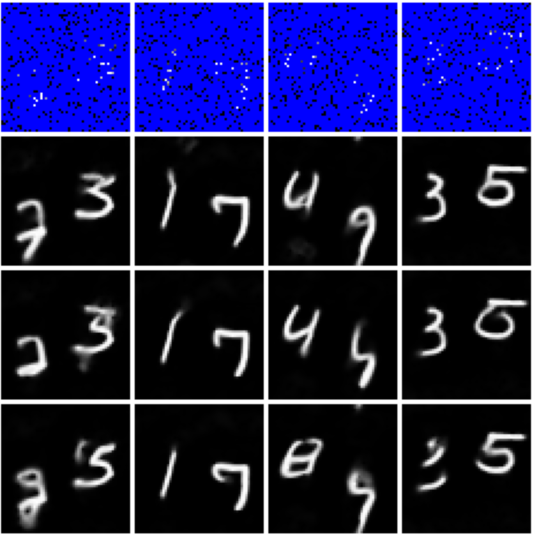}
        \subcaption{ConvNP}
        \label{fig:images_zsmm_convnp}
    \end{subfigure}
    \begin{subfigure}[b]{0.48\columnwidth}
        \includegraphics[width=\textwidth]{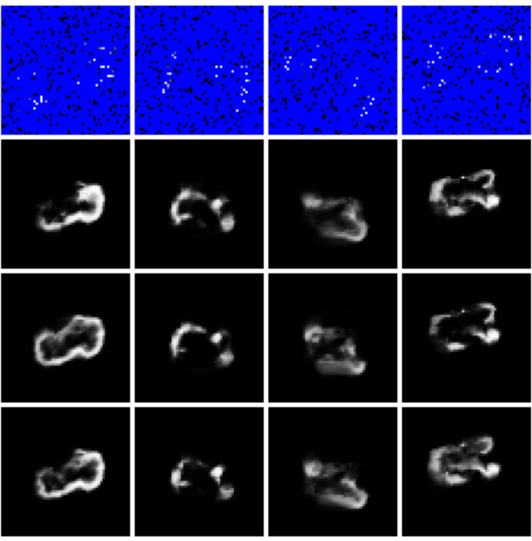}
        \subcaption{ANP}
        \label{fig:images_zsmm_anp}
    \end{subfigure}
\end{subfigure}
\begin{subfigure}[b]{0.56\textwidth}
    \begin{subfigure}[b]{0.48\columnwidth}
        \includegraphics[width=\textwidth]{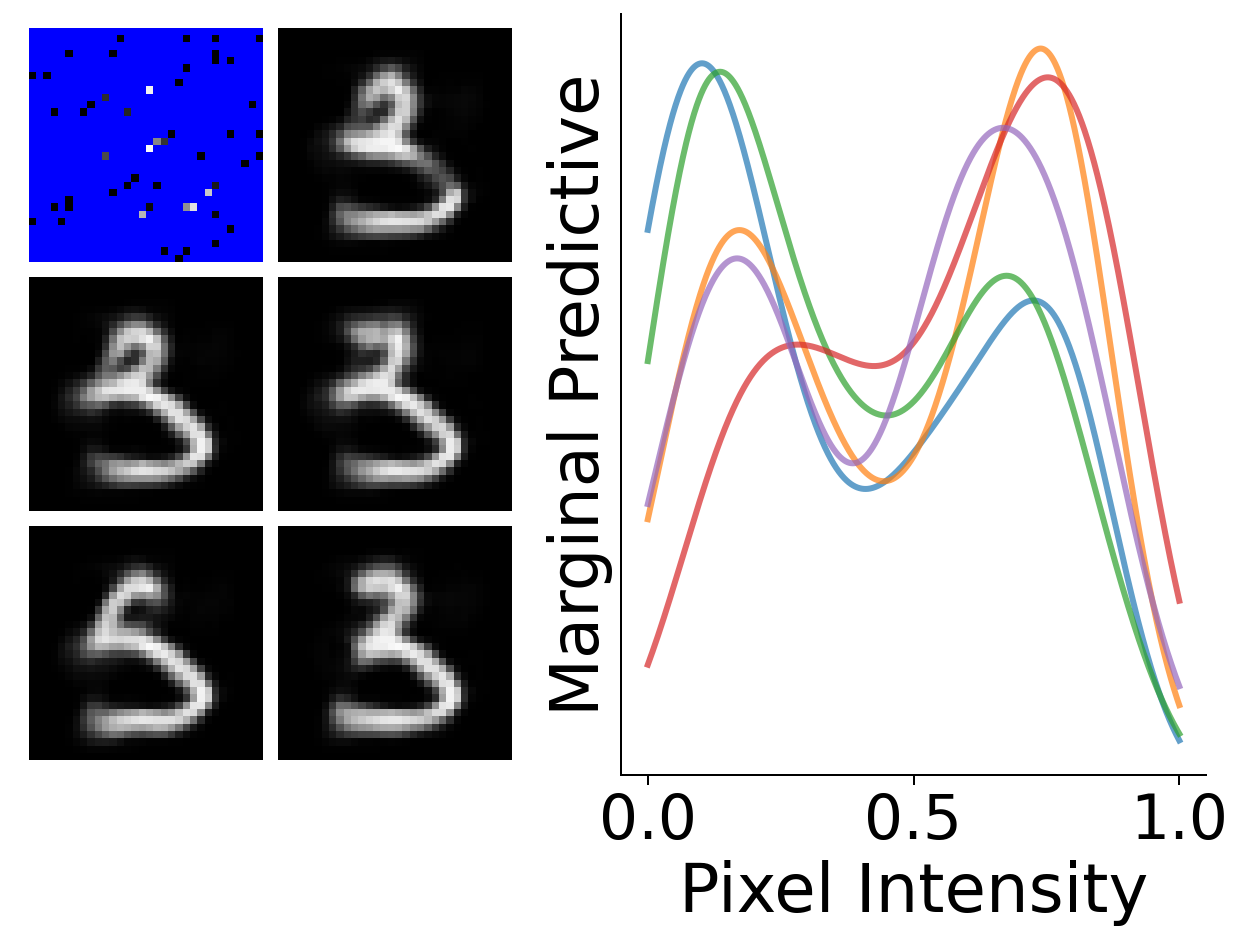}
        \subcaption{ConvNP}
        \label{fig:images_marginal_convnp}
    \end{subfigure}
    \begin{subfigure}[b]{0.48\columnwidth}
        \includegraphics[width=\textwidth]{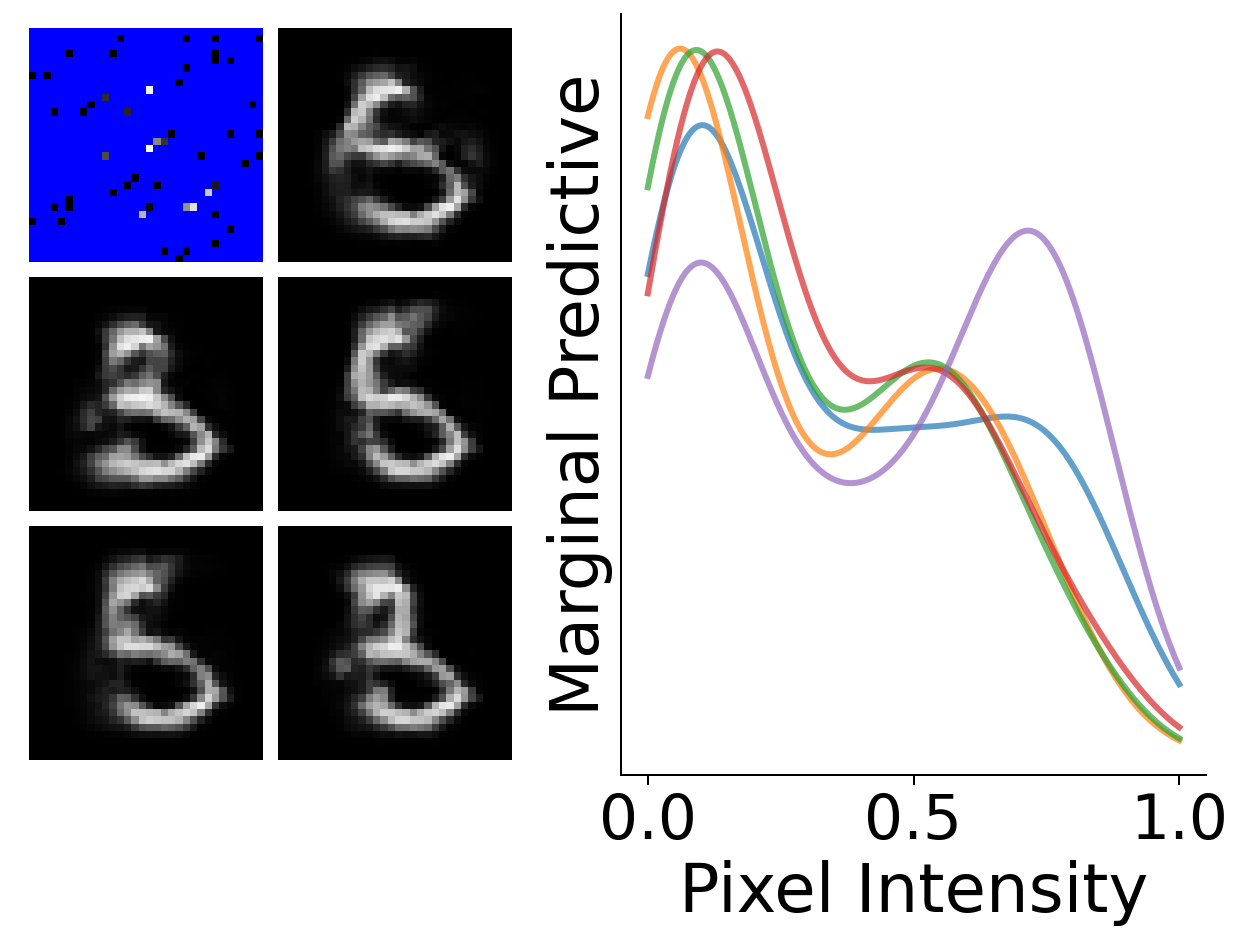}
        \subcaption{ANP}
        \label{fig:images_marginal_anp}
    \end{subfigure}
\end{subfigure}
\caption{Left two plots: predictive samples on zero-shot multi MNIST. Right two plots: samples and marginal predictives on standard MNIST. We plot the density of the five marginals that maximize Sarle's bimodality coefficient \cite{ellison1987effect}. We use $\gL_{\mathrm{ML}}$ for training.
Blue pixels are not in the context set.}
\label{fig:images_zsmm}
%\vspace*{-1.5em}
\end{figure}

%% file: tables/environmental_metrics_v2.tex
\begin{table}[t] 
\footnotesize
\caption{
    Joint predictive log-likelihoods (LL) and RMSEs on ERA5-Land, averaged over 1000 tasks.
    %Results within standard error of best result in bold.
}
\label{table:era5_loglik_rmse}
\begin{center}
\begin{tabular}{@{}llcccc@{}}
\toprule
&        & 
    Central (train) & 
    West (test)     & 
    East (test)     & 
    South (test)   \\ 
    \midrule
    \multirow{2}{*}{LL}   & 
    ConvNP                & 
    $\textbf{4.47} { \scriptstyle \,\pm\, 0.07 }$  & 
    $\textbf{4.55} { \scriptstyle \,\pm\, 0.08 }$  & 
    $\textbf{5.07} { \scriptstyle \,\pm\, 0.07 }$  & 
    $\textbf{4.65} { \scriptstyle \,\pm\, 0.08 }$   \\
    & GP     & 
    $3.33 { \scriptstyle \,\pm\, 0.06 }$           & 
    $3.65 { \scriptstyle \,\pm\, 0.06 }$           & 
    $4.07 { \scriptstyle \,\pm\, 0.06 }$           & 
    $3.34 { \scriptstyle \,\pm\, 0.06}$ \\
    \multirow{2}{*}{RMSE ($\times 10^{-2}$)}                           & 
    ConvNP                                          & 
    $\textbf{5.72} { \scriptstyle \,\pm\, 0.33 }$ & 
    $\textbf{5.77} { \scriptstyle \,\pm\, 0.37 }$ & 
    $\textbf{3.23} { \scriptstyle \,\pm\, 0.22 }$ & 
    $\textbf{6.92} { \scriptstyle \,\pm\, 0.39 }$          \\
    & GP     & 
    $\textbf{6.26} { \scriptstyle \,\pm\, 0.30 }$          & 
    $\textbf{5.75} { \scriptstyle \,\pm\, 0.29 }$ & 
    $\textbf{3.10} { \scriptstyle \,\pm\, 0.18 }$ & 
    $7.94 { \scriptstyle \,\pm\, 0.44 }$ \\ 
    \bottomrule
\end{tabular}
\end{center}
\end{table}

%% file: plots/percipitation_samples.tex
\begin{figure}[th!]
%%
     %%%  First row: Ground truth and NP samples  %%%
%%
\vspace*{-0.8em}
\begin{subfigure}{0.24\textwidth}
  \includegraphics[width=\linewidth]{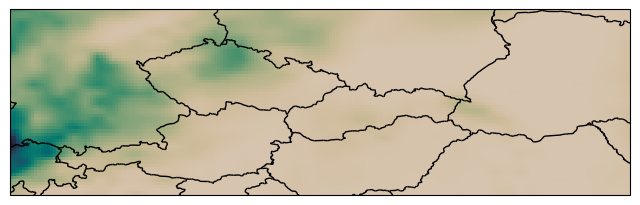}
  \label{fig:ground_truth}
  \vspace{-1.5em}
  \caption{Ground truth data}
\end{subfigure}\hfil
\begin{subfigure}{0.24\textwidth}
  \includegraphics[width=\linewidth]{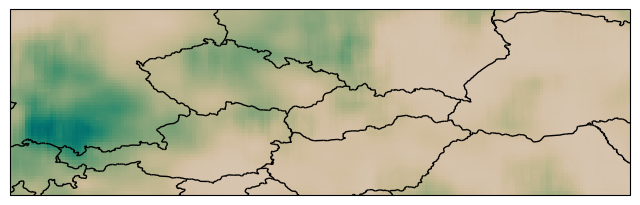}
  \label{fig:np_sample1}
  \vspace{-1.5em}
  \caption{ConvNP sample 1}
\end{subfigure}\hfil 
\begin{subfigure}{0.24\textwidth}
  \includegraphics[width=\linewidth]{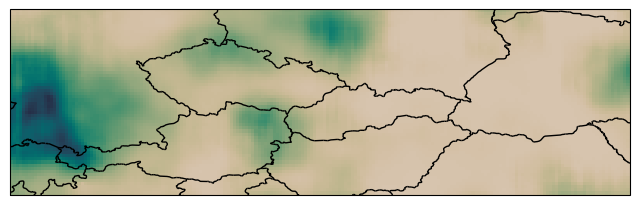}
  \label{fig:np_sample2}
  \vspace{-1.5em}
  \caption{ConvNP sample 2}
\end{subfigure}\hfil 
\begin{subfigure}{0.24\textwidth}
  \includegraphics[width=\linewidth]{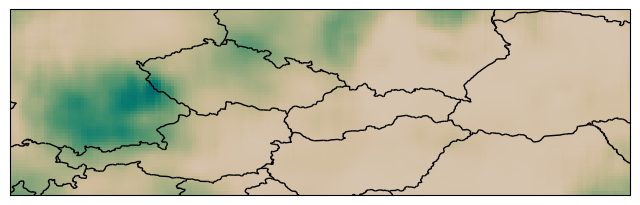}
  \label{fig:np_sample3}
  \vspace{-1.5em}
  \caption{ConvNP sample 3}
\end{subfigure}\\
%%
     %%%  Second row: Context set and GP samples  %%%
%%
\begin{subfigure}{0.24\textwidth}
  \includegraphics[width=\linewidth]{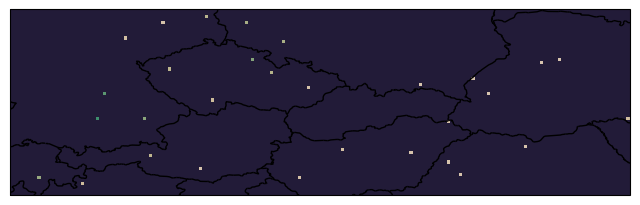}
  \label{fig:context_Set}
  \vspace{-1.5em}
  \caption{Context set}
\end{subfigure}\hfil 
\begin{subfigure}{0.24\textwidth}
  \includegraphics[width=\linewidth]{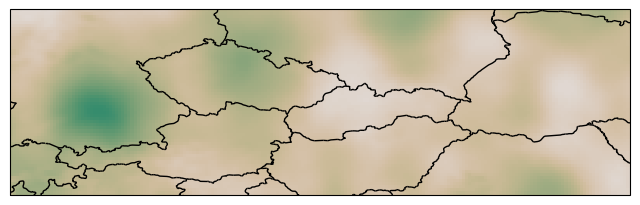}
  \label{fig:gp_sample1}
  \vspace{-1.5em}
  \caption{GP sample 1}
\end{subfigure}\hfil
\begin{subfigure}{0.24\textwidth}
  \includegraphics[width=\linewidth]{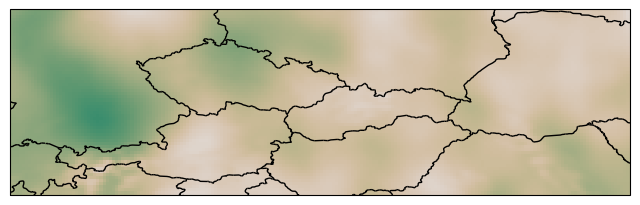}
  \label{fig:gp_sample2}
  \vspace{-1.5em}
  \caption{GP sample 2}
\end{subfigure}\hfil
\begin{subfigure}{0.24\textwidth}
  \includegraphics[width=\linewidth]{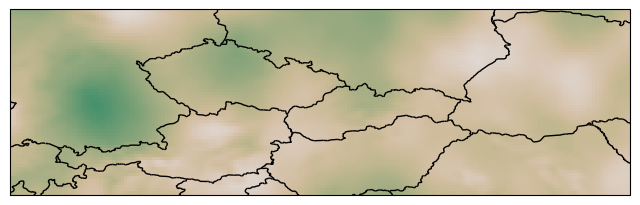}
  \label{fig:gp_sample3}
  \vspace{-1.5em}
  \caption{GP sample 3}
\end{subfigure}
\caption{Predictive samples overlaid on central Europe. Darker colours show higher precipitation. In (e), coloured pixels represent context points. GP samples often take negative values (lighter than ground truth data, see \cref{app:environmental_gp_baselines} for a discussion), whereas the NP has learned to produce non-negative samples which capture the \emph{sparsity} of precipitation. The model is trained on subregions roughly the size of the lengthscale of the precipitation process. More samples in \cref{app:environmental_figures}.}
\label{fig:precipitation_example_task}
%\vspace*{-1.5em}
\end{figure}

%% file: plots/bayes_opt_plot.tex
\begin{figure}[h!]
    \centering
    \includegraphics[width=1.\textwidth]{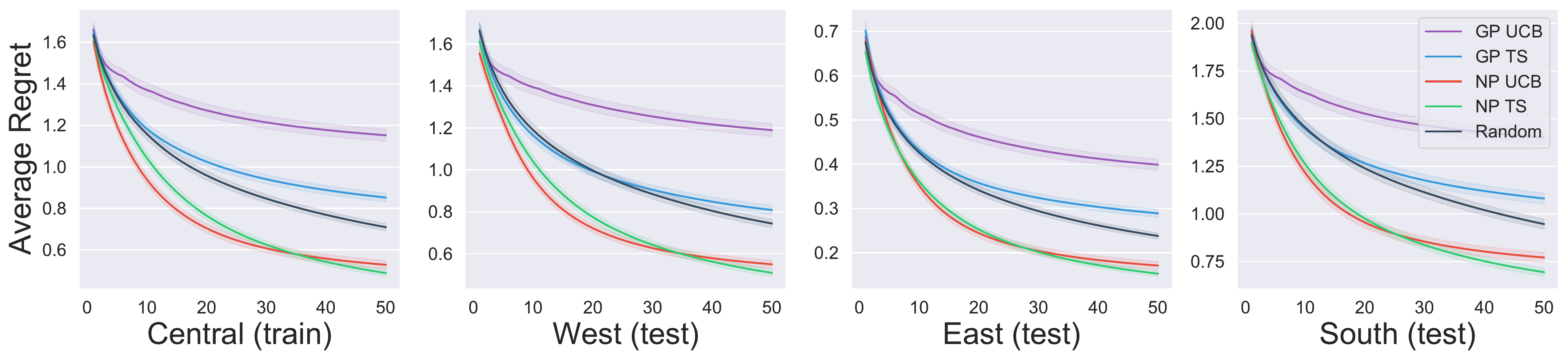}%
    \vspace*{-4pt}%
    \caption{Average regret plotted against number of points queried, averaged over 5000 tasks.}
    \label{fig:bayesopt}
    %\vspace*{-1em}
\end{figure}

%% file: plots/perf_v_samples/perf_v_samples.tex
\begin{figure}[t]
\begin{subfigure}{\textwidth}
    \centering
    \includegraphics[width=.9\textwidth]{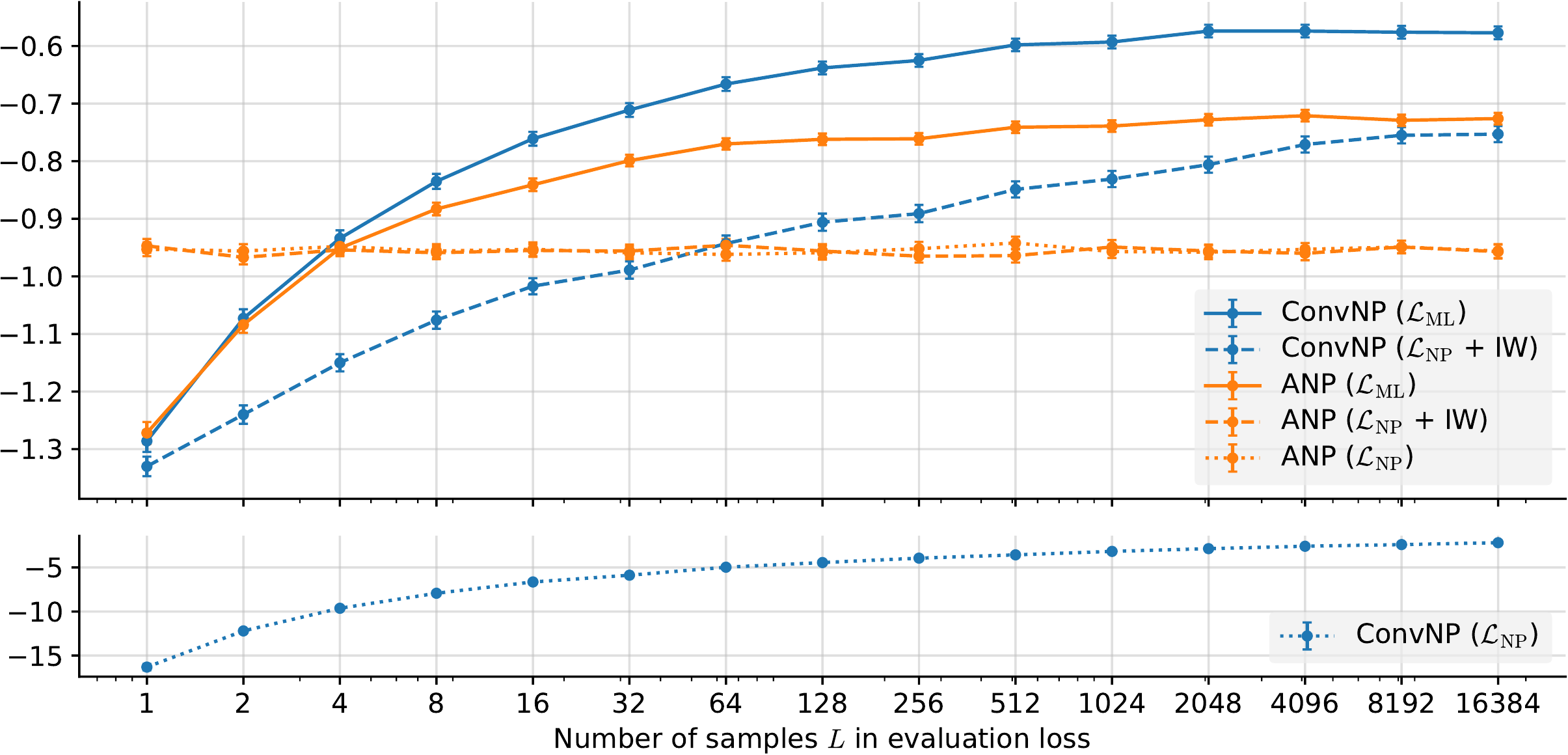}
    \caption{Mat\'ern--$\frac52$}
    \label{fig:perf_v_samples_matern}
\end{subfigure}
\begin{subfigure}{\textwidth}
    \centering
    \includegraphics[width=.9\textwidth]{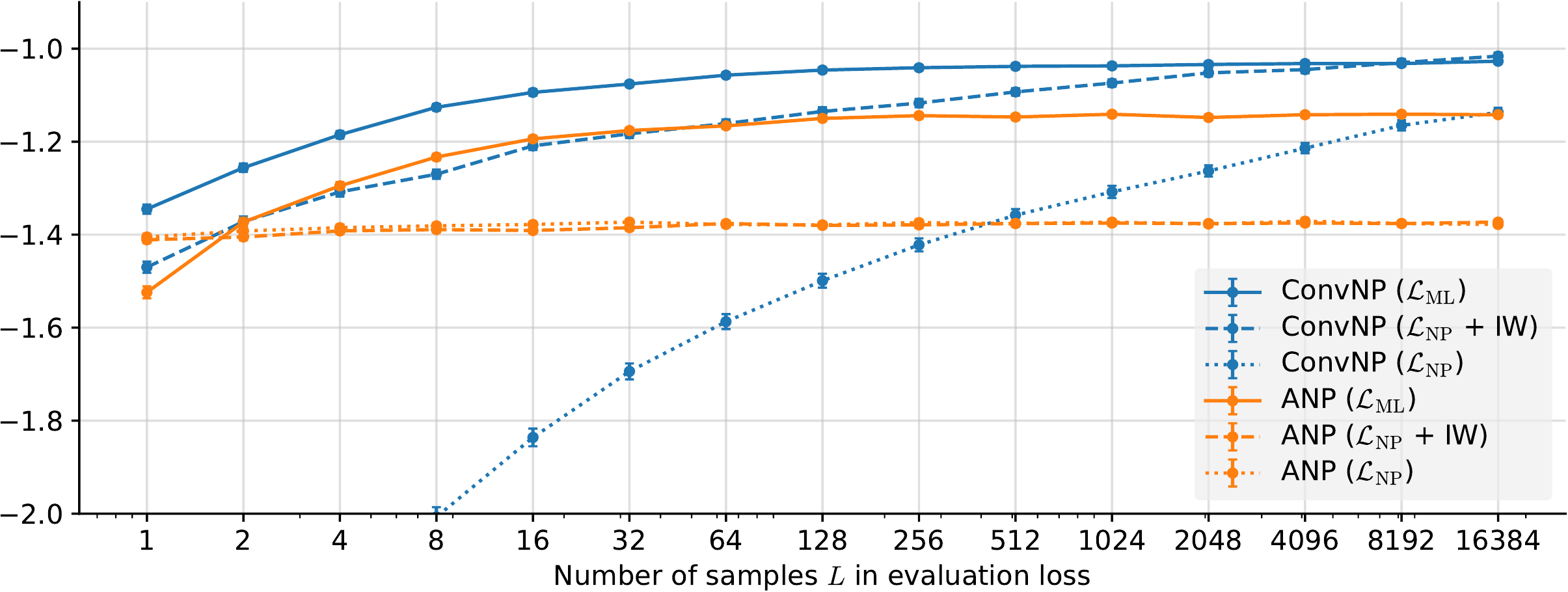}
    \caption{Weakly periodic kernel}
    \label{fig:perf_v_samples_weakly-periodic}
\end{subfigure}
\caption{
    Log-likelihood bounds achieved by various combination of models and training objectives when evaluated with $\gL_{\textrm{ML}}$ and $\gL_{\textrm{IW}}$ for various numbers of samples $L$.
    Color indicates model.
    Solid lines correspond to models trained and evaluated with $\gL_{\textrm{ML}}$.
    Dashed lines correspond to models trained with $\gL_{\textrm{NP}}$ and evaluated with $\gL_{\textrm{IW}}$.
    Dotted lines correspond to models trained with $\gL_{\textrm{ML}}$ and evaluated with $\gL_{\textrm{ML}}$.
    % Color-coding indicates model / loss pairs, solid lines denote evaluation with $\gL_{\textrm{ML}}$, and dashed lines denote evaluation with $\gL_{\textrm{IW}}$. Performance of ConvNP when trained with $\gL_{\textrm{NP}}$ and evaluated with $\gL_{\textrm{ML}}$ is plotted separately for each kernel to allow for the different scales.
}
\label{fig:perf_v_samples}
\end{figure}

%% file: tables/parameter_counts_1d.tex
\begin{table}[t]
    \centering
    \begin{tabular}{rccccc}
        \toprule
        & EQ & Mat\'ern--$\frac52$ & Noisy Mixt.\ & Weakly Per.\ & Sawtooth \\ \midrule
        ConvCNP & $42\,822$ & $42\,822$ & $51\,014$ & $51\,014$ & $100\,166$ \\
        ConvNP & $88\,486$ & $88\,486$ & $104\,870$ & $104\,870$ & $203\,174$ \\
        ANP & $530\,178$ & $530\,178$ & $530\,178$ & $530\,178$ & $530\,178$ \\
        NP & $479\,874$ & $479\,874$ & $479\,874$ & $479\,874$ & $479\,874$ \\
        \bottomrule
    \end{tabular}
    \vspace*{5pt}
    \caption{Parameter counts for the ConvCNP, ConvNP, ANP, and NP in the 1D regression tasks}
    \label{tab:1D_parameter_counts}
\end{table}

%% file: tables/bakeoff1d_amortised.tex
\begin{table}
\footnotesize
\caption{
    Log-likelihood for ConvCNP, ConvNP, ANP, and NP.
    Each of the stochastic models was trained on each data set with $\gL_{\mathrm{ML}}$ and $\gL_{\mathrm{NP}}$, separately.
    % The noise model for these experiments was \textit{amortized}.
}
\label{table:1d_bakeoff_amortized}
\begin{center}
\begin{tabular}{l@{\hspace{4pt}}l@{\hspace{4pt}}c@{\hspace{4pt}}c@{\hspace{4pt}}c@{\hspace{4pt}}c@{\hspace{4pt}}c}

\toprule 
 & & \multicolumn{1}{c}{EQ} & \multicolumn{1}{c}{Mat\'ern--$\frac52$} & \multicolumn{1}{c}{Noisy Mixt.} & \multicolumn{1}{c}{Weakly Per.} & \multicolumn{1}{c}{Sawtooth}\\
\midrule\multicolumn{7}{l}{\textsc{Interpolation inside training range}} \\[0.5em] 
GP (full) & &
    $5.80 { \scriptstyle \,\pm\, 0.02 }$ &
    $1.22 { \scriptstyle \,\pm\, 6.3\text{\textsc{e}}{\,\text{--}3} }$ &
    $1.00 { \scriptstyle \,\pm\, 4.1\text{\textsc{e}}{\,\text{--}3} }$ &
    $\text{--}0.06 { \scriptstyle \,\pm\, 4.6\text{\textsc{e}}{\,\text{--}3} }$ &
    N/A \\
GP (diag) & &
    $\text{--}0.59 { \scriptstyle \,\pm\, 0.01 }$ &
    $\text{--}0.84 { \scriptstyle \,\pm\, 9.0\text{\textsc{e}}{\,\text{--}3} }$ &
    $\text{--}0.89 { \scriptstyle \,\pm\, 0.01 }$ &
    $\text{--}1.17 { \scriptstyle \,\pm\, 5.2\text{\textsc{e}}{\,\text{--}3} }$ &
    N/A \\
ConvCNP & & $\text{--}0.70 { \scriptstyle \,\pm\, 0.02 }$ & $\text{--}0.88 { \scriptstyle \,\pm\, 0.01 }$ & $\text{--}0.92 { \scriptstyle \,\pm\, 0.02 }$ & $\text{--}1.19 { \scriptstyle \,\pm\, 7.0\text{\textsc{e}}{\,\text{--}3} }$ & $1.15 { \scriptstyle \,\pm\, 0.04 }$ \\[.5em]
ConvNP & $\gL_{\mathrm{ML}}$ & $\text{--}0.30 { \scriptstyle \,\pm\, 0.02 }$ & $\text{--}0.58 { \scriptstyle \,\pm\, 0.01 }$ & $\text{--}0.55 { \scriptstyle \,\pm\, 0.01 }$ & $\text{--}1.02 { \scriptstyle \,\pm\, 6.0\text{\textsc{e}}{\,\text{--}3} }$ & $2.30 { \scriptstyle \,\pm\, 0.01 }$ \\
ANP & $\gL_{\mathrm{ML}}$ & $\text{--}0.52 { \scriptstyle \,\pm\, 0.01 }$ & $\text{--}0.73 { \scriptstyle \,\pm\, 0.01 }$ & $\text{--}0.69 { \scriptstyle \,\pm\, 0.01 }$ & $\text{--}1.14 { \scriptstyle \,\pm\, 6.0\text{\textsc{e}}{\,\text{--}3} }$ & $0.09 { \scriptstyle \,\pm\, 3.0\text{\textsc{e}}{\,\text{--}3} }$ \\
NP & $\gL_{\mathrm{ML}}$ & $\text{--}0.84 { \scriptstyle \,\pm\, 9.0\text{\textsc{e}}{\,\text{--}3} }$ & $\text{--}0.96 { \scriptstyle \,\pm\, 7.0\text{\textsc{e}}{\,\text{--}3} }$ & $\text{--}0.93 { \scriptstyle \,\pm\, 9.0\text{\textsc{e}}{\,\text{--}3} }$ & $\text{--}1.23 { \scriptstyle \,\pm\, 5.0\text{\textsc{e}}{\,\text{--}3} }$ & $\text{--}0.02 { \scriptstyle \,\pm\, 2.0\text{\textsc{e}}{\,\text{--}3} }$ \\[.5em]
ConvNP & $\gL_{\mathrm{NP}}$ & $\text{--}0.50 { \scriptstyle \,\pm\, 0.02 }$ & $\text{--}0.77 { \scriptstyle \,\pm\, 0.01 }$ & $\text{--}0.48 { \scriptstyle \,\pm\, 0.02 }$ & $\text{--}1.03 { \scriptstyle \,\pm\, 8.0\text{\textsc{e}}{\,\text{--}3} }$ & $2.47 { \scriptstyle \,\pm\, 8.0\text{\textsc{e}}{\,\text{--}3} }$ \\
ANP & $\gL_{\mathrm{NP}}$ & $\text{--}0.82 { \scriptstyle \,\pm\, 0.01 }$ & $\text{--}0.96 { \scriptstyle \,\pm\, 0.01 }$ & $\text{--}1.04 { \scriptstyle \,\pm\, 0.01 }$ & $\text{--}1.37 { \scriptstyle \,\pm\, 6.0\text{\textsc{e}}{\,\text{--}3} }$ & $0.20 { \scriptstyle \,\pm\, 9.0\text{\textsc{e}}{\,\text{--}3} }$ \\
NP & $\gL_{\mathrm{NP}}$ & $\text{--}0.58 { \scriptstyle \,\pm\, 9.0\text{\textsc{e}}{\,\text{--}3} }$ & $\text{--}1.00 { \scriptstyle \,\pm\, 9.0\text{\textsc{e}}{\,\text{--}3} }$ & $\text{--}0.72 { \scriptstyle \,\pm\, 0.01 }$ & $\text{--}1.22 { \scriptstyle \,\pm\, 5.0\text{\textsc{e}}{\,\text{--}3} }$ & $\text{--}0.16 { \scriptstyle \,\pm\, 2.0\text{\textsc{e}}{\,\text{--}3} }$ \\
\midrule\multicolumn{7}{l}{\textsc{Interpolation beyond training range}} \\[0.5em] 
GP (full) & &
    $5.80 { \scriptstyle \,\pm\, 0.02 }$ &
    $1.22 { \scriptstyle \,\pm\, 6.3\text{\textsc{e}}{\,\text{--}3} }$ &
    $1.00 { \scriptstyle \,\pm\, 4.1\text{\textsc{e}}{\,\text{--}3} }$ &
    $\text{--}0.06 { \scriptstyle \,\pm\, 4.6\text{\textsc{e}}{\,\text{--}3} }$ &
    N/A \\
GP (diag) & &
    $\text{--}0.59 { \scriptstyle \,\pm\, 0.01 }$ &
    $\text{--}0.84 { \scriptstyle \,\pm\, 9.0\text{\textsc{e}}{\,\text{--}3} }$ &
    $\text{--}0.89 { \scriptstyle \,\pm\, 0.01 }$ &
    $\text{--}1.17 { \scriptstyle \,\pm\, 5.2\text{\textsc{e}}{\,\text{--}3} }$ &
    N/A \\
ConvCNP & & $\text{--}0.69 { \scriptstyle \,\pm\, 0.02 }$ & $\text{--}0.87 { \scriptstyle \,\pm\, 0.01 }$ & $\text{--}0.94 { \scriptstyle \,\pm\, 0.02 }$ & $\text{--}1.19 { \scriptstyle \,\pm\, 7.0\text{\textsc{e}}{\,\text{--}3} }$ & $1.11 { \scriptstyle \,\pm\, 0.04 }$ \\[.5em]
ConvNP & $\gL_{\mathrm{ML}}$ & $\text{--}0.30 { \scriptstyle \,\pm\, 0.02 }$ & $\text{--}0.58 { \scriptstyle \,\pm\, 0.01 }$ & $\text{--}0.56 { \scriptstyle \,\pm\, 0.01 }$ & $\text{--}1.03 { \scriptstyle \,\pm\, 6.0\text{\textsc{e}}{\,\text{--}3} }$ & $2.29 { \scriptstyle \,\pm\, 0.02 }$ \\
ANP & $\gL_{\mathrm{ML}}$ & $\text{--}1.35 { \scriptstyle \,\pm\, 6.0\text{\textsc{e}}{\,\text{--}3} }$ & $\text{--}1.39 { \scriptstyle \,\pm\, 7.0\text{\textsc{e}}{\,\text{--}3} }$ & $\text{--}1.65 { \scriptstyle \,\pm\, 5.0\text{\textsc{e}}{\,\text{--}3} }$ & $\text{--}1.35 { \scriptstyle \,\pm\, 4.0\text{\textsc{e}}{\,\text{--}3} }$ & $\text{--}0.17 { \scriptstyle \,\pm\, 1.0\text{\textsc{e}}{\,\text{--}3} }$ \\
NP & $\gL_{\mathrm{ML}}$ & $\text{--}2.70 { \scriptstyle \,\pm\, 3.0\text{\textsc{e}}{\,\text{--}3} }$ & $\text{--}2.60 { \scriptstyle \,\pm\, 3.0\text{\textsc{e}}{\,\text{--}3} }$ & $\text{--}2.82 { \scriptstyle \,\pm\, 3.0\text{\textsc{e}}{\,\text{--}3} }$ & - & $\text{--}0.03 { \scriptstyle \,\pm\, 2.0\text{\textsc{e}}{\,\text{--}3} }$ \\[.5em]
ConvNP & $\gL_{\mathrm{NP}}$ & $\text{--}0.48 { \scriptstyle \,\pm\, 0.02 }$ & $\text{--}0.79 { \scriptstyle \,\pm\, 0.01 }$ & $\text{--}0.48 { \scriptstyle \,\pm\, 0.02 }$ & $\text{--}1.04 { \scriptstyle \,\pm\, 8.0\text{\textsc{e}}{\,\text{--}3} }$ & $2.47 { \scriptstyle \,\pm\, 8.0\text{\textsc{e}}{\,\text{--}3} }$ \\
ANP & $\gL_{\mathrm{NP}}$ & $\text{--}1.91 { \scriptstyle \,\pm\, 0.03 }$ & $\text{--}1.48 { \scriptstyle \,\pm\, 4.0\text{\textsc{e}}{\,\text{--}3} }$ & $\text{--}1.85 { \scriptstyle \,\pm\, 7.0\text{\textsc{e}}{\,\text{--}3} }$ & $\text{--}1.66 { \scriptstyle \,\pm\, 0.01 }$ & $\text{--}0.30 { \scriptstyle \,\pm\, 4.0\text{\textsc{e}}{\,\text{--}3} }$ \\
NP & $\gL_{\mathrm{NP}}$ & $\text{--}13.7 { \scriptstyle \,\pm\, 0.82 }$ & $\text{--}3.96 { \scriptstyle \,\pm\, 0.04 }$ & $\text{--}3.80 { \scriptstyle \,\pm\, 0.02 }$ & - & $\text{--}4.98 { \scriptstyle \,\pm\, 0.02 }$ \\
\midrule\multicolumn{7}{l}{\textsc{Extrapolation beyond training range}} \\[0.5em] 
GP (full) & &
    $4.29 { \scriptstyle \,\pm\, 6.2\text{\textsc{e}}{\,\text{--}3} }$ &
    $0.82 { \scriptstyle \,\pm\, 4.3\text{\textsc{e}}{\,\text{--}3} }$ &
    $0.66 { \scriptstyle \,\pm\, 2.2\text{\textsc{e}}{\,\text{--}3} }$ &
    $\text{--}0.33 { \scriptstyle \,\pm\, 3.4\text{\textsc{e}}{\,\text{--}3} }$ &
    N/A \\
GP (diag) & &
    $\text{--}1.40 { \scriptstyle \,\pm\, 5.0\text{\textsc{e}}{\,\text{--}3} }$ &
    $\text{--}1.41 { \scriptstyle \,\pm\, 4.8\text{\textsc{e}}{\,\text{--}3} }$ &
    $\text{--}1.72 { \scriptstyle \,\pm\, 6.2\text{\textsc{e}}{\,\text{--}3} }$ &
    $\text{--}1.40 { \scriptstyle \,\pm\, 4.0\text{\textsc{e}}{\,\text{--}3} }$ &
    N/A \\
ConvCNP & & $\text{--}1.41 { \scriptstyle \,\pm\, 6.0\text{\textsc{e}}{\,\text{--}3} }$ & $\text{--}1.41 { \scriptstyle \,\pm\, 7.0\text{\textsc{e}}{\,\text{--}3} }$ & $\text{--}1.73 { \scriptstyle \,\pm\, 8.0\text{\textsc{e}}{\,\text{--}3} }$ & $\text{--}1.41 { \scriptstyle \,\pm\, 6.0\text{\textsc{e}}{\,\text{--}3} }$ & $0.27 { \scriptstyle \,\pm\, 0.02 }$ \\[.5em]
ConvNP & $\gL_{\mathrm{ML}}$ & $\text{--}1.09 { \scriptstyle \,\pm\, 5.0\text{\textsc{e}}{\,\text{--}3} }$ & $\text{--}1.11 { \scriptstyle \,\pm\, 5.0\text{\textsc{e}}{\,\text{--}3} }$ & $\text{--}1.30 { \scriptstyle \,\pm\, 4.0\text{\textsc{e}}{\,\text{--}3} }$ & $\text{--}1.24 { \scriptstyle \,\pm\, 4.0\text{\textsc{e}}{\,\text{--}3} }$ & $1.61 { \scriptstyle \,\pm\, 0.02 }$ \\
ANP & $\gL_{\mathrm{ML}}$ & $\text{--}1.29 { \scriptstyle \,\pm\, 6.0\text{\textsc{e}}{\,\text{--}3} }$ & $\text{--}1.29 { \scriptstyle \,\pm\, 5.0\text{\textsc{e}}{\,\text{--}3} }$ & $\text{--}1.55 { \scriptstyle \,\pm\, 5.0\text{\textsc{e}}{\,\text{--}3} }$ & $\text{--}1.34 { \scriptstyle \,\pm\, 5.0\text{\textsc{e}}{\,\text{--}3} }$ & $\text{--}0.25 { \scriptstyle \,\pm\, 2.0\text{\textsc{e}}{\,\text{--}3} }$ \\
NP & $\gL_{\mathrm{ML}}$ & $\text{--}2.23 { \scriptstyle \,\pm\, 4.0\text{\textsc{e}}{\,\text{--}3} }$ & $\text{--}2.08 { \scriptstyle \,\pm\, 3.0\text{\textsc{e}}{\,\text{--}3} }$ & $\text{--}2.50 { \scriptstyle \,\pm\, 4.0\text{\textsc{e}}{\,\text{--}3} }$ & $\text{--}1.39 { \scriptstyle \,\pm\, 4.0\text{\textsc{e}}{\,\text{--}3} }$ & $\text{--}0.06 { \scriptstyle \,\pm\, 2.0\text{\textsc{e}}{\,\text{--}3} }$ \\[.5em]
ConvNP & $\gL_{\mathrm{NP}}$ & $\text{--}1.21 { \scriptstyle \,\pm\, 0.01 }$ & $\text{--}1.31 { \scriptstyle \,\pm\, 0.01 }$ & $\text{--}1.19 { \scriptstyle \,\pm\, 0.01 }$ & $\text{--}1.51 { \scriptstyle \,\pm\, 8.0\text{\textsc{e}}{\,\text{--}3} }$ & $2.10 { \scriptstyle \,\pm\, 7.0\text{\textsc{e}}{\,\text{--}3} }$ \\
ANP & $\gL_{\mathrm{NP}}$ & $\text{--}1.44 { \scriptstyle \,\pm\, 6.0\text{\textsc{e}}{\,\text{--}3} }$ & $\text{--}1.45 { \scriptstyle \,\pm\, 6.0\text{\textsc{e}}{\,\text{--}3} }$ & $\text{--}1.77 { \scriptstyle \,\pm\, 7.0\text{\textsc{e}}{\,\text{--}3} }$ & $\text{--}1.46 { \scriptstyle \,\pm\, 6.0\text{\textsc{e}}{\,\text{--}3} }$ & $\text{--}0.20 { \scriptstyle \,\pm\, 2.0\text{\textsc{e}}{\,\text{--}3} }$ \\
NP & $\gL_{\mathrm{NP}}$ & $\text{--}5.85 { \scriptstyle \,\pm\, 0.05 }$ & $\text{--}2.65 { \scriptstyle \,\pm\, 3.0\text{\textsc{e}}{\,\text{--}3} }$ & $\text{--}4.06 { \scriptstyle \,\pm\, 0.04 }$ & $\text{--}1.49 { \scriptstyle \,\pm\, 5.0\text{\textsc{e}}{\,\text{--}3} }$ & $\text{--}1.99 { \scriptstyle \,\pm\, 6.0\text{\textsc{e}}{\,\text{--}3} }$ \\
\bottomrule

\end{tabular}
\end{center}
\end{table}

%% file: plots/image_completion/data/zsmm_data.tex
\begin{figure}[htb]
\centering
\begin{subfigure}{0.40\textwidth}
  \includegraphics[width=\linewidth]{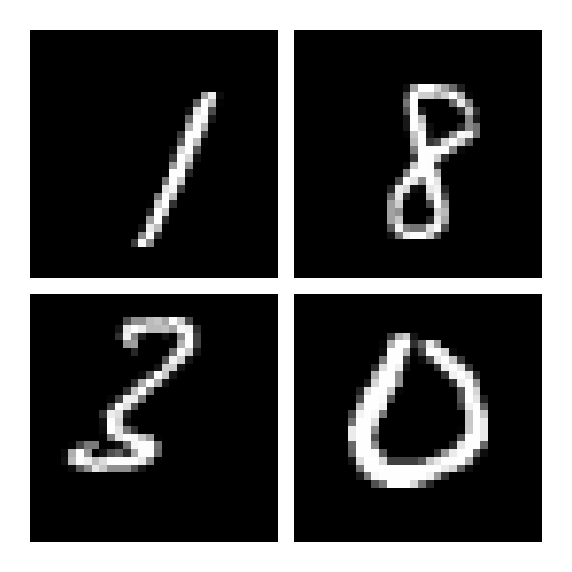}
  \subcaption{Train ($32  \times 32)$}
  \label{fig:zsmm_train}
\end{subfigure}
\begin{subfigure}{0.40\textwidth}
  \includegraphics[width=\linewidth]{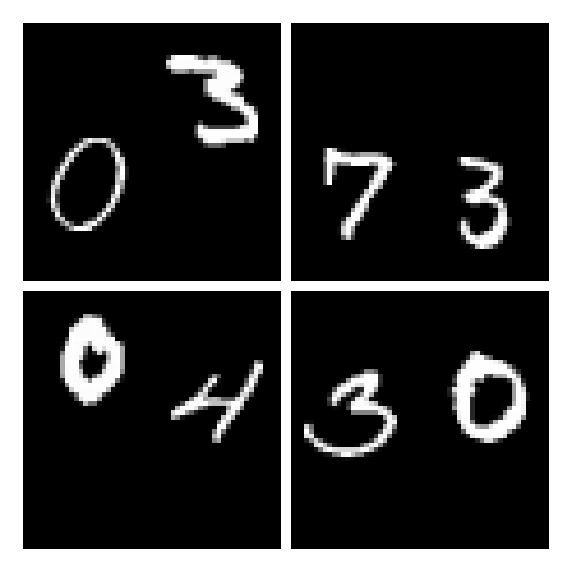}
  \subcaption{Test ($56 \times 56)$}
  \label{fig:zsmm_test}
\end{subfigure}

\caption{Samples from our generated Zero Shot Multi MNIST (ZSMM) data set.}
\label{fig:zsmm}
\end{figure}

%% file: plots/image_completion/samples_convnp/samples_convnp_mixed.tex
\begin{figure}[t]
\centering
\begin{subfigure}{0.47\columnwidth}
    \includegraphics[width=\textwidth]{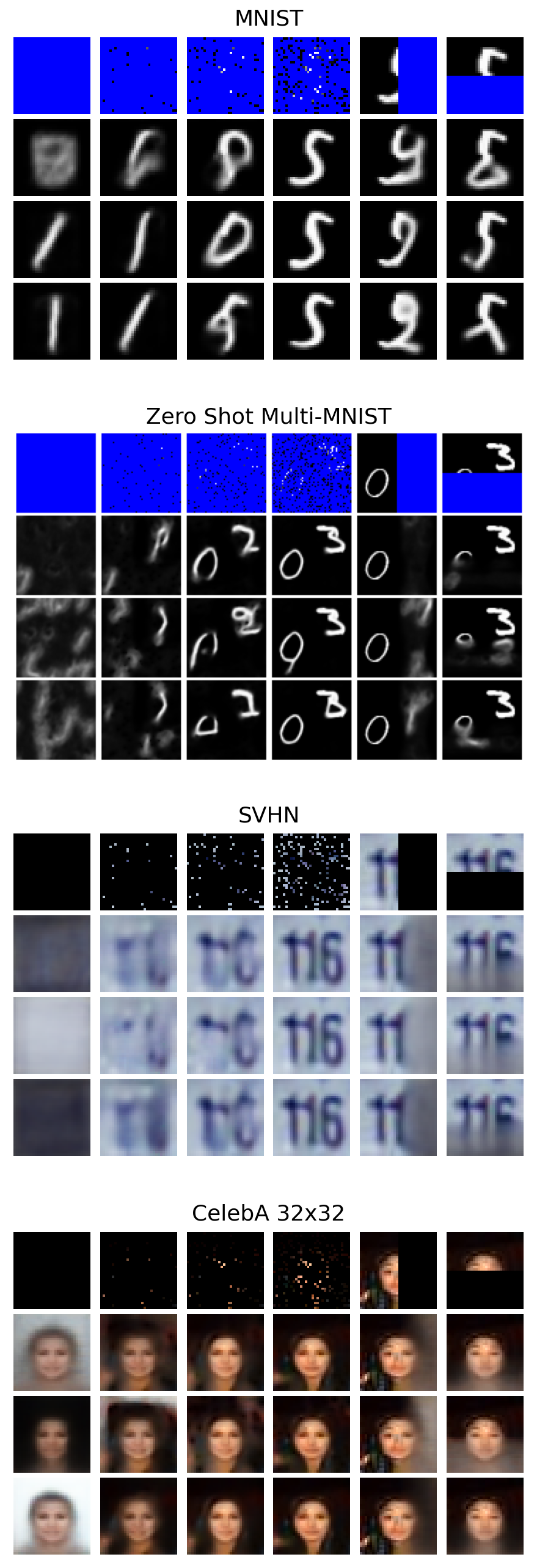}
\end{subfigure}
%
%\begin{subfigure}{0.32\columnwidth}
%
    %\includegraphics[width=\textwidth]{plots/image_completion/samples_convnp/samples_convnp_2nd.png}
%
%\end{subfigure}
%
\begin{subfigure}{0.47\columnwidth}
    \includegraphics[width=\textwidth]{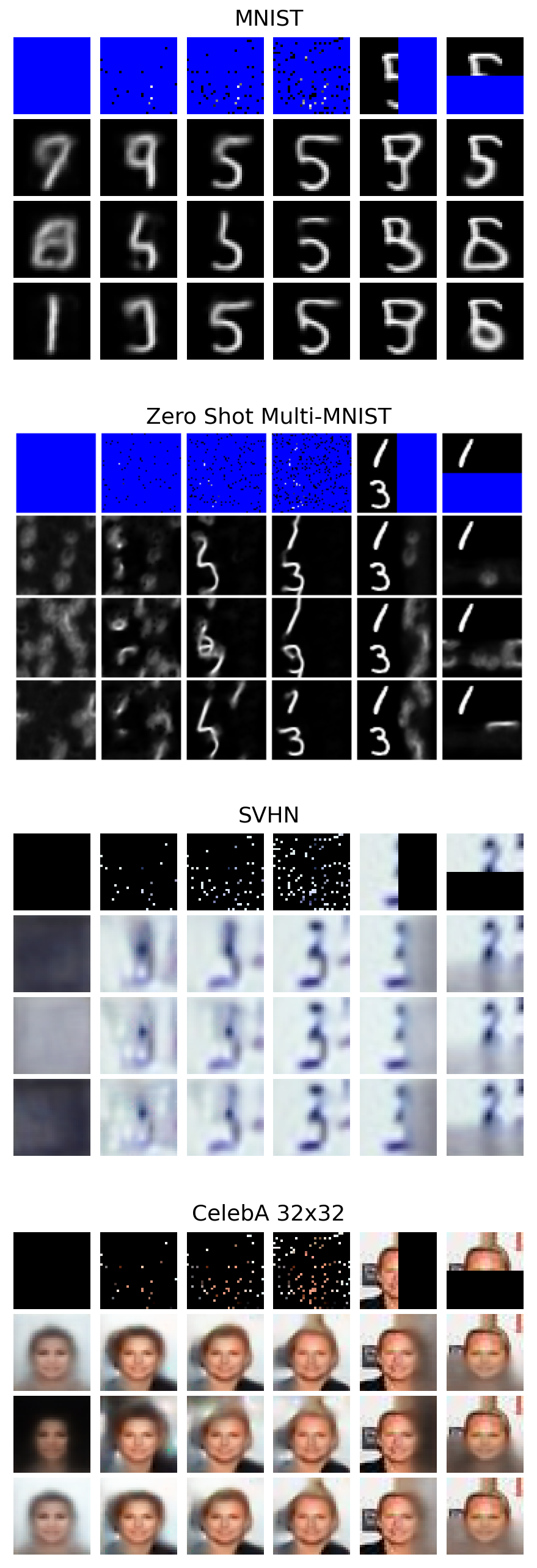}
\end{subfigure}
\caption{
Qualitative samples for one of the ConvNP trained with $\gL_{\mathrm{ML}}$ in \cref{table:images_main_text}.
From top to bottom the four major rows correspond to MNIST, ZSMM, SVHN, CelebA32 datasets.
For each dataset and each of the two major columns, a different image is randomly
sampled; the first sub-row shows the given context points (missing pixels are in blue for MNIST and ZSMM but in black for SVHN and CelebA), while the next three sub-rows show the mean of the posterior predictive corresponding to different samples of the latent function.
To show diverse samples we select three samples that maximize the average Euclidean distance between pixels of the samples. 
From left to right the first four sub-columns correspond to a context set with 0\%,
1\%, 3\%, 10\% randomly sampled context points. 
In the last two sub-columns, the
context sets respectively contain all the pixels in the left and top half of the image.
}
\label{fig:samples_convnp}
\end{figure}

%% file: plots/image_completion/samples_anp_convnp/samples_anp_convnp_mixed.tex
\begin{figure}[t]
\centering
\begin{subfigure}{0.32\columnwidth}
    \includegraphics[width=\textwidth]{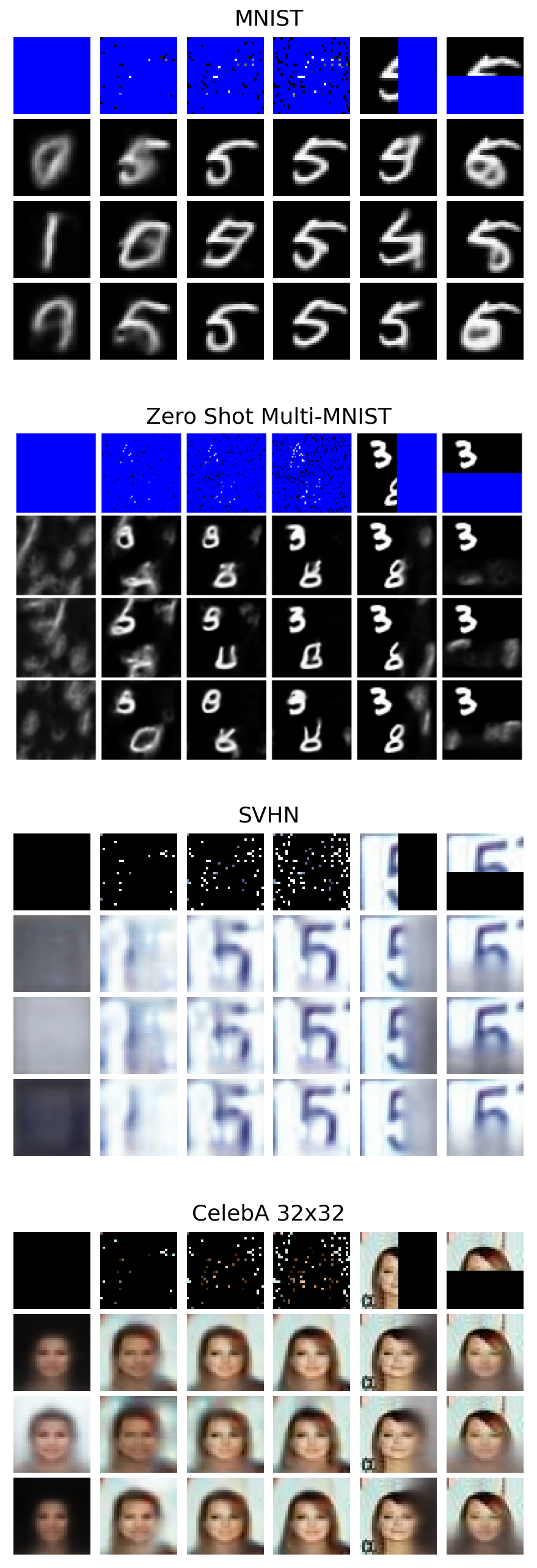}
    \caption{ConvNP  $\gL_{\mathrm{ML}}$}
\end{subfigure}
\begin{subfigure}{0.32\columnwidth}
    \includegraphics[width=\textwidth]{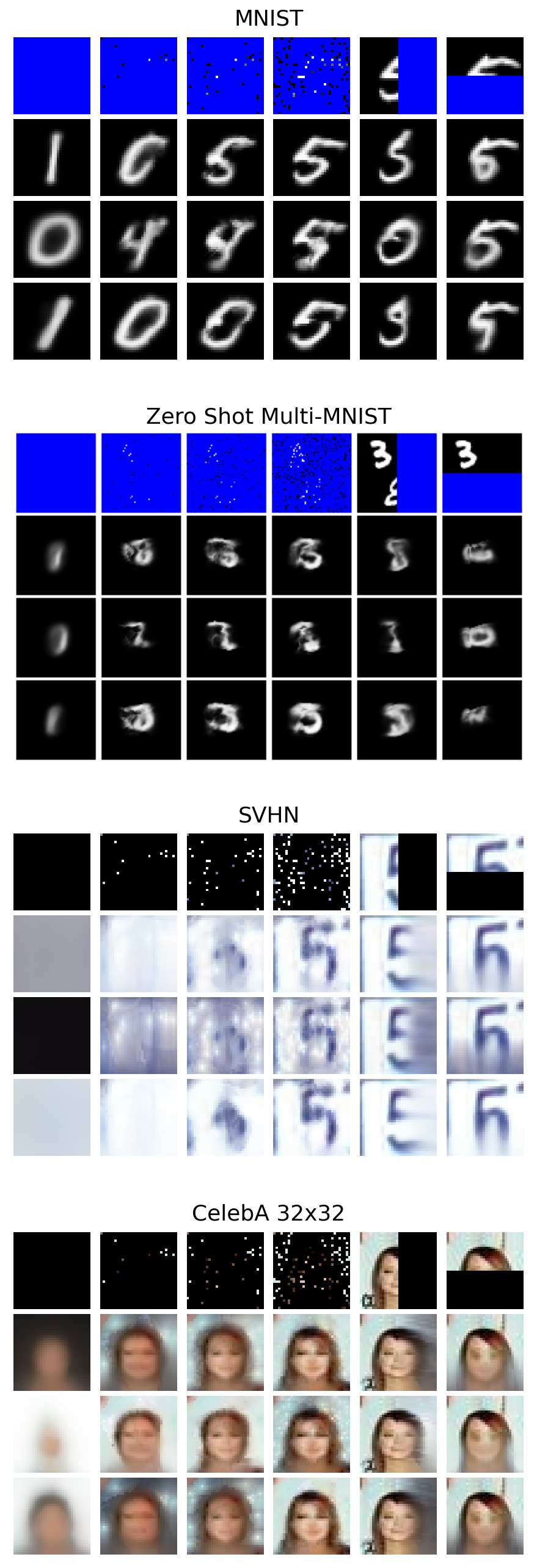}
    \caption{ANP  $\gL_{\mathrm{ML}}$}
\end{subfigure}
\begin{subfigure}{0.32\columnwidth}
    \includegraphics[width=\textwidth]{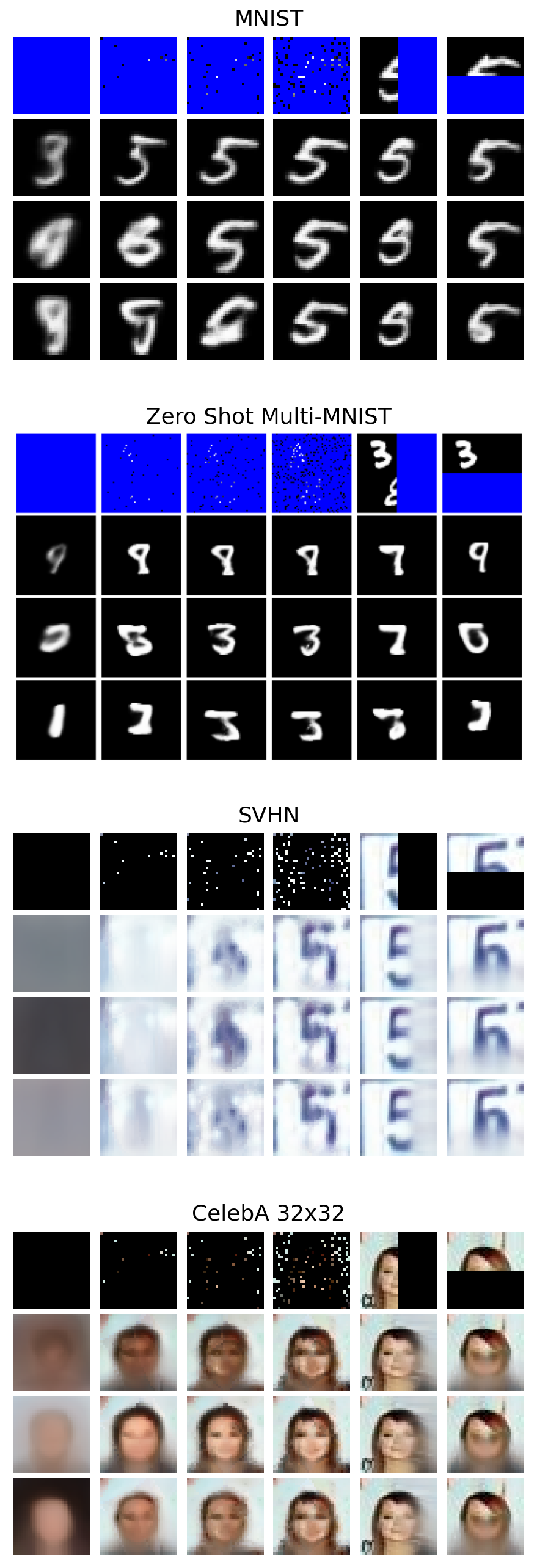}
    \caption{ANP  $\gL_{\mathrm{NP}}$}
\end{subfigure}
\caption{
Qualitative samples between (a) ConvNP trained with $\gL_{\mathrm{ML}}$; (b) ANP trained with $\gL_{\mathrm{ML}}$; (c) ANP trained with $\gL_{\mathrm{NP}}$.
For each model the figure shows the same as \cref{fig:samples_convnp}.
}
\label{fig:samples_anp_convnp}
\end{figure}

%% file: plots/image_completion/samples_anp_convnp_kde/samples_anp_convnp_kde_mixed.tex
\begin{figure}[t]
\centering
\begin{subfigure}{0.47\columnwidth}
    \includegraphics[width=\textwidth]{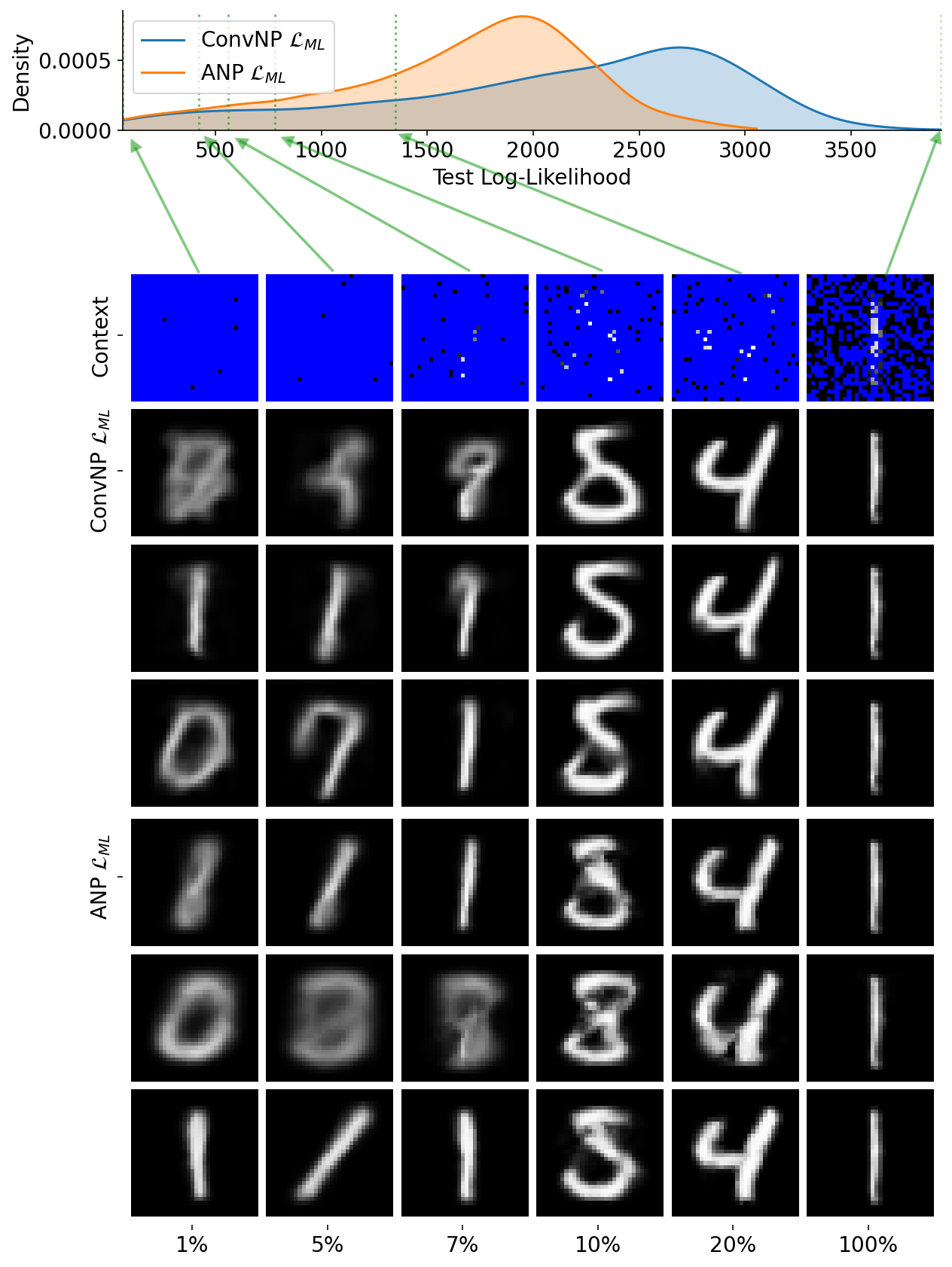}
    \caption{MNIST}
\end{subfigure}
\begin{subfigure}{0.47\columnwidth}
    \includegraphics[width=\textwidth]{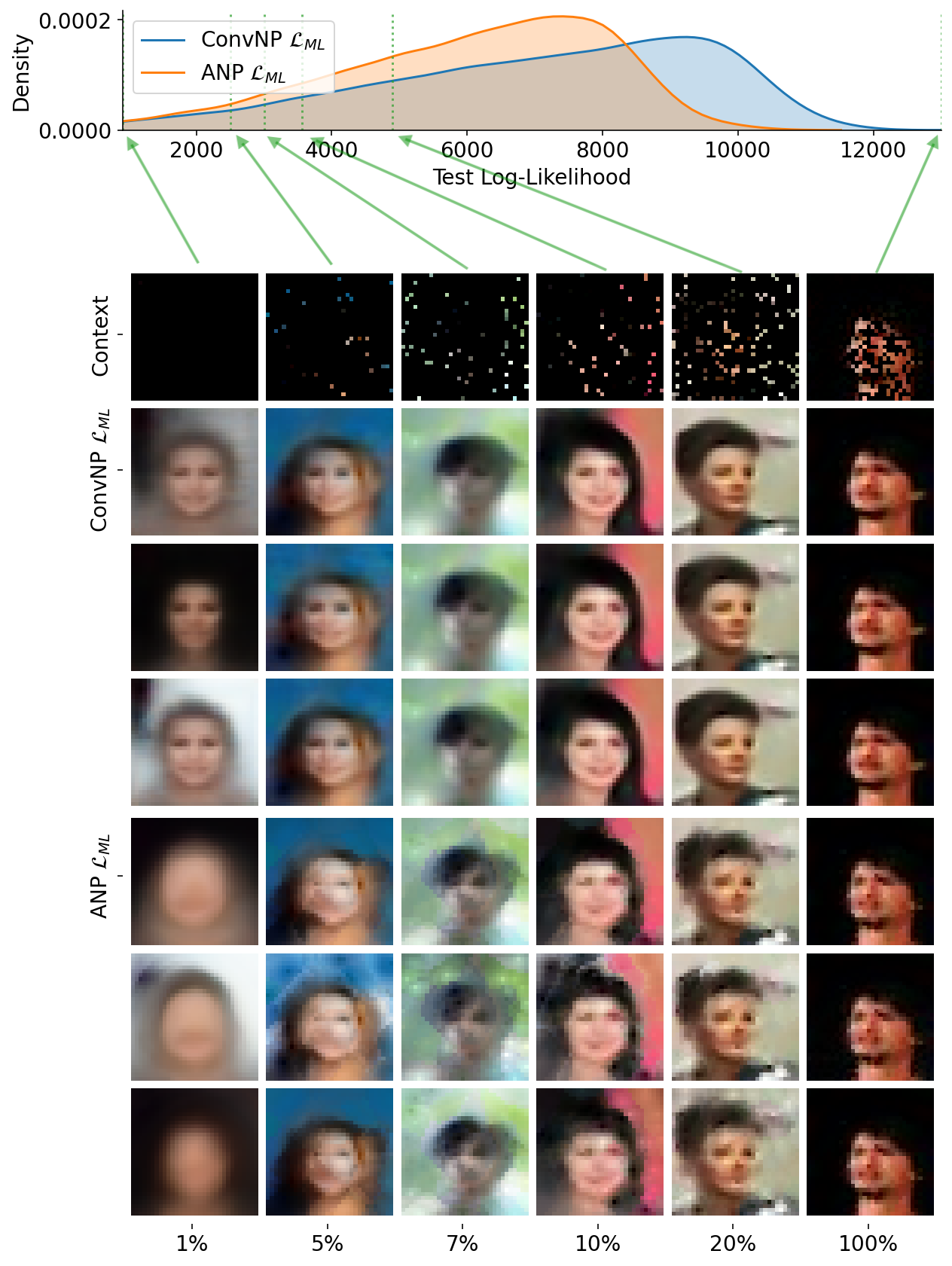}
    \caption{CelebA32}
\end{subfigure}
\hfill
\begin{subfigure}{0.47\columnwidth}
    \includegraphics[width=\textwidth]{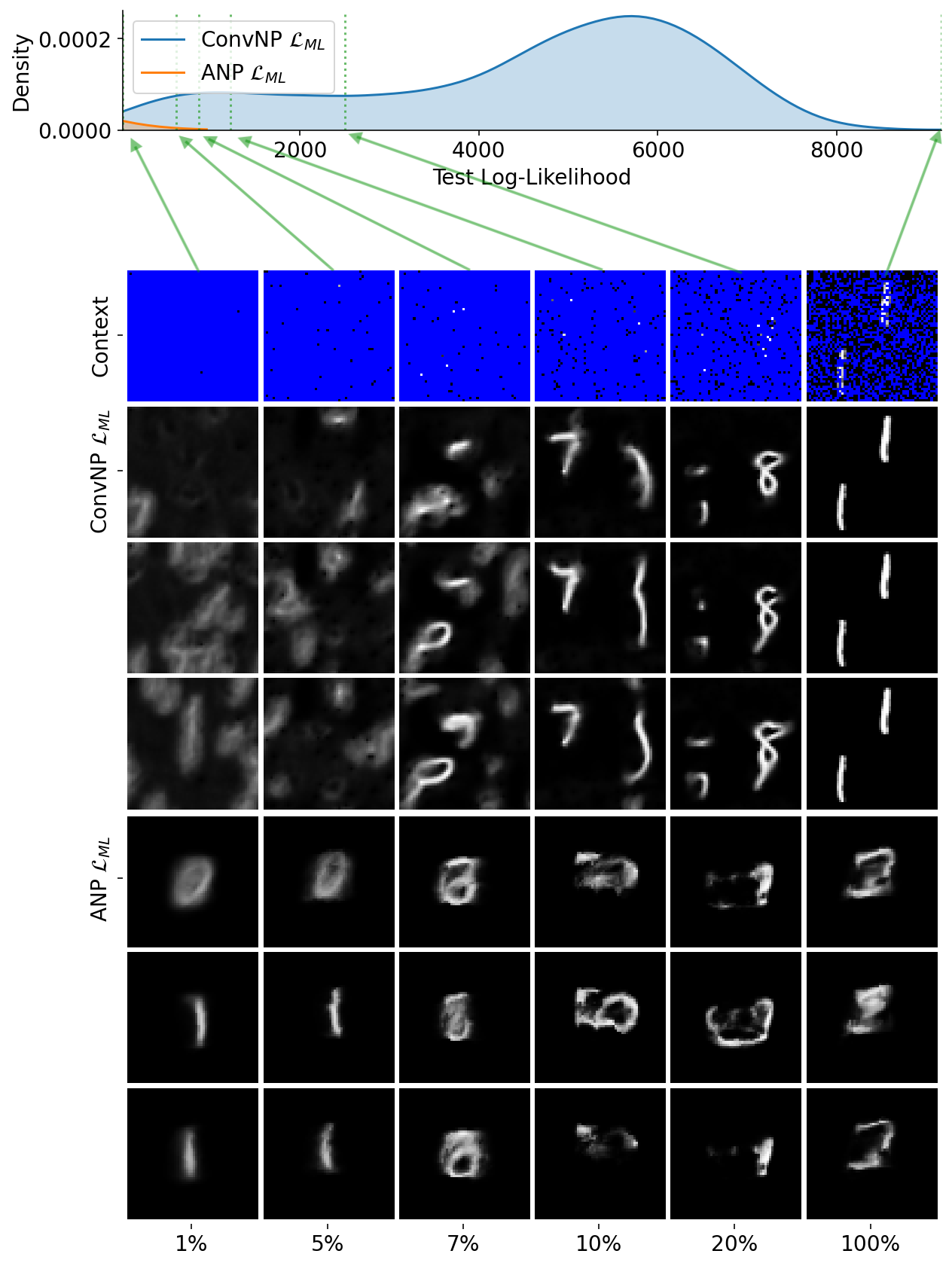}
    \caption{Zero Shot Multi-MNIST}
\end{subfigure}
\begin{subfigure}{0.47\columnwidth}
    \includegraphics[width=\textwidth]{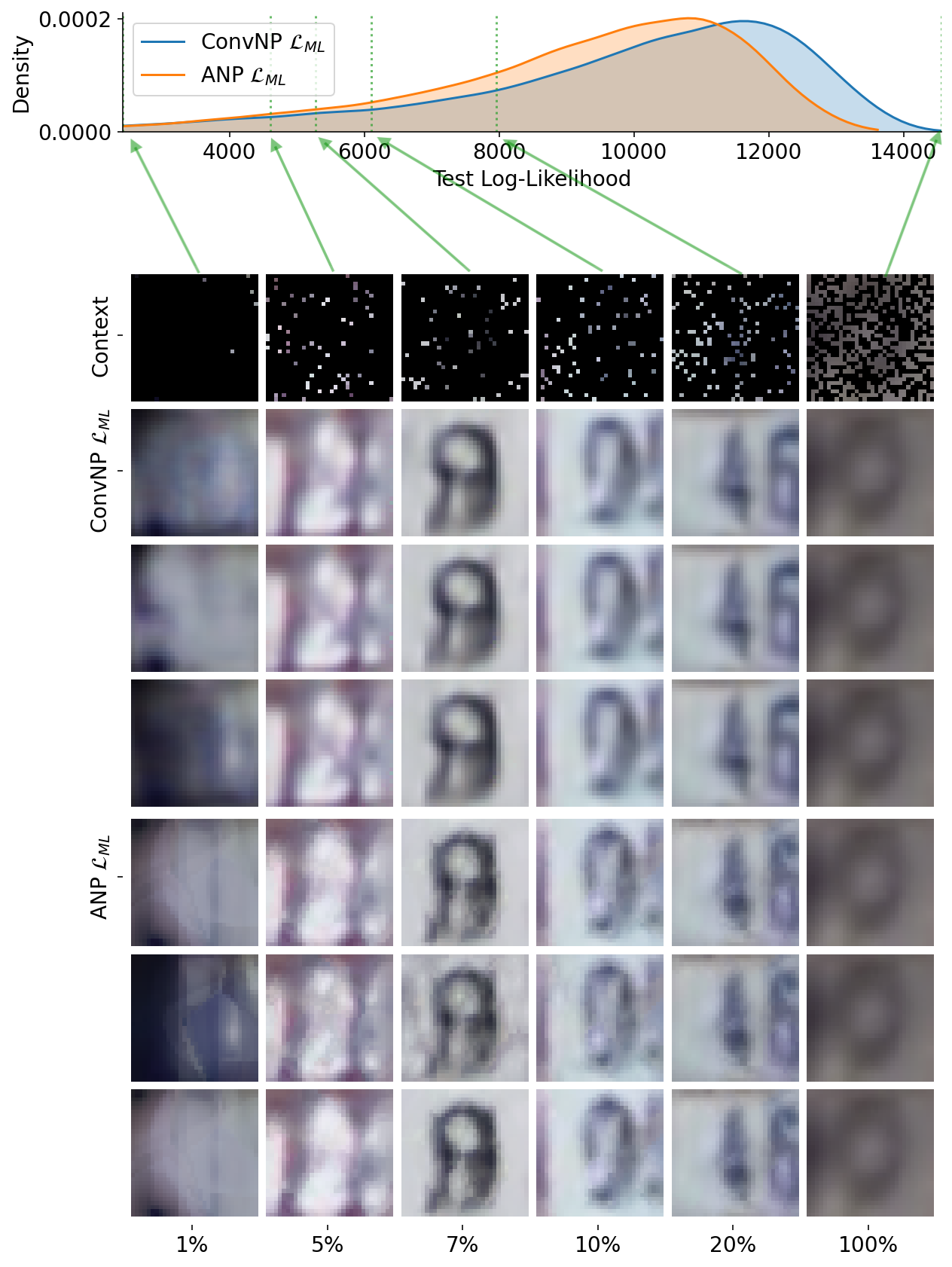}
    \caption{SVHN}
\end{subfigure}
\caption{
Log-likelihood and qualitative samples comparing ConvNP and ANP trained with $\gL_{\mathrm{ML}}$ on (a) MNIST; (b) CelebA; (c) ZSMM; (d) SVHN. For each sub-figure, the top row shows the log-likelihood distribution for both models.
The images below correspond to the context points (top), followed by three samples form ConvNP (mean of the posterior predictive corresponding to different samples from the latent function),  and three samples from ANP. 
Each column corresponds to a given percentile of the ConvNP test log likelihood (as shown by green arrows).
}
\label{fig:samples_anp_convnp_kde}
\end{figure}

%% file: tables/region_coordinates.tex
% Central Europe (8, 28) (52, 46) 
% West box (1, 8) (50, 46)
% East box (28,35) (52, 49)
% South box (19, 26) (46, 42)

\begin{table}[htb] 
\caption{Coordinates for boxes defining the train and test regions. Latitidues are given as (north, south), and longitudes as (west, east).}
\label{table:era5_region_coordinates}
\begin{center}
\begin{tabular}{l@{\hspace{4pt}}
c@{\hspace{4pt}}
c@{\hspace{4pt}}
c@{\hspace{4pt}}
c@{\hspace{4pt}}}
\toprule 
 & 
 \multicolumn{1}{c}{Central (train)} & 
 \multicolumn{1}{c}{Western (test)} & 
 \multicolumn{1}{c}{Eastern (test)} & 
 \multicolumn{1}{c}{Southern (test)}\\
 \midrule 
 Latitudes  & 
 $(52, 46)$ & 
 $(50, 46)$ & 
 $(52, 49)$ & 
 $(46, 42)$ \\
 Longitudes & 
 $(08, 28)$ & 
 $(01, 08)$ & 
 $(28, 35)$ & 
 $(19, 26)$ \\
 \bottomrule
\end{tabular}
\end{center}
\end{table}

%% file: plots/europe_regions.tex
\begin{figure}[t]
    \centering
    \includegraphics[width=.8\textwidth]{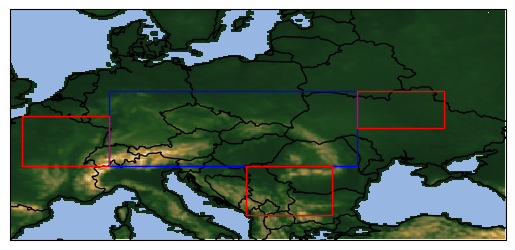}
    \caption{Training (blue) and test (red) regions in Europe, along with orography data from ERA5Land.}
    \label{fig:eu_data_regions}
\end{figure}

%% file: plots/kde_plots.tex
\begin{figure}[t]
    \centering
    \begin{subfigure}{0.49\columnwidth}
        \includegraphics[width=\textwidth]{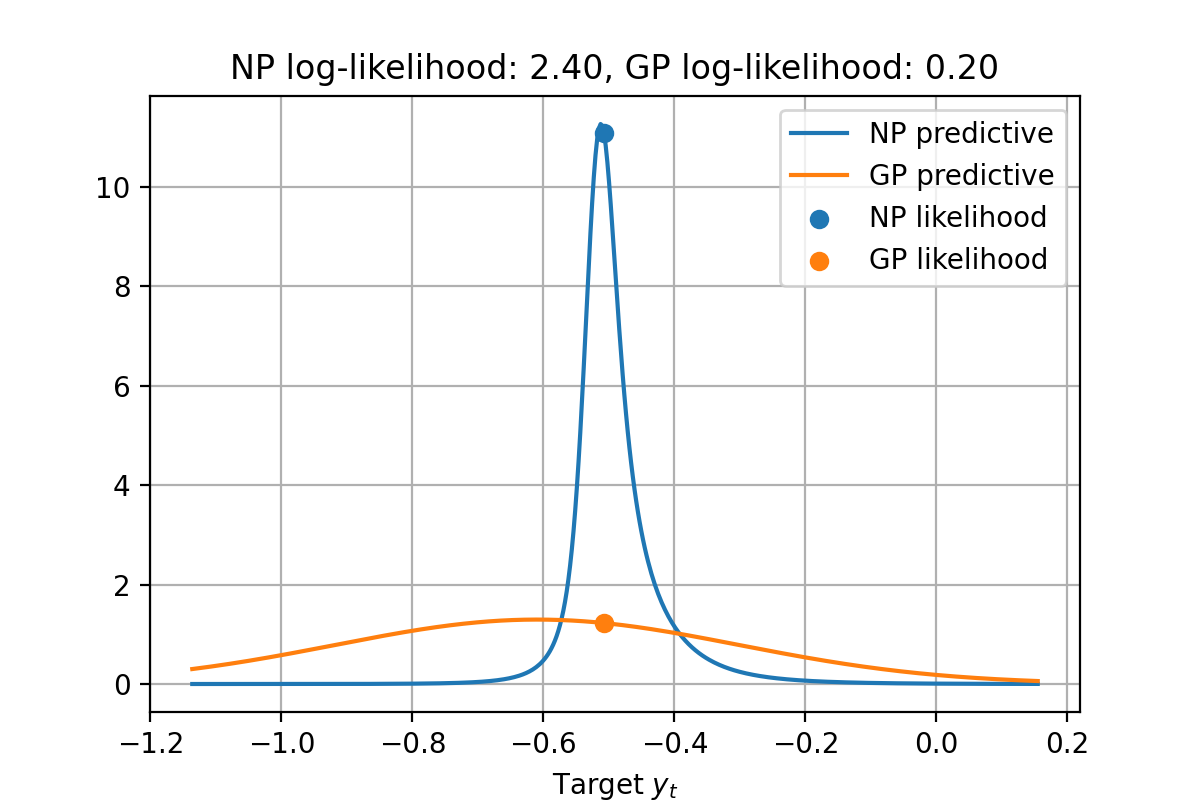}
        \subcaption{}
        \label{fig:pred_density_sparse}
    \end{subfigure}
    \begin{subfigure}{0.49\columnwidth}
        \includegraphics[width=\textwidth]{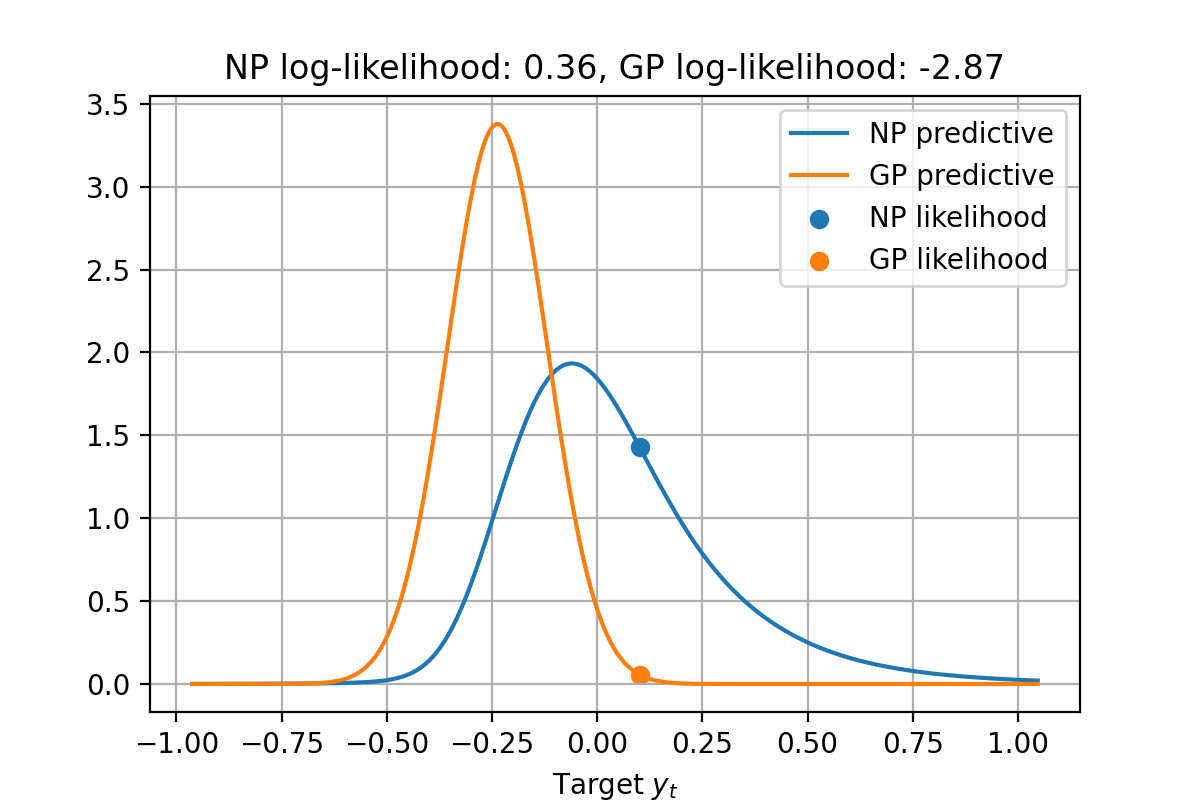}
        \subcaption{}
        \label{fig:pred_density_asymm}
    \end{subfigure}
    \caption{Predictive density at two target points, where the ConvNP significantly outperforms the GP. The orange and blue circles show the likelihood of the ground truth target value under the GP and ConvNP. Note that as the precipitation values are normalized to zero mean and unit standard deviation, $y_t = -0.53$ corresponds to no rain. In \cref{fig:pred_density_sparse}, we see the ConvNP sometimes produces predictions heavily centered on this value, showing it has learned the sparsity of precipitation values. In \cref{fig:pred_density_asymm} we see the ConvNP predictive distribution is sometimes asymmetric with a heavier positive tail, reflecting the non-negativity of precipitation.}
    \label{fig:pred_density}
\end{figure}

%% file: plots/prec_samples/prec_samples_app_28x28.tex
\begin{figure}[t]
%%
     %%%  First row: Ground truth and NP samples  %%%
%%
\vspace*{-0.8em}
\begin{subfigure}{0.24\textwidth}
  \includegraphics[width=\linewidth]{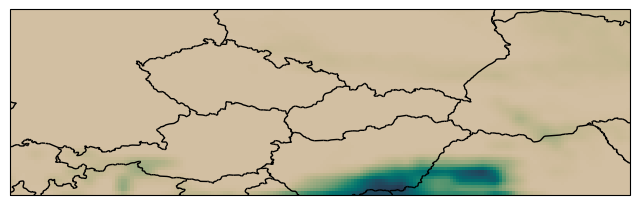}
  \vspace{-1.5em}
  \caption{Ground truth data}
\end{subfigure}\hfil
\begin{subfigure}{0.24\textwidth}
  \includegraphics[width=\linewidth]{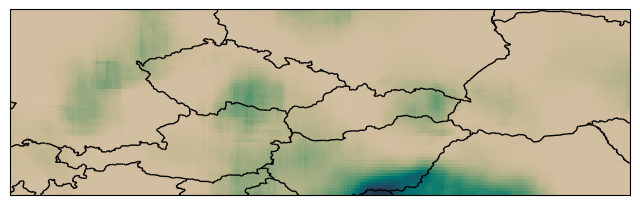}
  \vspace{-1.5em}
  \caption{ConvNP sample 1}
\end{subfigure}\hfil 
\begin{subfigure}{0.24\textwidth}
  \includegraphics[width=\linewidth]{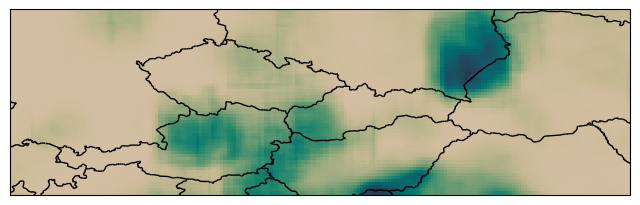}
  \vspace{-1.5em}
  \caption{ConvNP sample 2}
\end{subfigure}\hfil 
\begin{subfigure}{0.24\textwidth}
  \includegraphics[width=\linewidth]{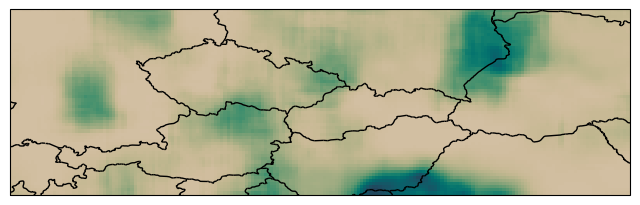}
  \vspace{-1.5em}
  \caption{ConvNP sample 3}
\end{subfigure}\\
%%
     %%%  Second row: Context set and GP samples  %%%
%%
\begin{subfigure}{0.24\textwidth}
  \includegraphics[width=\linewidth]{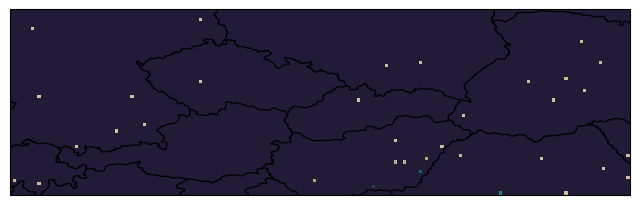}
  \vspace{-1.5em}
  \caption{Context set}
\end{subfigure}\hfil 
\begin{subfigure}{0.24\textwidth}
  \includegraphics[width=\linewidth]{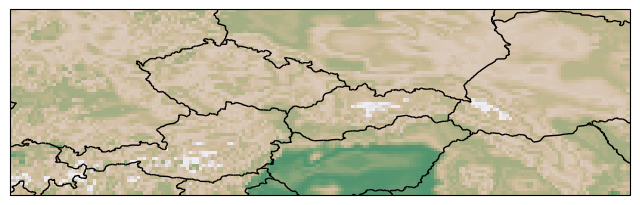}
  \vspace{-1.5em}
  \caption{GP sample 1}
\end{subfigure}\hfil
\begin{subfigure}{0.24\textwidth}
  \includegraphics[width=\linewidth]{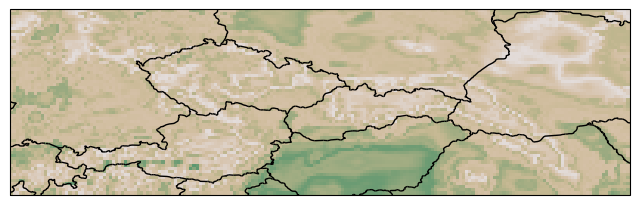}
  \vspace{-1.5em}
  \caption{GP sample 2}
\end{subfigure}\hfil
\begin{subfigure}{0.24\textwidth}
  \includegraphics[width=\linewidth]{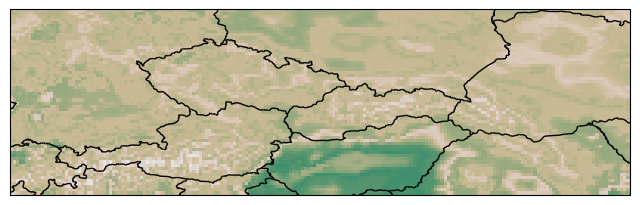}
  \vspace{-1.5em}
  \caption{GP sample 3}
\end{subfigure}
\caption{Samples from the predictive processes overlaid on central Europe, for  a model trained on random $28 \times 28$ subregions of the full $61 \times 201$ central Europe region. Note some blocky artefacts in the ConvNP samples due to training on small subregions. Here the GP has overfit to the orography data, with samples that resemble the orography rather than precipitation.}
\label{fig:precipitation_sample_28x28}
\end{figure}

%% file: plots/prec_samples/prec_samples_app_28x28_2.tex
\begin{figure}[th!]
%%
     %%%  First row: Ground truth and NP samples  %%%
%%
\vspace*{-0.8em}
\begin{subfigure}{0.24\textwidth}
  \includegraphics[width=\linewidth]{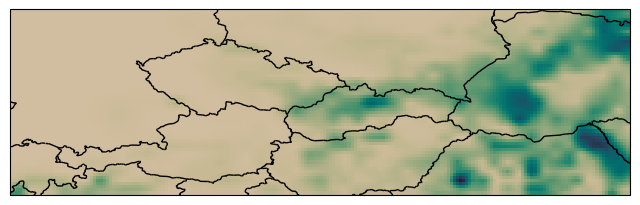}
  \vspace{-1.5em}
  \caption{Ground truth data}
\end{subfigure}\hfil
\begin{subfigure}{0.24\textwidth}
  \includegraphics[width=\linewidth]{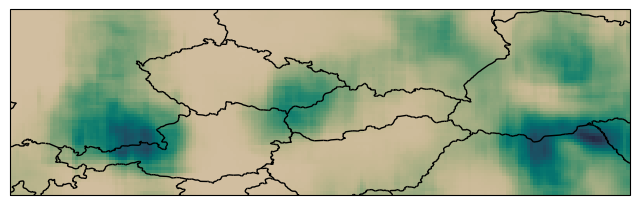}
  \vspace{-1.5em}
  \caption{ConvNP sample 1}
\end{subfigure}\hfil 
\begin{subfigure}{0.24\textwidth}
  \includegraphics[width=\linewidth]{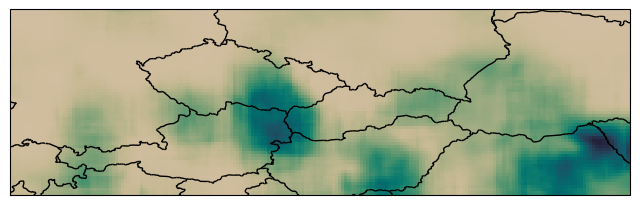}
  \vspace{-1.5em}
  \caption{ConvNP sample 2}
\end{subfigure}\hfil 
\begin{subfigure}{0.24\textwidth}
  \includegraphics[width=\linewidth]{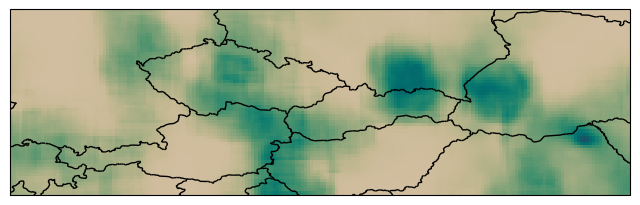}
  \vspace{-1.5em}
  \caption{ConvNP sample 3}
\end{subfigure}\\
%%
     %%%  Second row: Context set and GP samples  %%%
%%
\begin{subfigure}{0.24\textwidth}
  \includegraphics[width=\linewidth]{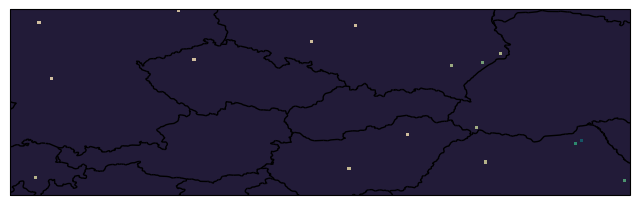}
  \vspace{-1.5em}
  \caption{Context set}
\end{subfigure}\hfil 
\begin{subfigure}{0.24\textwidth}
  \includegraphics[width=\linewidth]{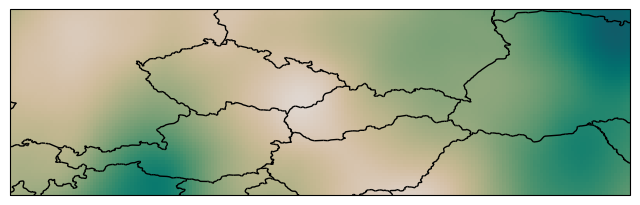}
  \vspace{-1.5em}
  \caption{GP sample 1}
\end{subfigure}\hfil
\begin{subfigure}{0.24\textwidth}
  \includegraphics[width=\linewidth]{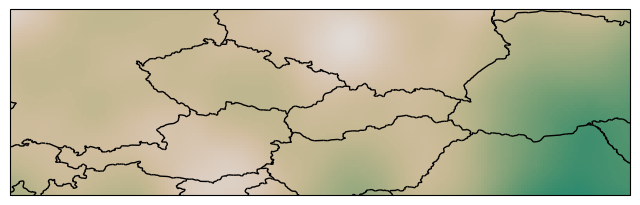}
  \vspace{-1.5em}
  \caption{GP sample 2}
\end{subfigure}\hfil
\begin{subfigure}{0.24\textwidth}
  \includegraphics[width=\linewidth]{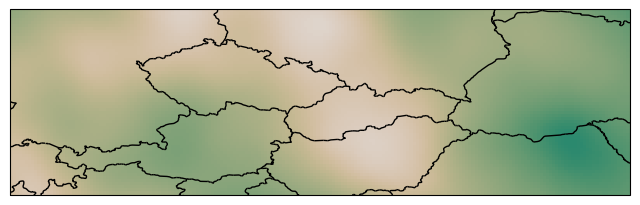}
  \vspace{-1.5em}
  \caption{GP sample 3}
\end{subfigure}
\caption{Samples from the predictive processes overlaid on central Europe, for  a model trained on random $28 \times 28$ subregions of the full $61 \times 201$ central Europe region. Here the GP has learned a lengthscale that is too large.}
\label{fig:precipitation_sample_28x28_2}
\end{figure}

%% file: plots/prec_samples/prec_samples_app_40x40.tex
\begin{figure}[t]
%%
     %%%  First row: Ground truth and NP samples  %%%
%%
\vspace*{-0.8em}
\begin{subfigure}{0.24\textwidth}
  \includegraphics[width=\linewidth]{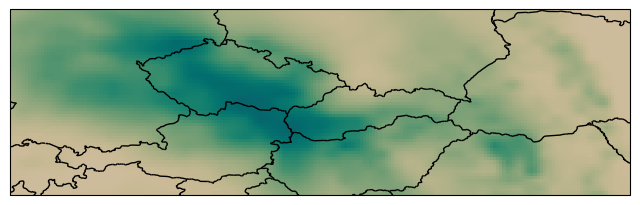}
  \vspace{-1.5em}
  \caption{Ground truth data}
\end{subfigure}\hfil
\begin{subfigure}{0.24\textwidth}
  \includegraphics[width=\linewidth]{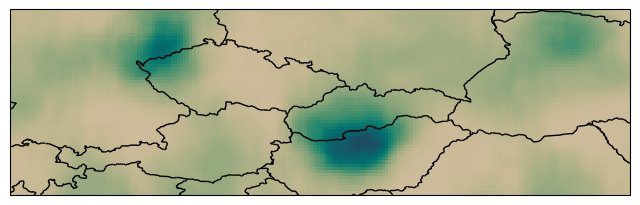}
  \vspace{-1.5em}
  \caption{ConvNP sample 1}
\end{subfigure}\hfil 
\begin{subfigure}{0.24\textwidth}
  \includegraphics[width=\linewidth]{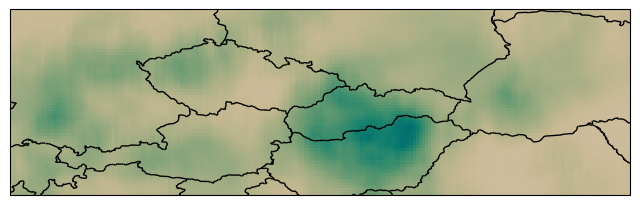}
  \vspace{-1.5em}
  \caption{ConvNP sample 2}
\end{subfigure}\hfil 
\begin{subfigure}{0.24\textwidth}
  \includegraphics[width=\linewidth]{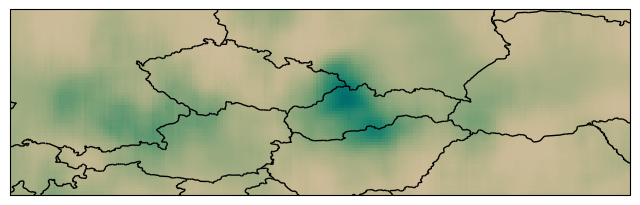}
  \vspace{-1.5em}
  \caption{ConvNP sample 3}
\end{subfigure}\\
%%
     %%%  Second row: Context set and GP samples  %%%
%%
\begin{subfigure}{0.24\textwidth}
  \includegraphics[width=\linewidth]{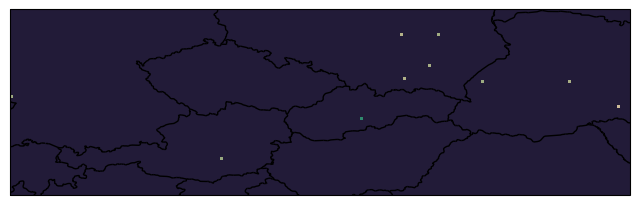}
  \vspace{-1.5em}
  \caption{Context set}
\end{subfigure}\hfil 
\begin{subfigure}{0.24\textwidth}
  \includegraphics[width=\linewidth]{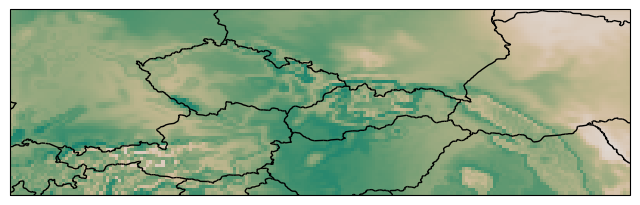}
  \vspace{-1.5em}
  \caption{GP sample 1}
\end{subfigure}\hfil
\begin{subfigure}{0.24\textwidth}
  \includegraphics[width=\linewidth]{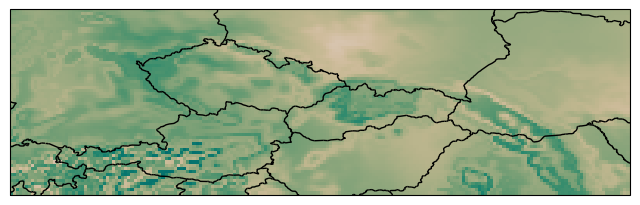}
  \vspace{-1.5em}
  \caption{GP sample 2}
\end{subfigure}\hfil
\begin{subfigure}{0.24\textwidth}
  \includegraphics[width=\linewidth]{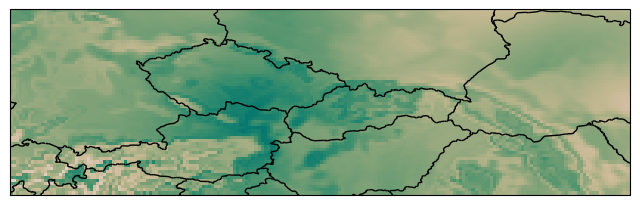}
  \vspace{-1.5em}
  \caption{GP sample 3}
\end{subfigure}
\caption{Samples from the predictive processes overlaid on central Europe, for a model trained on random $40 \times 40$ subregions of the full $61 \times 201$ central Europe region. Here the GP has overfit to the orography data, with samples that resemble the orography rather than precipitation.}
\label{fig:precipitation_sample_40x40}
\end{figure}

%% file: plots/prec_samples/prec_samples_app_40x40_2.tex
\begin{figure}[t]
%%
     %%%  First row: Ground truth and NP samples  %%%
%%
\vspace*{-0.8em}
\begin{subfigure}{0.24\textwidth}
  \includegraphics[width=\linewidth]{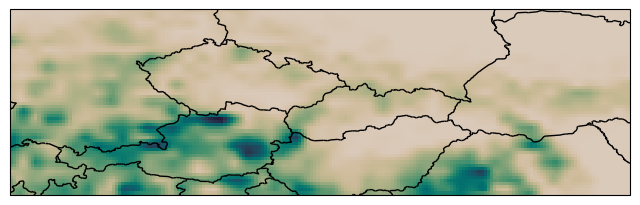}
  \vspace{-1.5em}
  \caption{Ground truth data}
\end{subfigure}\hfil
\begin{subfigure}{0.24\textwidth}
  \includegraphics[width=\linewidth]{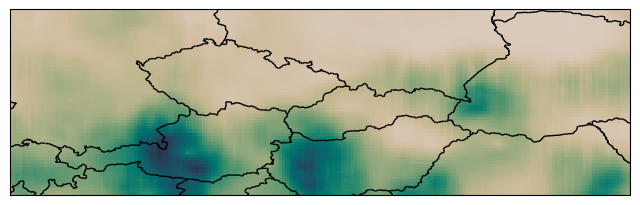}
  \vspace{-1.5em}
  \caption{ConvNP sample 1}
\end{subfigure}\hfil 
\begin{subfigure}{0.24\textwidth}
  \includegraphics[width=\linewidth]{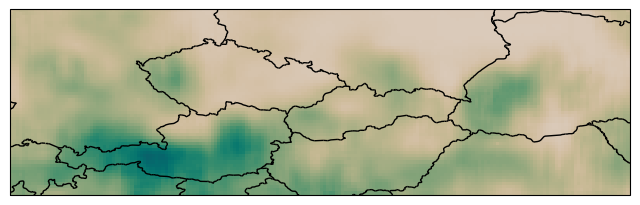}
  \vspace{-1.5em}
  \caption{ConvNP sample 2}
\end{subfigure}\hfil 
\begin{subfigure}{0.24\textwidth}
  \includegraphics[width=\linewidth]{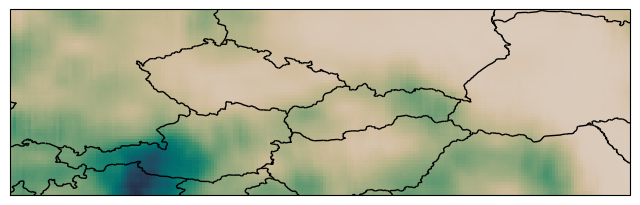}
  \vspace{-1.5em}
  \caption{ConvNP sample 3}
\end{subfigure}\\
%%
     %%%  Second row: Context set and GP samples  %%%
%%
\begin{subfigure}{0.24\textwidth}
  \includegraphics[width=\linewidth]{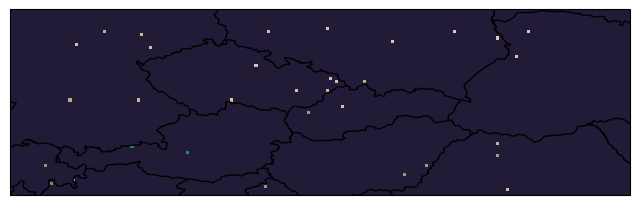}
  \vspace{-1.5em}
  \caption{Context set}
\end{subfigure}\hfil 
\begin{subfigure}{0.24\textwidth}
  \includegraphics[width=\linewidth]{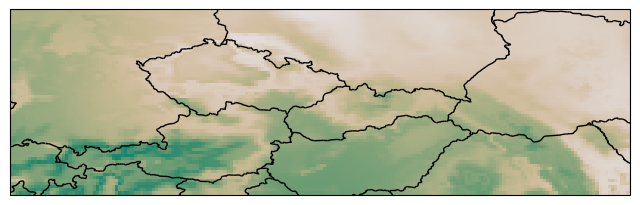}
  \vspace{-1.5em}
  \caption{GP sample 1}
\end{subfigure}\hfil
\begin{subfigure}{0.24\textwidth}
  \includegraphics[width=\linewidth]{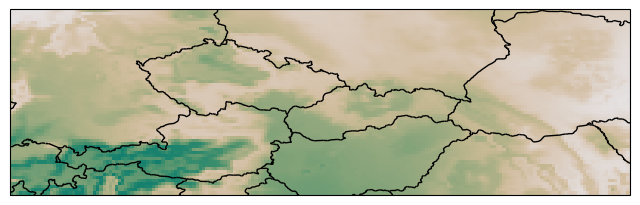}
  \vspace{-1.5em}
  \caption{GP sample 2}
\end{subfigure}\hfil
\begin{subfigure}{0.24\textwidth}
  \includegraphics[width=\linewidth]{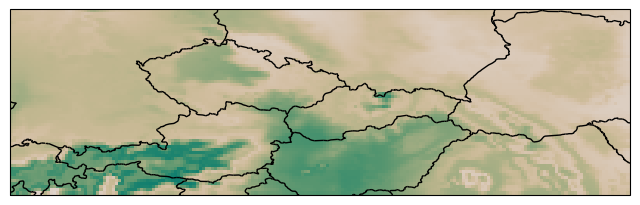}
  \vspace{-1.5em}
  \caption{GP sample 3}
\end{subfigure}
\caption{Samples from the predictive processes overlaid on central Europe, for a model trained on random $40 \times 40$ subregions of the full $61 \times 201$ central Europe region. The GP has again overfit to the orography data.}
\label{fig:precipitation_sample_40x40_2}
\end{figure}

%% file: plots/prec_samples/prec_samples_app_40x40_3.tex
\begin{figure}[t]
%%
     %%%  First row: Ground truth and NP samples  %%%
%%
\vspace*{-0.8em}
\begin{subfigure}{0.24\textwidth}
  \includegraphics[width=\linewidth]{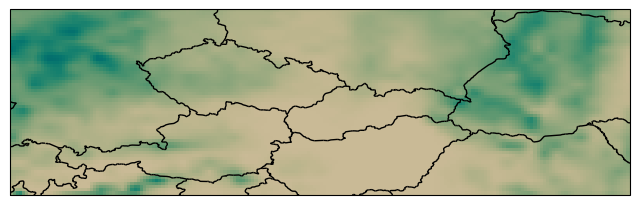}
  \vspace{-1.5em}
  \caption{Ground truth data}
\end{subfigure}\hfil
\begin{subfigure}{0.24\textwidth}
  \includegraphics[width=\linewidth]{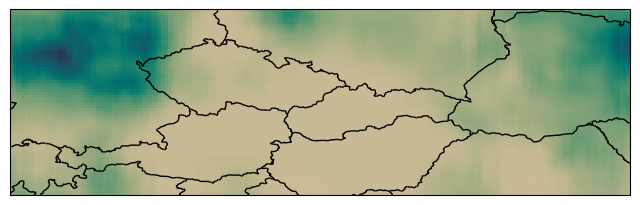}
  \vspace{-1.5em}
  \caption{ConvNP sample 1}
\end{subfigure}\hfil 
\begin{subfigure}{0.24\textwidth}
  \includegraphics[width=\linewidth]{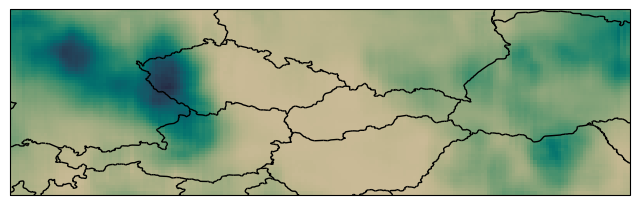}
  \vspace{-1.5em}
  \caption{ConvNP sample 2}
\end{subfigure}\hfil 
\begin{subfigure}{0.24\textwidth}
  \includegraphics[width=\linewidth]{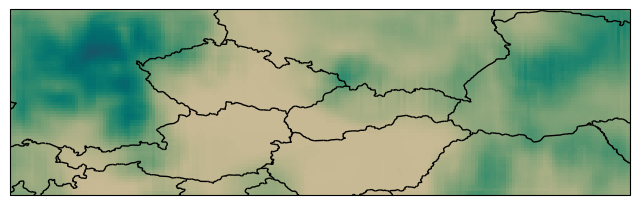}
  \vspace{-1.5em}
  \caption{ConvNP sample 3}
\end{subfigure}\\
%%
     %%%  Second row: Context set and GP samples  %%%
%%
\begin{subfigure}{0.24\textwidth}
  \includegraphics[width=\linewidth]{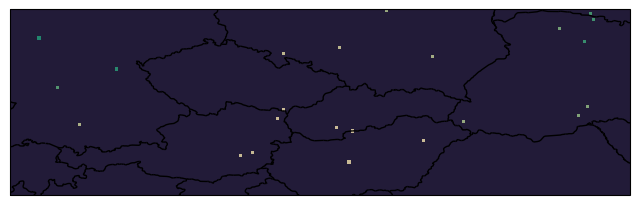}
  \vspace{-1.5em}
  \caption{Context set}
\end{subfigure}\hfil 
\begin{subfigure}{0.24\textwidth}
  \includegraphics[width=\linewidth]{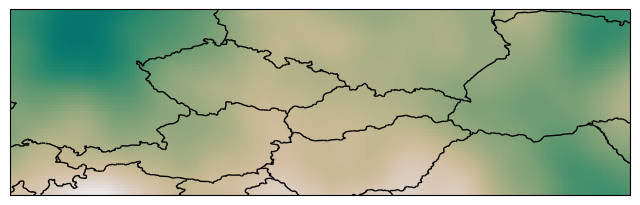}
  \vspace{-1.5em}
  \caption{GP sample 1}
\end{subfigure}\hfil
\begin{subfigure}{0.24\textwidth}
  \includegraphics[width=\linewidth]{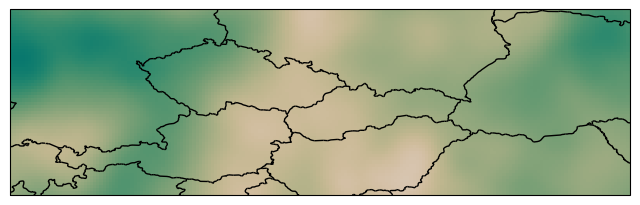}
  \vspace{-1.5em}
  \caption{GP sample 2}
\end{subfigure}\hfil
\begin{subfigure}{0.24\textwidth}
  \includegraphics[width=\linewidth]{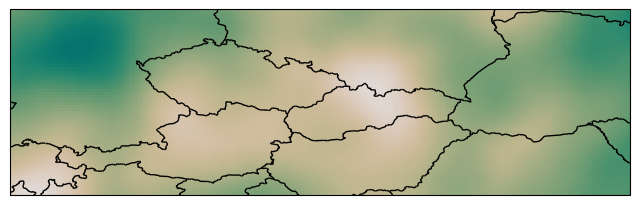}
  \vspace{-1.5em}
  \caption{GP sample 3}
\end{subfigure}
\caption{Samples from the predictive processes overlaid on central Europe, for a model trained on random $40 \times 40$ subregions of the full $61 \times 201$ central Europe region.}
\label{fig:precipitation_sample_40x40_3}
\end{figure}